\documentclass{article}

\usepackage[final,nonatbib]{neurips_2022}

\usepackage{titletoc}
\usepackage[utf8]{inputenc} 
\usepackage[T1]{fontenc}    
\usepackage{hyperref}       
\usepackage{url}            
\usepackage{booktabs}       
\usepackage{amsfonts}       
\usepackage{nicefrac}       
\usepackage{microtype}      
\usepackage{xcolor}         
\usepackage{graphicx}
\usepackage[normalem]{ulem}
\usepackage{amsthm}
\usepackage{amsmath}
\usepackage{xcolor}
\usepackage{comment}
\usepackage{yfonts}
\usepackage{amssymb}
\usepackage{bm}

\usepackage[toc,page,header]{appendix}

\newcommand{\R}[0] {\mathbb R}

\newtheorem{theorem}{Theorem}[section]

\newtheorem{lemma}[theorem]{Lemma}
\newtheorem{definition}[theorem]{Definition}

\title{Why neural networks find simple solutions:\\
the many regularizers of geometric complexity}

\author{%
  Benoit Dherin\thanks{equal contribution}\\
  Google\\
  \texttt{dherin@google.com} \\
  \And
  Michael Munn\footnotemark[1]\\
  Google\\
  \texttt{munn@google.com} \\
  \And
  Mihaela Rosca \\
  DeepMind, London \\ University College London \\ \texttt{mihaelacr@deepmind.com}\\  
  \And
  David G.T. Barrett \\
  DeepMind, London \\ \texttt{barrettdavid@deepmind.com}\\
}

\begin{document}

\maketitle

\begin{abstract}%
In many contexts, simpler models are preferable to more complex models and the control of this model complexity is the goal for many methods in machine learning such as regularization, hyperparameter tuning and architecture design. In deep learning, it has been difficult to understand the underlying mechanisms of complexity control, since many traditional measures are not naturally suitable for deep neural networks. Here we develop the notion of geometric complexity, which is a measure of the variability of the model function, computed using a discrete Dirichlet energy. Using a combination of theoretical arguments and empirical results, we show that many common training heuristics such as parameter norm regularization, spectral norm regularization, flatness regularization, implicit gradient regularization, noise regularization and the choice of parameter initialization all act to control geometric complexity, providing a unifying framework in which to characterize the behavior of deep learning models.
\end{abstract}

\section{Introduction}

Regularization is an essential ingredient in the deep learning recipe and understanding its impact on the properties of the learned solution is a very active area of research  \cite{GoodBengCour16, hoffman2020robust, qin20198efficient, Sokolic2017RobustLM}. Regularization can assume a multitude of forms, either added explicitly as a penalty term in a loss function \cite{GoodBengCour16} or implicitly through our choice of hyperparameters \cite{barrett2021implicit, rosca2021discretization, SeongLKHK18, smith2021on,  SmithQuoc2018}, model architecture \cite{NEURIPS2018_a41b3bb3, lin2017does, ma2020quenching} or initialization  \cite{understanding2010glorot, gunasekar2018characterizing, he2015delving,  Li2018learning, nagarajan2019generalization,  zhang2020type, zou2020gradient}. These forms are generally not designed to be analytically tractable, but in practice, regularization is often invoked in the control of model complexity, putting a pressure on a model to discover simple solutions more so than complex solutions. 

To understand regularization in deep learning, we  need to precisely define model ‘complexity’ for deep neural networks. Complexity theory provides many techniques for measuring the complexity of a model, such as a simple parameter count, or a parameter norm measurement \cite{arora2018stronger, dziugaite2017computing, maddox2020rethinking, neyshabur2017implicit} but many of these measures can be problematic for neural networks \cite{Jiang2020Fantastic, generalization, neyshabur2018role}. The recently observed phenomena of ‘double-descent’ \cite{belkin2021fear, Belkin15849,  deep_double_descent} illustrates this clearly: neural networks with high model complexity, as measured by a parameter count, can fit training data closely (sometimes interpolating the data exactly), while simultaneously having low test error \cite{deep_double_descent, generalization}. Classically, we expect that interpolation of training data is evidence of overfitting, but yet, neural networks seem to be capable of interpolation while also having low test error. It is often suggested that some form of implicit regularization or explicit regularization is responsible for this, but how should we  account for this in theory, and what complexity measure is most appropriate? 

In this work, we develop a measure of model complexity, called Geometric Complexity (GC), that has properties that are suitable for the analysis of deep neural networks. We use theoretical and empirical techniques to demonstrate that many different forms of regularization and other training heuristics can act to control geometric complexity through different mechanisms. We argue that the geometric complexity provides a convenient proxy for neural network performance.

Our primary contributions are:

\begin{itemize}

    \item We develop a computationally tractable notion of complexity (Section \ref{section:complexity}), which we call \emph{Geometric Complexity} (GC), that has many close relationships with many areas in deep learning and mathematics including harmonic function theory, Lipschitz smoothness (Section \ref{section:complexity}),  and regularization theory (Section \ref{section:explicit_reg}).
    
    \item{We provide evidence that common training heuristics keep the geometric complexity low, including: (i) common initialization schemes (Section \ref{section:initialization}) (ii) the use of overparametrized models with a large number of layers (Fig.~\ref{fig:initialization}) (iii) large learning rates, small batch sizes, and implicit gradient regularization (Section \ref{section:implicit_regularization} and Fig. \ref{figure:implicit_regularization}) (iv) explicit parameter norm regularization, spectral norm regularization, flatness regularization, and label noise regularization (Section \ref{section:explicit_reg} and Fig.~\ref{fig:explicit_regularization})}
    
    \item{We show that the geometric complexity captures the double-descent behaviour observed in the test loss as model parameter count increases (Section \ref{section:double_descent} and Fig. \ref{fig:double_descent}).}
    
\end{itemize}
The aim of this paper is to introduce geometric complexity, explore its properties and highlight its connections with existing implicit and explicit regularizers. To disentangle the effects studied here from optimization choices, we use stochastic gradient descent without momentum to train all models. We also study the impact of a given training heuristic on geometric complexity in isolation of other techniques to avoid masking effects. For this reason we do not use data augmentation or learning rate schedules in the main part of the paper. 
In the Supplementary Material (SM) we redo most experiments using SGD with momentum (Section \ref{appendix:additional_experiments_momentum}) and Adam (Section \ref{appendix:additional_experiments_adam}) with very similar conclusion. We also observe the same behavior of the geometric complexity in a setting using learning rate schedule, data augmentation, and explicit regularization in conjunction to improve model performance (Section \ref{appendix:real_life_gc_regularization}).
The exact details of all experiments in the main paper are listed in SM Section \ref{appendix:experiments}. All  additional experiment results and details can be found in SM Section \ref{appendix:additional_experiments}.

\section{Geometric complexity and  Dirichlet energy} \label{section:complexity}

Although many different forms of complexity measures have been proposed and investigated (e.g., \cite{dziugaite2020search, Jiang2020Fantastic,neyshabur2017implicit}), it is not altogether clear what properties they should have, especially for deep learning. For instance, a number of them like the Rademacher complexity \cite{Koltchinskii99rademacherprocesses}, the VC dimension \cite{vapnik1971on}, or the simple model parameter count focus  on measuring the entire hypothesis space, rather than a specific function, which can be problematic in deep learning \cite{Jiang2020Fantastic, generalization}. Other measures like the number of linear pieces for ReLU networks \cite{arora2018understanding, serra2018bounding} or various versions of the weight matrix norms \cite{neyshabur2017implicit} measure the complexity of the model function independently from the task at hand, which is not desirable \cite{pmlr-v137-rosca20a}. An alternative approach is to learn a complexity measure directly from data \cite{lee2020neural} or to take the whole training procedure over a dataset into account \cite{deep_double_descent}. Recently, other measures focusing on the model function complexity over a dataset have been proposed in \cite{gao2016degree} and \cite{maddox2020rethinking} to help explain the surprising generalization power of deep neural networks.
Following that last approach and motivated by frameworks well established in the field of geometric analysis \cite{jost2008riemannian}, we propose a definition of complexity related to the theory of harmonic functions and minimal surfaces. Our definition has the advantage of being computationally tractable and implicitly regularized by many training heuristics in the case of neural networks. It focuses on measuring the complexity of individual functions rather than that of  the whole function space, which makes it different from the Radamacher or VC complexity.
\begin{definition}\label{definition:geometric_complexity}
Let $g_{\theta}: \mathbb{R}^d \to \mathbb{R}^k$ be a neural network parameterized by $\theta$. We can write $g_\theta(x) = a(f_\theta(x))$ where $a$ denotes the last layer activation, and $f_\theta$ its logit network.  
The GC of the network over a dataset $D$ is defined to be the discrete Dirichlet energy of its logit network:
\begin{equation}\label{eqn:geometric_complexity}
    \langle f_\theta,\, D\rangle_G = \frac{1}{|D|}\sum_{x\in D} \|\nabla_x f_\theta(x)\|_F^2,
\end{equation}
where $\|\nabla_x f_\theta(x)\|_F$ is the Frobenius norm of the network Jacobian. 
\end{definition}
Note that this definition is well-defined for any differentiable model, not only a neural network, and incorporates both the model function and the dataset over which the task is determined. 

Next, we discuss how GC relates to familiar concepts in deep learning.

\paragraph{Geometric complexity and linear models:} Consider a linear transformation $f(x) = Ax + b$ from $\R^d$ to $\R^k$ and a dataset $D = \{x_i\}_{i=1}^N$ where $x_i\in \mathbb{R}^d$. At each point $x\in D$, we have that $\|\nabla_x f(x)\|^2_F = \|A\|_F^2$, hence the GC for a linear transformation is
$
\langle f_\theta, D\rangle_G = \|A\|_F^2.
$
Note, this implies that GC for linear transformations (and more generally, affine maps), is independent of the dataset $D$ and zero for constant functions. 
Furthermore, note that the GC in this setting coincides precisely with the L2 norm of the model weight matrix. Thus, enforcing an L2 norm penalty is equivalent to regularizing the GC for linear models (see Section \ref{section:explicit_reg} for more on that point).

\paragraph{Geometric complexity and ReLU networks:} For a ReLU network $g_\theta:\R^d\rightarrow\R^k$  as defined in Definition \ref{definition:geometric_complexity}, the GC over a dataset $D$ has a very intuitive form. Since a ReLU network parameterizes piece-wise linear functions \cite{arora2018understanding}, the domain can be broken into a partition of subsets $X_i\subset \mathbb{R}^d $ where $f_\theta$ is an affine map $A_i x + b_i$. Now denote by $D_i$ the points in the dataset $D$ that fall in the  linear piece defined on $X_i$. For every point $x$ in $X_i$, we have that $\|\nabla_x f_\theta(x)\|_F^2 = \|A_i\|^2_F$. Since the $D_i$'s partition the dataset $D$, we obtain 
\begin{equation}\label{eq:relu_inference_bias}
    \langle f_\theta, D\rangle_G = \sum_i \left(\frac{n_i}{|D|}\right) \|A_i\|^2_F,
\end{equation}
where $n_i$ is the number of points in the dataset $D$ falling in $X_i$.
We see from Eqn.~\eqref{eq:relu_inference_bias} that for ReLU networks the GC over the whole dataset coincides exactly with the GC on a batch $B\subset D$, provided that the proportion of points in the batch falling into each of the the linear pieces are preserved.  This makes the evaluation of the GC on large enough batches a very good proxy to the overall GC over the dataset, and computationally tractable during training. 

\paragraph{Geometric complexity and Lipschitz smoothness:}
One way to measure the smoothness of a  function $f:\R^d\rightarrow \R^k$ on a subset $X\subset \R^d$ is by its Lipschitz constant; i.e., the smallest $f_L \geq 0$ such that
$
\|f(x_1) -  f(x_2)\| \leq f_L \|x_1 - x_2\|,
$
for all $x_1,\,x_2\in X$. Intuitively, the constant $f_L$ measures the maximal amount of variation allowed by $f$ when the inputs change by a given amount. 
Using the Lipschitz constant, one can define a complexity measure of a function $f$ as the Lipschitz constant $f_L$ of the function over the input domain $\R^d$.
Since $\|\nabla_x f(x)\|^2_F \leq \min(k, d)\|\nabla_x f(x)\|^2_{op} \leq \min(k,d) f_L^2$ where $\|\cdot\|_{op}$ is the operator norm, we obtain a general bound on the GC by the Lipschitz complexity:
\begin{equation}\label{equation:lipschitz_bounds_geometric}
\langle f, D\rangle_G ~=~
\frac{1}{|D|}\sum_{x\in D} \|\nabla_x f(x)\|_F^2 
~\leq ~\min(k, d) f_L^2.
\end{equation}
While the Lipschitz smoothness of a function provides an upper bound on GC, there are a few fundamental differences between the two quantities. Firstly, GC is data dependent: a model can have low  GC while having high Lipschitz constant due to the model not being smooth in parts of the space where there is no training data. Secondly, the GC can be computed exactly given a model and dataset, while for neural networks only loose upper bounds are available to estimate the Lipschitz constant.

\paragraph{Geometric complexity, arc length and harmonic maps:}
Let us start with a motivating example: Consider a dataset consisting of 10 points lying on a parabola in the plane and a large ReLU deep neural network trained via gradient descent to learn a function $f:\mathbb{R} \to \mathbb{R}$ that fits this dataset (Fig. \ref{fig:training_sequence}).
\begin{figure}[h]
  \centering
  \includegraphics[width=0.99
  \linewidth]{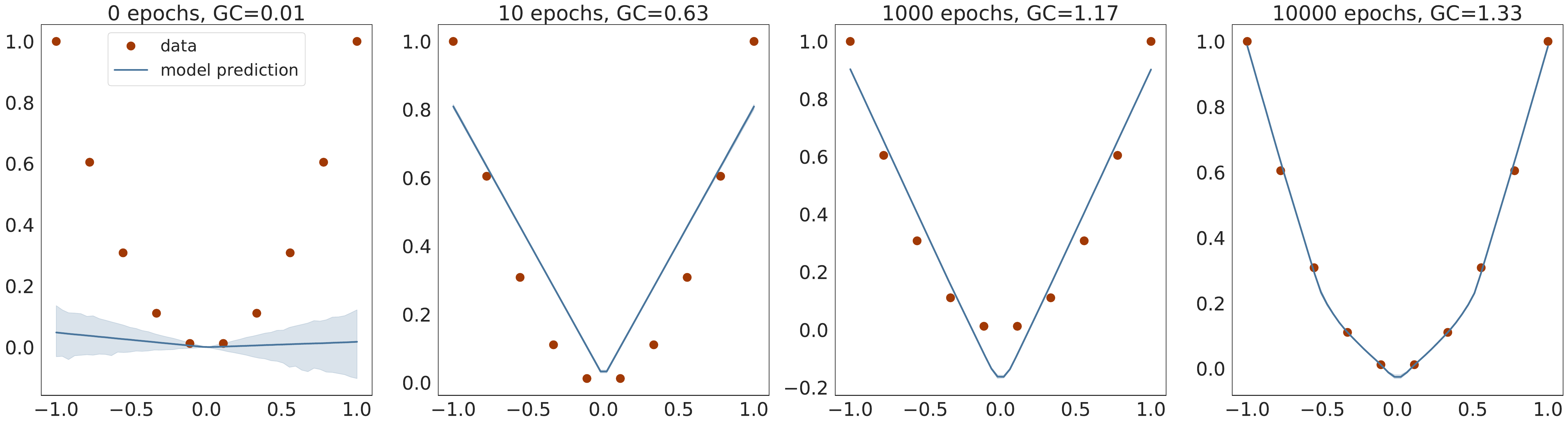}
  \caption{For a large MLP fitting 10 points, the complexity of the function being learned gradually grows in training, while avoiding unnecessary complexity by keeping the function arc length minimal.}
  \label{fig:training_sequence}
\end{figure}
Throughout training the model function seems to attain minimal arc length for a given level of training error. Recalling the formula for arc length of $f$, which is the integral of over $[-1, 1]$ of  $\sqrt{1+ f'(x)^2}$, and using the Taylor approximation $\sqrt{1+ x^2} \approx 1 + \frac{x^2}{2}$, it follows that minimizing the arc length is equivalent to minimizing the classic Dirichlet energy:
\begin{equation}\label{eqn:classicDE}
    E(f) = \dfrac{1}{2}\int_{\Omega} \|\nabla_x f(x)\|^2 dx.
\end{equation}
where $\Omega = [-1, 1]$. The Dirichlet energy can be thought intuitively of as a measure of the variability of the function $f$ on $\Omega$. Its minimizers, subject to a boundary condition $f_{|\partial \Omega} = h$, are called \emph{harmonic maps}, which are maps causing the ``least intrinsic stretching'' of the domain $\Omega$ \cite{solomon2013dirichlet}. The geometric complexity in Definition \ref{definition:geometric_complexity} is an unbiased estimator of a very related quantity
\begin{eqnarray} \label{eq:inference_bias}
\mathbb E_X(\|\nabla_x f_\theta(x)\|_F^2) = \int_{\R^d}\|\nabla_x f_\theta(x)\|_F^2\, p_X(x) dx,
\end{eqnarray}
where the domain $\Omega$ is replaced by the probability distribution of the features $p_X$ and the boundary condition is replaced by the dataset $D$. We could call the quantity defined in Eqn. \eqref{eq:inference_bias} the \emph{theoretical geometric complexity} as opposed to the \emph{empirical geometric complexity} in Definition \ref{definition:geometric_complexity}. The theoretical geometric complexity is very close to a complexity measure investigated in \cite{novak2018sensitivity}, where the Jacobian of the full network is considered rather than just the logit network as we do here. In their work, the expectation is also evaluated on the test distribution. They observe a correlation between this complexity measure and generalization empirically in a set of extensive experiments. In our work, we use the logit network (rather than the full network) and we evaluate the empirical geometric complexity on the train set (rather than on the test set) in order to derive theoretically that the implicit gradient regularization mechanism from \cite{barrett2021implicit} creates a regularizing pressure on GC (see Section \ref{section:implicit_regularization}).

In the remaining sections we provide evidence
that common training heuristics do indeed keep the GC low, encouraging neural networks to find intrinsically simple solutions.

\section{Impact of initialization on geometric complexity}
\label{section:initialization}
Parameter initialization choice is an important factor in deep learning. Although we are free to specify exact parameter initialization values, in practice, a small number of default initialization schemes have emerged to work well across a wide range of tasks \cite{understanding2010glorot,he2015delving}. Here, we explore the relationship between some of these initialization schemes and GC.

To begin, consider the one dimensional regression example that we introduced in Figure~\ref{fig:training_sequence}. In this experiment, we employed a standard initialization scheme to initialise the parameters: we sample them from a truncated normal distribution with variance inversely proportional to the number of input units and the bias terms were set to zero. We observe that the initialised function on the interval $[-1,1]$ is very close to the zero function (Fig.~\ref{fig:training_sequence}), and the zero function has zero GC. This observation suggests that initialization schemes that have low initial GC are useful for deep learning.
\begin{figure}[h]
  \centering
  \includegraphics[width=1\linewidth]{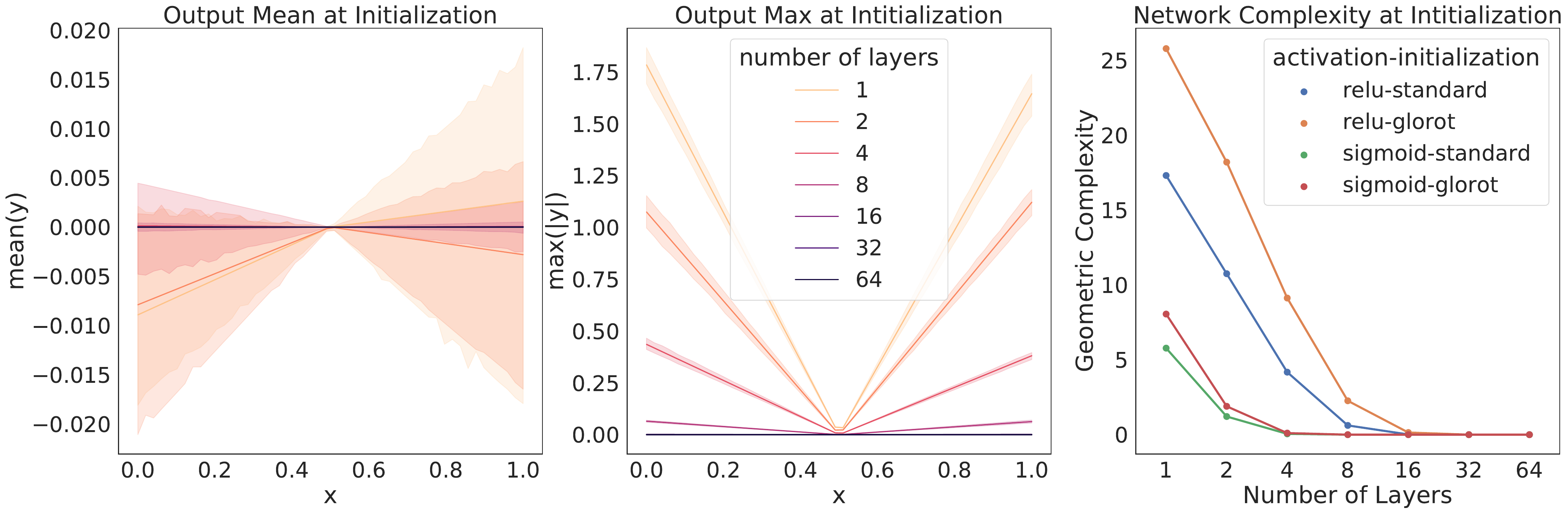}
  \caption{MLP's given by $y=f_{\theta_0}(x)$ initialize closer to the zero function and closer to zero GC, as the number of layers increases. \textbf{Left and Middle:} The ReLU MLP's initialized with the standard scheme are evaluated using input values along the line $P_1 + (P_2 - P_1)x$ with $x\in[0,1]$ between two diagonal points $P_1$ and $P_2$ of the hyper-cube $[-1,1]^d$. \textbf{Right:} GC is computed on a dataset $D$ of 100 normalized data points. All MLP's have 500 neurons per layer. }
  \label{fig:initialization}
\end{figure}

To explore this further, we consider deep ReLU networks with larger input and output spaces initialized using the same scheme as above, and measure the GC of the resulting model.
Specifically, consider the initialized ReLU network given by $f_{\theta_0}: \mathbb{R}^d \to \mathbb{R}^k$ with $d=150528$ and $k=1000$, with parameter initialization $\theta_0$, and varying network depth. We measure the ReLU network output size by recording the mean and maximum output values, evaluated using input values along the line $P_1 + (P_2 - P_1)x$ with $x\in[0,1]$ between two diagonal points $P_1$ and $P_2$ of the hyper-cube $[-1,1]^d$. We observe that these ReLU networks initialize to functions close to the zero function, and become progressively closer to a zero valued function as the number of layer increases (Fig.~\ref{fig:initialization}). For ReLU networks, this is not entirely surprising. With biases initialised to zero, we can express a ReLU network in a small neighborhood of a given point $x\in\R^d$ as a product of matrices
$
f_{\theta_0}(x) =  W_1 P_2 W_2 P_3 W_3\cdots P_l W_l x,
$
where the $W_i$'s are the weight matrices and the $P_i$'s are diagonal matrices with 0 and 1 on their diagonals. This representation makes it clear that at initialization the ReLU network passes through the origin; i.e., $f_{\theta_0}(0) = 0$. Furthermore, with weight matrices initialised around zero, using a scaling that can reduce the spread of the distribution as the matrices grow, we can expect that deeper ReLU networks generated by multiplying a large number of small-valued weights, can produce output values close to zero (for input values taken from a hyper-cube $[-1, 1]^d$). 
We extend these results further to include additional initialization schemes and experimentally confirm that the GC can be brought close to zero with a sufficient number of layers. In fact, this is true not only for ReLU networks with the standard initialization scheme, but for a number of other common activation functions and initialization setups \cite{understanding2010glorot, he2015delving}, and even on domains much larger than the normalized hyper-cube (Fig.~\ref{fig:initialization} and SM Section \ref{appendix:initialization_on_large_domains}). Theoretically, it has been shown very recently in \cite{avelin2022deep} (their Theorem 5), that under certain technical conditions, a neural network at random initialization will converge to a constant function (which has GC equal to zero) as the number of layers increases.

\section{Impact of explicit regularization on geometric complexity}
\label{section:explicit_reg}
Next, we explore the relationship between GC and various forms of explicit regularization using a combination of theoretical and empirical results.

\begin{figure}[h]
  \centering
  \includegraphics[width=.3\linewidth]{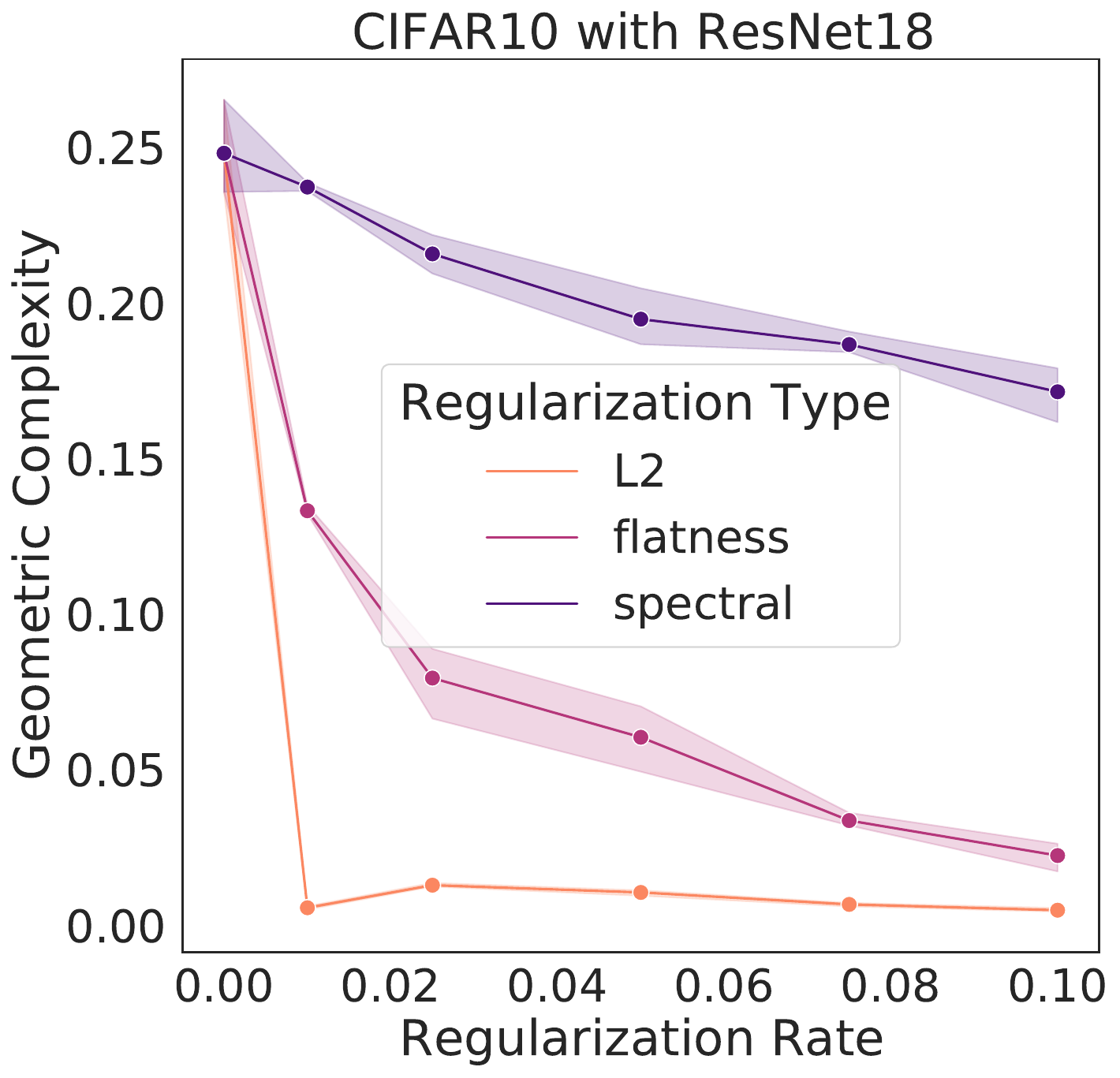}
  \includegraphics[width=.3\linewidth]{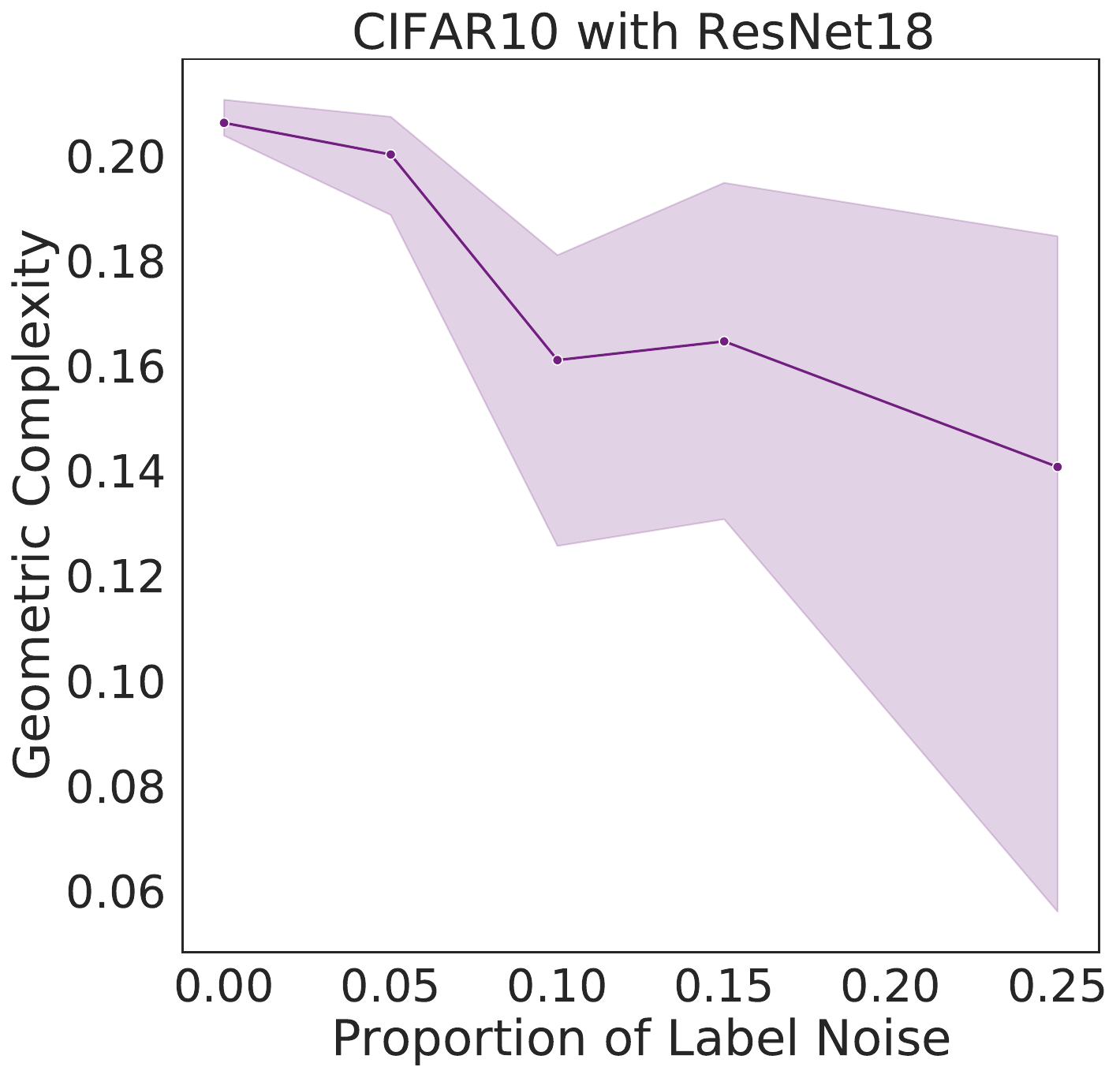}
  \includegraphics[width=.3\linewidth]{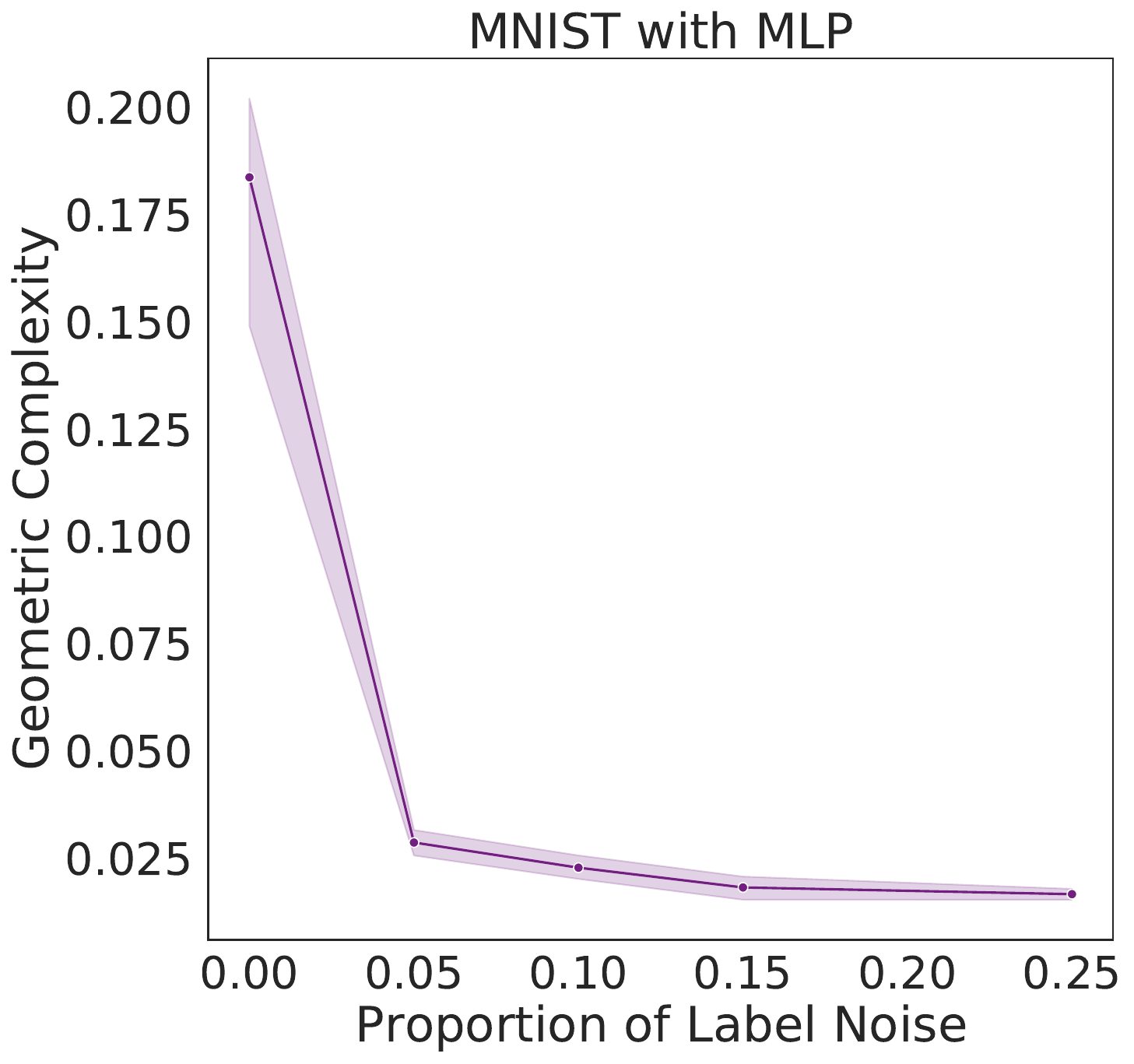}
  \caption{Explicit regularization and GC. \textbf{Left:} As the L2, flatness, and spectral regularization increase, GC decreases. (See SM Section \ref{appendix:additional_experiments_for_explicit_regularization} for additional experiments on explicit regularization with ranges targeted to each regularization type.)
 \textbf{Middle and Right:} As label noise regularization increases, GC decreases.}
  \label{fig:explicit_regularization}
\end{figure}

\paragraph{L2 regularization:} In L2 regularization, the Euclidean norm  of the parameter vector is explicitly added to a loss function, so as to penalize solutions that have excessively large parameter values. This is one of the simplest, and most widely used forms of regularization. For linear models ($f_\theta(x) = Ax + b$), we saw in Section \ref{section:complexity} that the GC coincides with the Frobenius norm of the matrix $A$. This means that standard L2 norm penalties on the weight matrix coincide in this case with a direct explicit regularization of the GC of the linear model.  
For non-linear deep neural networks, we cannot directly identify the L2 norm penalty with the model GC. However, in the case of ReLU networks, for each input point $x$, the network output $y$ can be written in a neighborhood of $x$ as $y = P_l W_l \dots P_1 W_1 x + c$,
where $c$ is a constant, the $P_i$'s are diagonal matrices with 0 and 1 on the diagonal, and the $W_i$'s are the network weight matrices. This means that the derivative at $x$ of the network coincides with the matrix $P_l W_l \dots P_1 W_1$. Therefore, the GC is just the Frobenius norm of the product of matrices.
Now, an L2 penalty  $\|W_l\|_F^2 + \cdots + \|W_1\|_F^2$
encourages small numbers in the values of the weight matrices, which in turn is likely to encourage small numbers for the values in the product $P_l W_l \dots P_1 W_1$, resulting in a lower GC.

We can demonstrate this relationship empirically, by training a selection of neural networks with L2 regularization, each with a different  regularization strength. We measure the GC for each network at the time of maximum test accuracy. We observe empirically that strengthening L2 regularization coincides with a decrease in GC values (Fig.~\ref{fig:explicit_regularization}). 

\paragraph{Lipschitz regularization via spectral norm regularization:}

A number of explicit regularization schemes have been used to tame the Lipschitz smoothness of the model function and produce smoother models with reduced test error \cite{arbel2018gradient, arora2018understanding, elsayed2018large, fedus2017many,  gulrajani2017improved, kodali2018on, miyato2018spectral, pmlr-v137-rosca20a,  Yoshida2017SpectralNR}. Smoothness regularization has also been shown to be beneficial outside the supervised learning context, in GANs~\cite{brock2018large, miyato2018spectral,zhang2019self} and reinforcement learning~\cite{bjorck2021towards, gogianu2021spectral}.
One successful approach to regularising the Lipschitz constant of a neural network with 1-Lipschitz activation functions (e.g. ReLU, ELU) is to constrain the spectral norm of each layer of the network (i.e., the maximal singular values $\sigma_{\max}(W_i)$ of the weight matrices $W_i$), since the product of the spectral norms of the networks weight matrices is an upper bound to the Lipschitz constant of the model: $f_L \le \sigma_{\max}(W_1)\cdots \sigma_{\max}(W_l)$.  
Using inequality~(\ref{equation:lipschitz_bounds_geometric}), we see that any approach that constrains the Lipschitz constant of the model constrains GC. To confirm this theoretical prediction, we train a ResNet18 model on CIFAR10 \cite{krizhevsky2009learning} and regularize using spectral regularization~\cite{Yoshida2017SpectralNR}, an approach which adds the regularizer $\frac{\lambda}{2}\sum_l (\sigma_{\max}(W_l))^2$ to the model loss function.
We observe that GC decreases as the strength of spectral regularization $\lambda$ increases (Fig. \ref{fig:explicit_regularization}).

\paragraph{Noise regularization:} The addition of noise during training is known to be an effective form of regularization. For instance, it has been demonstrated  \cite{blanc2020} that the addition of noise to training labels during the optimization of a least-square loss using SGD exerts a regularising pressure on $\sum_{(x, y)\in D} \|\nabla_\theta f_\theta(x)\|^2/|D|$.  
Here, we demonstrate empirically that the GC of a ResNet18 trained on CIFAR10 reduces as the proportion of label noise increases (Fig.~\ref{fig:explicit_regularization}, middle). The same is true of an MLP trained on MNIST \cite{deng2012mnist} (Fig.~\ref{fig:explicit_regularization}, right). In Section \ref{section:implicit_regularization}, we provide a theoretical argument which justifies these experiments, showing that a regularizing pressure on $\|\nabla_\theta f_\theta(x)\|^2$ transfers to a regularizing pressure on $\|\nabla_x f_\theta(x)\|^2$. Thus, label noise in SGD in turn translates into a regularizing pressure on the GC in the case of neural networks.

\paragraph{Flatness regularization:} 
In flatness regularization, an explicit gradient penalty term, $\|\nabla_\theta L_B\|^2$ is added to the loss (where $L_B$ is the loss evaluated across a batch $B$). It has been observed in practice that flatness regularization can be effective in many deep learning settings, from supervised learning \cite{ geiping2022stochastic, smith2021on} to GAN training \cite{balduzzi2018mechanics, mescheder2017numerics, nagarajan2017gradient,odegan, rosca2021discretization}. Flatness regularization penalizes learning trajectories that follow steep slopes across the loss surface, thereby encouraging learning to follow shallower slopes toward flatter regions of the loss surface. 
We demonstrate empirically that GC decreases as the strength of flatness regularization increases (Fig. \ref{fig:explicit_regularization}). In the next section, we will provide a theoretical argument that flatness regularization can control GC.

\paragraph{Explicit GC regularization and Jacobian regularization:}
All the forms of regularization above have known benefits for improving the test accuracy in deep learning. As we saw, they all also implicitly regularize GC. This raises the question as whether regularizing for GC directly and independently of any other mechanism is sufficient to improve model performance. Namely, we can add GC computed on the batch to the loss 
$
L_{\textrm{reg}}(\theta) = L_B(\theta) + \alpha/B \sum_{x \in B}\|\nabla_x f_\theta(x)\|^2_F.
$
This is actually a known form of explicit regularization, called {\bf Jacobian regularization}, which has been correlated with increased generalization but also robustness to input shift \cite{hoffman2020robust, Sokolic2017RobustLM, varga2018gradient, Yoshida2017SpectralNR}. \cite{Sokolic2017RobustLM} specifically shows that adding Jacobian regularization to the loss function can lead to an increase in test set accuracy (their Tables III, IV, and V). In SM Section \ref{appendix:gc_regularization_additional_experiments} we train a MLP on MNIST and a ResNet18 on CIFAR10 regularized explicitly with the geometric complexity. We observe an increase of test accuracy and a decrease of GC with more regularization. Related regularizers include \emph{gradient penalties} of the form  $\sum_{x} (||\nabla_{x} f_{\theta}(x)||- K)^2$ which have been used for GAN training~\cite{fedus2017many,gulrajani2017improved, kodali2017convergence}. 
We leave the full investigation of the importance of GC outside supervised learning for future work.

\section{Impact of implicit regularization on geometric complexity} \label{section:implicit_regularization}

Implicit regularization is a hidden form of regularization that emerges as a bi-product of model training. Unlike explicit regularization, it is not explicitly added to a loss function. In deep learning settings where no explicit regularization is used, it is the only form of regularization. Here, we use a combination of theoretical and empirical results to argue that some recently identified implicit regularization mechanisms in gradient descent \cite{barrett2021implicit, chao2021sobolev, smith2021on} exert a regularization pressure on geometric complexity. Our  argument proceeds as follows: 1) we identify a mathematical term (the implicit gradient regularization term) that characterizes implicit regularization in gradient descent, 2) we demonstrate that this term depends on model gradients, 3) we identify the conditions where model gradient terms apply a regularization pressure on geometric complexity.      

\paragraph{Step 1:} The implicit regularization that we consider emerges as a bi-product of the discrete nature of gradient descent updates. In particular it has been shown that a discrete gradient update $\theta' = \theta - h \nabla_\theta L_B(\theta)$ over a batch of data implicitly minimizes a modified loss, 
$ \widetilde L_B = L_B + \frac h4 \|\nabla_\theta L_B\|^2$,
where the second term is called the Implicit Gradient Regularizer (IGR) \cite{barrett2021implicit}. Gradient descent optimization is better characterized as a continuous flow along the gradient of the modified loss, rather than the original unmodified loss. By inspection, the IGR term implicitly regularizes training toward trajectories with smaller loss gradients toward flatter region on the loss surface \cite{barrett2021implicit} . 

\paragraph{Step 2:} Next, we develop this implicit regularizer term for a multi-class classification cross-entropy loss term. We can write (see SM Section \ref{appendix:modified_loss_expansion} for details):
\begin{equation}\label{eq:expanded_modified_loss}
    \widetilde L_B  =  L_B + \frac h{4B}\left(
    \frac 1B
    \sum_{x, i}
    \epsilon_x^i(\theta)^2
    \|\nabla_\theta f^i_\theta(x)\|^2 \right)+ \frac h{4} A_B(\theta) + \frac h4 C_B(\theta), 
\end{equation}
where $C_B(\theta)$ measures the \emph{batch gradient alignment}
\begin{equation}
C_B(\theta)  = 
\frac 1{B^2}\sum_{(x, y)\neq (x', y')}\left\langle \nabla_\theta L(x,y, \theta), \nabla_\theta L(x',y', \theta) \right\rangle.
\end{equation}
and $A_B(\theta)$ measures the  \emph{label gradient alignment}:
\begin{equation}
A_B(\theta) = \frac 1{B^2} \sum_{x\in B} \sum_{i\neq j} 
\langle
\epsilon_x^i \nabla_\theta f^i_\theta(x), 
\epsilon_x^j \nabla_\theta f^j_\theta(x)
\rangle.
\end{equation}
where $\epsilon_x^i(x)=a(f_\theta(x))^i - y^i$ is the signed residual and $a$ denotes the activation function of the last layer. Note that this residual term arises in our calculation using $\nabla_\theta L(x, y, \theta)  =  \epsilon^1_x(\theta) \nabla_\theta f_\theta^1(x) + \cdots + \epsilon_x^k(\theta) \nabla_\theta f_\theta^k(x)$ (see SM Section \ref{appendix:gradients}, Eqn.~\eqref{eq:loss_theta_gradient}). This development also extends to other widely used loss functions such as the cross entropy and least-square loss.

\paragraph{Step 3:} Now, from the modified loss in Eqn.~\eqref{eq:expanded_modified_loss}, we can observe the conditions under which SGD puts an implicit pressure on the gradient norms  $\|\nabla_\theta f^i_\theta(x)\|^2$ to be small: the batch gradient alignment and label gradient alignment terms must be small relative to the gradient norms, or positive valued. 

We also derive conditions under which this implicit pressure on $\|\nabla_\theta f_\theta(x)\|_F^2$ transfers to a regularizing pressure on $\|\nabla_x f_\theta(x)\|_F^2$ (see SM Section \ref{appendix:transfer} for the proof):.
\begin{theorem} \label{thm:transfer} Consider a logit network $f_\theta:\R^d\rightarrow \R^k$ with $l$ layers, then we have the following inequality
\begin{equation}\label{thm:transfer_inequality}
\|\nabla_x f_\theta(x)\|^2_F  \leq  \frac{\|\nabla_\theta f_\theta(x)\|^2_F}{T_1^2(x)+\cdots+T_l^2(x)},
\end{equation}
where $T_i$ is the \it{transfer function} for layer $i$ given by 
\begin{equation}
T_i(x, \theta) 
 = 
\frac{1}{\sqrt{\min(d,k)}}
\frac{\sqrt{1 + \|h_i(x)\|^2_2}}{\sigma_{\max}(W_i) \sigma_{\max}(h_i'(x))},
\end{equation}
where $h_i$ is the subnetwork to layer $i$ and $\sigma_{\max}(A)$ is the maximal singular value of matrix $A$ (i.e., its spectral norm).
\end{theorem}
Here, we can see that for any settings where the sum $T_1^2(x)+\cdots+T_l^2(x)$ of squared transfer functions  diminishes slower than the gradient norm $\|\nabla_\theta f_\theta(x)\|_F^2$ during training, we expect that implicit gradient regularization will apply a regularization pressure on GC. 

An immediate prediction arising under these conditions is that the size of the regularization pressure on geometric complexity will depend on the implicit regularization rate $h/B$ in Eqn.~\eqref{eq:expanded_modified_loss} (Note that the ratio $h/B$ has been linked to implicit regularization strength in many instances \cite{goyal2017accurate, smith2018dont, mccandlish1812empirical, NEURIPS2020_6e17a5fd, chao2021sobolev}). Specifically, under these conditions, networks trained with larger learning rates or smaller batch sizes, or both, will apply a stronger regularization pressure on geometric complexity,

We test this prediction by performing experiments on ResNet18 trained on CIFAR10 and show results in Figure~\ref{figure:implicit_regularization}. The results show that while all models achieve a zero train loss, consistently the higher the learning rate, the higher the test accuracy and the lower the GC. Similarly, we observe that the lower the batch size, the lower the GC and the higher the test accuracy. 

\begin{figure}[t]
  \centering
    \includegraphics[width=1\linewidth]{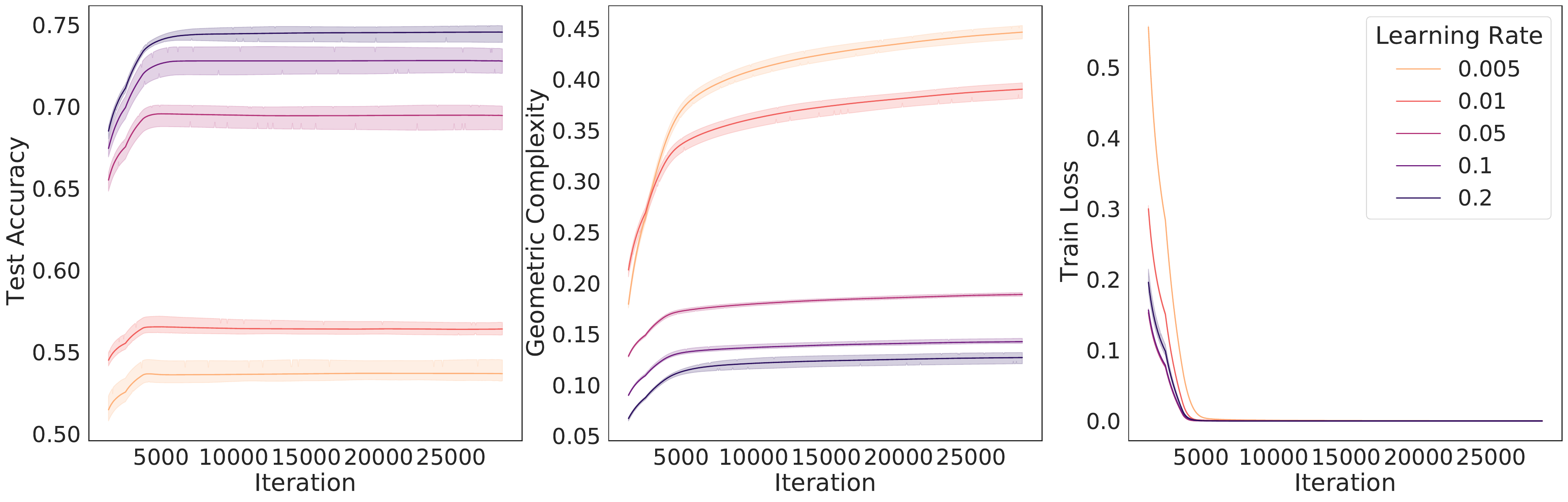}
    \includegraphics[width=1\linewidth]{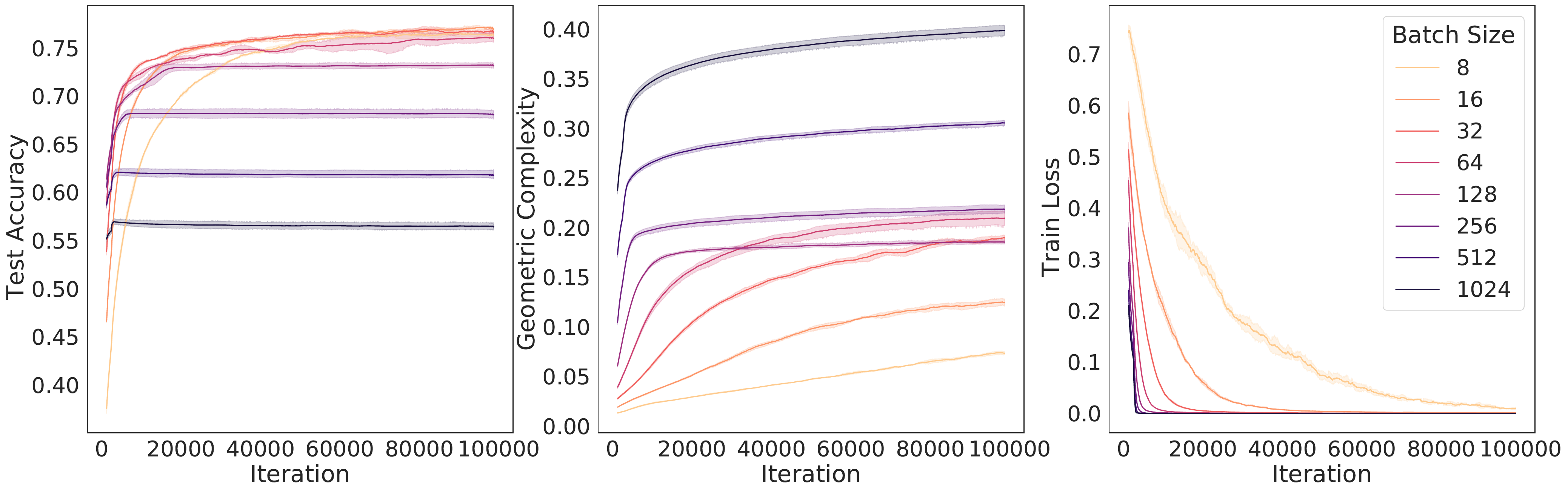}
  \caption{Impact of IGR when training ResNet18 on CIFAR10. \textbf{Top row:} As IGR increases through higher learning rates, GC decreases. \textbf{Bottom row:} Similarly, lower batch size leads to decreased GC.}
  \label{figure:implicit_regularization}
\end{figure}

The increased performance of lower batch sizes has been long studied in deep learning, under the name `the generalisation gap'~\cite{hoffer2017train,keskar2016large}. Crucially however, this gap was recently bridged~\cite{geiping2021stochastic}, showing that full batch training can achieve the same test set performance as mini-batch gradient descent. To obtain these results, they use an explicit regularizer similar to the implicit regularizer in Eqn.~\eqref{eq:expanded_modified_loss}, introduced to compensate for the diminished implicit regularization in the full batch case.  Their results further strengthen our hypothesis that implicit regularization via GC results in improved generalisation.

\section{Geometric complexity and double descent}
\label{section:double_descent}

 When complexity is measured using a simple parameter count, a double descent phenomena has been consistently observed: as the number of parameters increases, the test accuracy decreases at first, before increasing again, followed by a second descent toward a low test error value~\cite{belkin2021fear,Belkin15849,deep_double_descent}.  An excellent overview of the double descent  phenomena in deep learning can be found in ~\cite{belkin2021fear}, together with connections to smoothness as an inductive bias of deep networks and the role of optimization.
 
 To explore the double descent phenomena using GC we follow the set up introduced in~\cite{deep_double_descent}: we train multiple ResNet18 networks on CIFAR10 with increasing layer width, and show results in Fig.~\ref{fig:double_descent}. We make two observations: first, like the test loss, GC follows a double descent curve as width increases; second, when plotting GC against the test loss, we observe a U-shape curve, recovering the traditional expected behaviour of complexity measures. Importantly, we observe the connection between the generalisation properties of overparametrized models and GC: after the critical region, increase in model size leads to a decrease in GC. This provides further evidence that GC is able to capture model capacity in a meaningful way, suggestive of a reconciliation of traditional complexity theory and deep learning.

\begin{figure}[h]
\centering
\includegraphics[width=0.43\columnwidth]{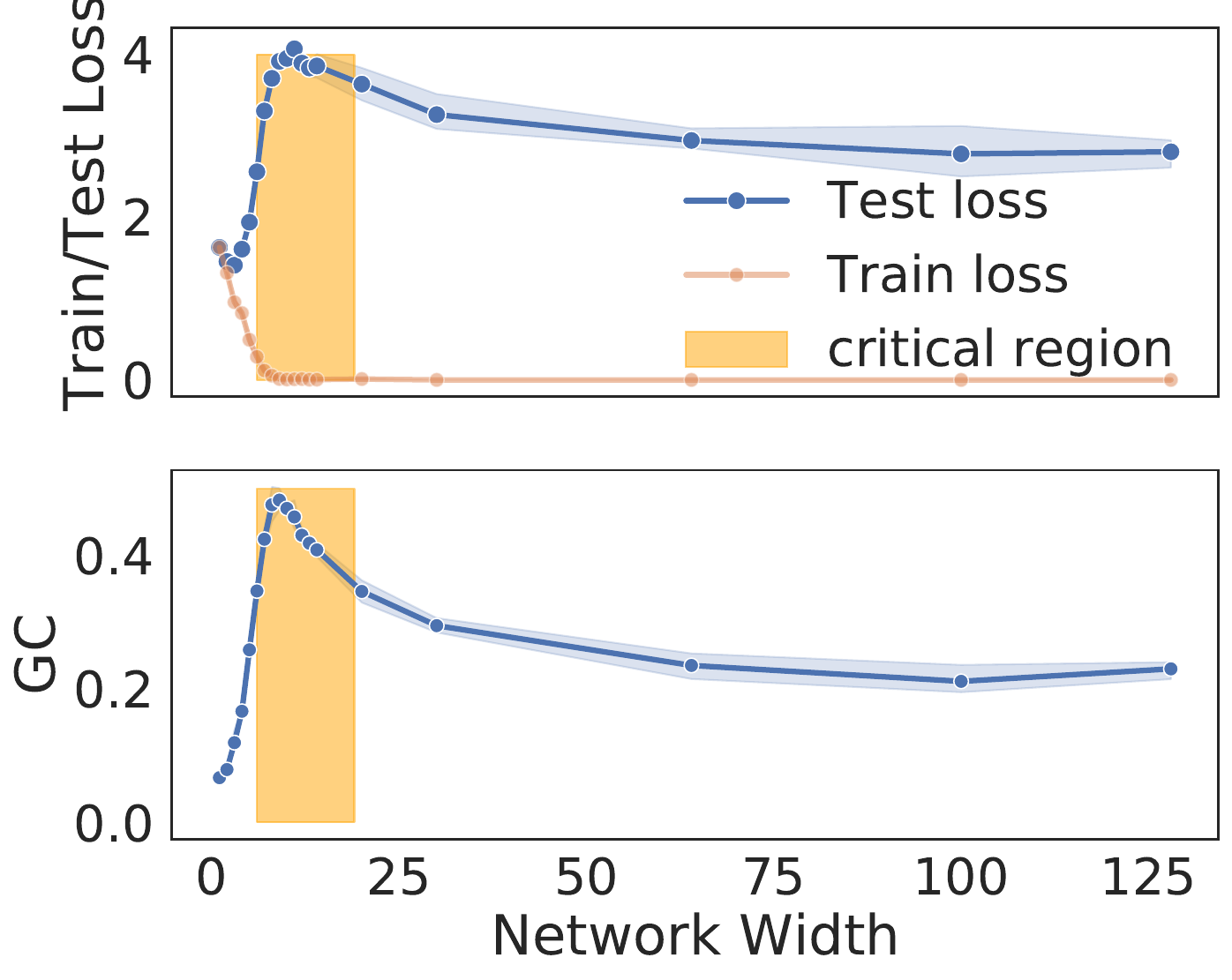}
\includegraphics[width=0.43\columnwidth]{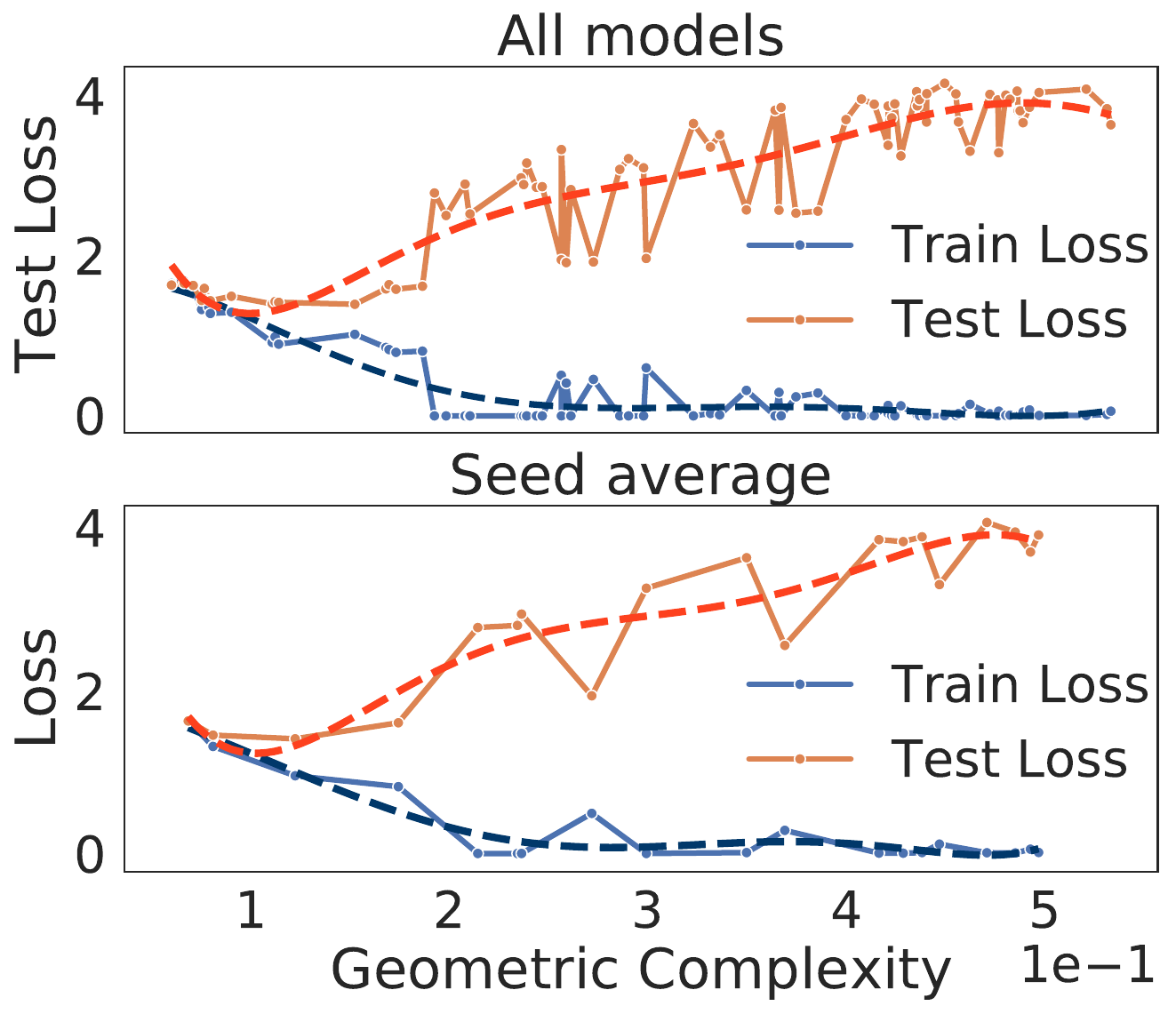}
\caption{Double descent and GC. \textbf{Left:} GC captures the double descent phenomenon. \textbf{Right:} GC captures the traditional U-shape curve, albeit with some noise, showing (top) GC vs Test Loss for all models and (bottom) GC vs Test Loss averaged across different seeds. We fit a 6 degree polynomial to the curves to showcase the trend.}
\label{fig:double_descent}
\end{figure}

\section{Related and Future Work }\label{section:related_work}

The aim of this work is to introduce a measure which captures the complexity of neural networks. There are many other approaches aiming at doing so, ranging from naive parameter count and more data driven approaches \cite{jiang2018predicting, lee2020neural, maddox2020rethinking} to more traditional measures such as the VC dimension and Rademacher complexity \cite{bartlett2002rademacher, Jiang2020Fantastic, koltchinskii2002empirical, sontag1998vc} which focus on the entire hypothesis space. \cite{neyshabur2018role} analyzes existing measures of complexity and shows that they increase as the size of the hidden units increases, and thus cannot explain behaviours observed in over-parameterized neural networks (their Figure 5).~\cite{Jiang2020Fantastic} performs an extensive study of existing complexity measures and their correlation with generalization. They find that flatness has the strongest positive connection with generalization. \cite{bartlett2017spectrally} provides a generalization bound depending on classification margins and the product of spectral norms of the model's weights and shows how empirically the product of spectral norms correlates with empirical risk. \cite{bubeck2021universal} connects Lipschitz smoothness with overparametrization, by providing a probabilistic lower bound on the model's Lipschitz constant based on its training performance and the inverse of its number of parameters.
In concurrent work \cite{gamba2022deep} discusses the connection between the Jacobian and the Hessian of the \textit{loss} with respect to inputs, namely the empirical average of $\|\nabla_x L\|_2$ and $\|\nabla_x^2 L\|_2$ over the training set and shows that they follow a double-descent curve.  \cite{novak2018sensitivity} investigates empirically a complexity measure similar to GC using the Jacobian norm of the full network (rather than the logit network) and evaluating it on test distribution (rather than the train distribution). They show a correlation between this measure and generalization in an extensive set of experiments.

As we saw in Fig.~1, the interpolating ReLU network with minimal GC is also the piecewise linear function with minimal volume or length \cite{dherin2021geometric}. In 1D this minimal function can be described only with the information given by the data points. This description with minimal information is reminiscent of the Kolmogorov complexity \cite{Schmidhuber1997Discovering} and the minimum description length \cite{NIPS1993_9e3cfc48}, and we believe that the exact relationship between these notions, GC, and minimal volume is worth investigating.
Similarly, \cite{achille2018emergence} argues that flat solutions have low information content, and for neural networks these flat regions are also the regions of low loss gradient and thus of low GC, as explained by the Transfer Theorem in Section \ref{section:implicit_regularization}. Another recent measure of complexity is the \emph{effective dimension} which is defined in relation to the training data, but is computed using the spectral decomposition of the loss Hessian ~\cite{maddox2020rethinking}, making flat regions in the loss surface also regions of low complexity w.r.t. this measure. This hints toward the effective dimension being related to GC. Note that similarly to GC, the effective dimension can also capture the double descent phenomena. The effective dimension is an efficient mechanism for model selection ~\cite{maddox2020rethinking}, which our experiments seem also to indicate may be the case for GC. Note that the generalized degrees of freedom (which considers the sensitivity of a classifier to the labels, rather than to the features as GC does) explored in \cite{grant2022predicting} in the context of deep learning also captures the double-descent phenomena.

GC is close to considerations about smoothness  \cite{pmlr-v137-rosca20a} and the Sobolev norm implicit regularization \cite{chao2021sobolev}. While GC and Lipchitz smoothness are connected, here we focused on a tractable quantity for neural networks and its connections with  existing regularizers.  Smoothness regularization has been particularly successful in the GAN literature~\cite{biggan,gulrajani2017improved,kodali2017convergence,miyato2018spectral,zhang2019self}, and the connection between Lipschitz smoothness and GC begs the question of whether their success is due to their implicit regularization of GC, but we leave the application of GC outside supervised learning for future work.

The  Dirichlet energy in  Eqn.~\eqref{eqn:classicDE} is a well-known quantity in harmonic function theory \cite{evans2010partial} and a fundamental concept throughout mathematics, physics and, more recently, 3D modeling \cite{solomon2013dirichlet}, manifold learning \cite{bronstein2017geometric}, and image processing \cite{getreuer2012rudin, rudin1992nonlinear}. Minimizers of the Dirichlet energy are harmonic and thus enjoy certain guarantees in regularity; i.e., any such solution is a smooth function. It has been demonstrated that for mean squared error regression on wide-shallow ReLU networks, gradient descent is biased towards smooth functions at interpolation \cite{jin2020implicit}. Similarly, our work suggests that neural networks, through a mechanism of implicit regularization of GC, are biased towards minimal Dirichlet energy and thus encourage smooth interpolation. It may be interesting to understand how the relationship between GC, Dirichlet energy, and harmonic theory can help futher improve the learning process, in particular in transfer learning, or in out-of-distribution generalization.

\section{Discussion}\label{section:Discussion}

In terms of \emph{limitations}, the theoretical arguments presented in this work have a focus on ReLU activations and DNN architectures, with log-likelihood losses coming from the exponential family, such as multi-dimensional regression with least-square losses, or multi-class classification with cross-entropy loss.
The experimental results are obtained on the image datasets MNIST and CIFAR using DNN and ResNet architectures.

In terms of \emph{societal impact}, we are not introducing new training methods, but focus on providing a better understanding of the impact of common existing training methods. We hope that this understanding will ultimately lead to more efficient training techniques. While existing training methods may have their own risk, we do not foresee any potential negative societal impact of this work.

In conclusion, altogether, geometric complexity provides a useful lens for understanding deep learning and sheds light into why neural networks are able to achieve low test error with highly expressive models. We hope this work will encourage further research around this new connection, and help to better understand current best practices in model training as well as discover new ones.

\begin{ack}
We would like to thank Chongli Qin, Samuel Smith, Soham De, Yan Wu, and the reviewers for helpful discussions and feedback as well as Patrick Cole, Xavi Gonzalvo, and Shakir Mohamed for their support.
\end{ack}

\bibliographystyle{plain}

\newpage

\appendix

\appendixpage
\startcontents[sections]
\printcontents[sections]{l}{1}{\setcounter{tocdepth}{2}}

\section{Proofs for Section 5}

In Section \ref{appendix:gradients} we derive that for a large class of models comprising regression models with the least-square loss and classification models with the cross-entropy loss, the norm square of the loss gradients has a very particular form (Eqn. \ref{eq:grad_square_theta}) that we need to develop the modified loss in Eqn. 6 in Step 2 of Section 5. The detail of this development is given in Section \ref{appendix:modified_loss_expansion}. In Section \ref{appendix:transfer}, we give the proof of Thm. 5.1 also needed in Step 2 of Section 5, which bounds the Frobenius norm of the network Jacobian w.r.t. the input with the network Jacobian w.r.t. the parameters. 

\subsection{Gradient structure in the exponential family} \label{appendix:gradients}

In order to be able to treat both regression and classification on the same footing, we need to introduce the notion of a transformed target $y = \phi(z)$. In regression $\phi$ is typically the identity, while for classification $\phi$ maps the labels $z\in \{1,\dots,k\}$ onto their one-hot-encoded version $\phi(z)\in \R^k$. We will also need to distinguish between the \emph{logit neural network} $f_\theta: \R^d \rightarrow \R^k$ from the \emph{response function} $g_\theta(x) = a(f_\theta(x))$ that models the transformed target. Typically, $a$ is the activation function of the last layer. For instance, for regression both the neural network and its response function coincide, $a$ being the identity, while for classification $f_\theta(x)$ are the logits and $a$ is typically a sigmoid or a softmax function. 
Now both regression and classification losses at a data point $(x,y)$ can be obtained as the negative log-likelihood of a conditional probability model
\begin{equation}
L(x, y, \theta)  =  -\log P (y | x, \theta),
\end{equation}
The conditional probability model for both regression and classification has the same structure. It is obtained by using the neural network to estimate the natural parameter vector $\eta = f_\theta(x)$ of an exponential family distribution:
\begin{equation}
P (y | x, \theta)  =  h(y) \exp\left(\langle y, f_\theta(x) \rangle - S(f_\theta(x)\right),
\end{equation}
where $S(\eta)$ is the log-partition function. For models in the exponential family, the distribution mean coincides with the gradient of the log-partition function: $E(Y | \eta) = \nabla_\eta S(\eta)$. The response function (which is the mean of the conditional distribution) is of the form $g_\theta(x) = \nabla_\eta S(f_\theta(x))$, and the last layer is thus given by $a(\eta) = \nabla_\eta S(\eta)$.
For these models, the log-likelihood loss at a data point has the simple form
\begin{equation}\label{eq:general_dl_loss}
L(x, y, \theta)  =  S(f_\theta(x)) - \langle y, f_\theta(x)  \rangle
 + \textrm{constant}.
\end{equation}
We then obtain immediately that the loss derivative w.r.t. to the parameters and w.r.t. to the input can be written as a sum of the the corresponding network derivatives weighted by the signed residuals:
\begin{eqnarray}
\nabla_\theta L(x, y, \theta) 
& = & \epsilon^1_x(\theta) \nabla_\theta f_\theta^1(x) + \cdots + \epsilon_x^k(\theta) \nabla_\theta f_\theta^k(x), \label{eq:loss_theta_gradient} \\
\nabla_x L(x, y, \theta) 
& = & \epsilon^1_x(\theta) \nabla_x f_\theta^1(x) + \cdots + \epsilon_x^k(\theta) \nabla_x f_\theta^k(x), \label{eq:loss_x_gradient}
\end{eqnarray}
where the $\epsilon^i$'s are the signed residual, that is, the $i^{th}$ components of the signed error between the response function and the transformed target:
\begin{equation}
\epsilon_x^i(\theta)  =  a^i(f_\theta(x)) - y^i.
\end{equation}
This means that the square norm of these gradients can be written as
\begin{eqnarray}
\|\nabla_\theta L(x, y, \theta) \|^2 & = & \sum_i \epsilon_x^i(\theta)^2 \|\nabla_\theta f^i_\theta(x)\|^2 + A_\theta(x, y, \theta) \label{eq:grad_square_theta}\\
\|\nabla_x L(x, y, \theta) \|^2 & = & \sum_i \epsilon_x^i(\theta)^2 \|\nabla_x f^i_\theta(x)\|^2 + A_x(x, y, \theta),
\end{eqnarray}
where $A_\theta$ and $A_x$ are the gradient alignment terms:
\begin{eqnarray}
A_\theta(x, y, \theta) & = & \sum_{i\neq j} 
\langle
\epsilon_x^i \nabla_\theta f^i_\theta(x), 
\epsilon_x^j \nabla_\theta f^j_\theta(x)
\rangle \label{eqn:A_theta}\\
A_x(x, y, \theta) & = & \sum_{i\neq j} 
\langle
\epsilon_x^i \nabla_x f^i_\theta(x), 
\epsilon_x^j \nabla_x f^j_\theta(x)
\rangle.
\end{eqnarray}

\subsubsection{Multi-class classification}

Consider the multinoulli distribution where a random variable $Z$ can take values in $k$ classes, say $z \in \{1,\dots, k\}$ with probabilities $p_1, \dots, p_k$ for each class respectively. Let the transformed target map $y=\phi(z)$ associate to a class $i$ its one-hot-encoded vector $y$ with the $i^{th}$ component equal to $1$ and all other components zero. The multinoulli density can then be written as
$
P(y) = p_1^{y_1}\cdots p_k^{y_k}
$
which can be re-parameterized, showing that the multinoulli is a member of the exponential family distribution:
\begin{eqnarray*}
P(y) 
& = & \exp(\log(p_1^{y_1}\cdots p_l^{y_l})) \\
& = & \exp\left( y_1 \log p_1 + \cdots +  y_k \log p_k \right) \\
& = & \exp\left( y_1 \log p_1 + \cdots 
+ (1 - \sum_{l=1}^{k-1} y_l) \log p_k \right) \\
& = & \exp\Bigg( 
y_1 \log \frac{p_1}{p_k} + \cdots +  y_{k-1} \log \frac{p_{k-1}}{p_k} + \cdots +\log p_k\Bigg)  
\end{eqnarray*}
We can now express the natural parameter vector $\eta$ in terms of the class probabilities:
\begin{equation}
    \eta_i : = \log \frac{p_i}{p_k} \quad \textrm{ for }  i = 1, 2, \dots, k.
\end{equation}
Note, $\eta_k = 0$. Taking the exponential of that last equation, and summing up, we obtain that $1/p_k = \sum_{i=1}^k e^{\eta_i}$, which we use to express the class probabilities in terms of the canonical parameters: 
\begin{equation}
    p_i = \frac{e^{\eta_i}}{\sum_l e^{\eta_l}}.
\end{equation}
Since $1/p_k = \sum_{i=1}^k e^{\eta_i}$, we obtain that the log-partition function is 
\begin{equation}
    S(\eta) = \log \sum_{i=1}^k e^{\eta_i},
\end{equation}
whose derivative is the softmax function:
\begin{equation}
    a(\eta) = \nabla_\eta S(\eta) = \frac{e^\eta}{\sum_{i=1}^k e^{\eta_i}}.
\end{equation}
Now, if we estimate the natural parameter $\eta$ with a neural network $\eta = f_\theta(x)$, we obtain the response 
\begin{equation}
    E(Y \, | \, x, \theta) = \nabla_\eta S(f_\theta(x)) = a(f_\theta(x)),
\end{equation}

\subsection{Computing the expanded modified loss from Section \ref{section:implicit_regularization}}\label{appendix:modified_loss_expansion}
In this section we verify that expanding the modified loss in Step 1.~of Section \ref{section:implicit_regularization}; i.e., $\widetilde L_B = L_B + \dfrac{h}{4}\|\nabla L_B\|^2$, yields Eqn.(\ref{eq:expanded_modified_loss}). Namely, we show that
\begin{equation}
    \widetilde L_B = L_B + \dfrac{h}{4B}\left(\dfrac{1}{B}\sum_{x,i}\epsilon^i_x(\theta)^2 \|\nabla_{\theta}f^i_{\theta}(x)\|^2 \right)  + \frac h{4} A_B(\theta) + \frac h4 C_B(\theta), 
\end{equation}
where
\begin{eqnarray}
A_B(\theta) & = &\frac 1{B^2} \sum_{x\in B} \sum_{i\neq j} 
\langle
\epsilon_x^i \nabla_\theta f^i_\theta(x), 
\epsilon_x^j \nabla_\theta f^j_\theta(x)
\rangle \\
C_B(\theta) & = &
\frac 1{B^2}\sum_{(x, y)\neq (x', y')}\left\langle \nabla_\theta L(x,y, \theta), \nabla_\theta L(x',y', \theta) \right\rangle.
\end{eqnarray}

It suffices to show
\begin{eqnarray}
    \|\nabla_{\theta} L_B(\theta)\|^2 &=& \dfrac{1}{B^2}\sum_{x,i}\epsilon^i_x(\theta)^2 \|\nabla_{\theta}f^i_{\theta}(x)\|^2 \\ 
    && + ~~ \frac 1{B^2} \sum_{x\in B} \sum_{i\neq j} 
\langle
\epsilon_x^i \nabla_\theta f^i_\theta(x), 
\epsilon_x^j \nabla_\theta f^j_\theta(x)
\rangle \\
    && + ~~ \frac 1{B^2}\sum_{(x, y)\neq (x', y')}\left\langle \nabla_\theta L(x,y, \theta), \nabla_\theta L(x',y', \theta) \right\rangle.
\end{eqnarray}

Indeed this follows directly by computation. Firstly, note that 
\begin{eqnarray}
\|\nabla_{\theta} L_B(\theta)\|^2 
&=& 
\left\langle 
\nabla_{\theta} \dfrac{1}{B}\sum_{x\in B} L(x,y,\theta) ,
\nabla_{\theta} \dfrac{1}{B}\sum_{x'\in B} L(x',y',\theta)
\right\rangle \\
&=& 
\dfrac{1}{B^2}\sum_{x\in B} \left\langle \nabla_{\theta} L(x,y,\theta), \nabla_{\theta} L(x,y,\theta)\right\rangle \\
&&~+~ 
\dfrac{1}{B^2}\sum_{(x, y)\neq (x', y')} \left\langle \nabla_{\theta} L(x,y,\theta), \nabla_{\theta} L(x',y',\theta) \right\rangle \\
&=& 
\dfrac{1}{B^2}\sum_{x\in B} \|\nabla_{\theta} L(x,y,\theta)\|^2 \\
&&~+~ 
\dfrac{1}{B^2}\sum_{(x, y)\neq (x', y')} \left\langle \nabla_{\theta} L(x,y,\theta), \nabla_{\theta} L(x',y',\theta) \right\rangle.
\end{eqnarray}

Now, by (\ref{eq:grad_square_theta}) and (\ref{eqn:A_theta}), the last equality can be simplified so that
\begin{eqnarray}
\|\nabla_{\theta} L_B(\theta)\|^2 
&=&
\dfrac{1}{B^2} \sum_{x\in B} \sum_i \epsilon_x^i(\theta)^2 \|\nabla_\theta f^i_\theta(x)\|^2 \\
&&~+~ \dfrac{1}{B^2} \sum_{x\in B} \sum_{i\neq j} 
\langle
\epsilon_x^i \nabla_\theta f^i_\theta(x), 
\epsilon_x^j \nabla_\theta f^j_\theta(x)
\rangle \\
&&~+~ 
\dfrac{1}{B^2}\sum_{(x,y) \neq (x',y')} \left\langle \nabla_{\theta} L(x,y,\theta), \nabla_{\theta} L(x',y',\theta) \right\rangle.
\end{eqnarray}
This completes the computation. 

\subsection{The Transfer Theorem} \label{appendix:transfer}

To frame the statement and proof of the Transfer Theorem, let's begin by setting up and defining some notation. Consider a deep neural network $f_{\theta} : \mathbb{R}^d \to \mathbb{R}^k$ parameterized by $\theta$ and consisting of $l$ layers stacked consecutively. We can express $f_{\theta}$ as 
\begin{equation}
    f_{\theta}(x) = f_l \circ f_{l-1} \circ \cdots \circ f_1(x),
\end{equation}
where $f_i(z) = a_i(w_i z + b_i)$ denotes the $i$-th layer of the network defined by the weight matrix $w_i$, the bias $b_i$ and the layer activation function $a_i$ which acts on the output $z$ of the previous layer. In this way, we can write $\theta = (w_1, b_1, w_2, b_2, \dots, w_l, b_l)$ to represent all the learnable parameters of the network. 

Next, let $h_i(x)$ denote the sub-network from the input $x \in \mathbb{R}^d$ up to and including the output of layer $i$; that is, 
\begin{equation}
    h_i(x) = a_i(w_i h_{i-1}(x) + b_i), \text{ for } i = 1, 2, \dots, l
\end{equation}
where we understand $h_0(x)$ to be $x$.
Note that, by this notation, $h_l(x)$ represents the full network; i.e., $h_l(x) = f_{\theta}(x)$. 

At times it will be convenient to consider $f_{\theta}(x)$ as dependent only on a particular layer's parameters, for example $w_i$ and $b_i$, and independent of all other parameter values in $\theta$. In this case, we will use the notation $f_{w_i}$ or $f_{b_i}$ to represent $f_{\theta}$ as dependent only on $w_i$ or $b_i$ (resp.). Using the notation above, for each weight matrix $w_1, w_2, \dots, w_l$, we can rewrite $f_{w_i}(x)$ as (similarly, for $f_{b_i}(x)$)
\begin{equation}\label{equation:f_in_terms_of_g}
    f_{w_i}(x) = g_i(w_i h_{i-1}(x) + b_i), \text{ for } i = 1, 2, \dots, l
\end{equation}
where each $g_i(z)$ denotes the remainder of the full network function $f_{\theta}(x)$ following the $i$-th layer. That is to say, $g_i$ represents the part of the network deeper than the $i$-th layer (i.e., from layer $i$ to the output $f_{\theta}(x)$) and $h_i$ represents the part of the network shallower than the $i$-th layer (i.e., from the input $x$ up to layer $i$). 

We are now ready to state our main Theorem:

\begin{theorem} [Transfer Theorem]
\label{theorem:transfer}
Consider a network $f_\theta:\R^d\rightarrow \R^k$ with $l$ layers parameterized by $\theta = (w_1, b_1, \dots, w_l, b_l)$, then we have the following inequality
\begin{equation}\label{thm:transfer_inequality}
\|\nabla_x f_\theta(x)\|^2_F  \leq  \frac{\|\nabla_\theta f_\theta(x)\|^2_F}{T_1^2(x, \theta)+\cdots+T_l^2(x, \theta)},
\end{equation}
where $T_i(x, \theta)$ is the \it{transfer function} for layer $i$ given by 
\begin{equation}
T_i(x, \theta) 
 = 
\frac{1}{\sqrt{\min(d,k)}}
\frac{\sqrt{1 + \|h_{i-1}(x)\|^2_2}}{\sigma_{\max}(w_i) \sigma_{\max}(\nabla_xh_{i-1}(x))},
\end{equation}
where $h_i$ is the subnetwork to layer $i$ and $\sigma_{\max}(A)$ is the maximal singular value of matrix $A$ (i.e., its spectral norm).

\end{theorem}

The proof of Theorem \ref{theorem:transfer}, inspired by the perturbation argument in  \cite{SeongLKHK18}, follows from the two following Lemmas. The idea is to examine the layer-wise structure of the neural network to compare the gradients $\nabla_{w_i} f_\theta$ with the gradients $\nabla_x f_\theta$. Due to the nature of  $f_{\theta}(x)$ and its dependence on the inputs $x$ and parameters $\theta$, we show that at each layer $i$, a small perturbation of the inputs $x$ of the model function $f_{\theta}(x)$ transfers to a small perturbation of the weights $w_i$. 

\begin{lemma}\label{lemma:grad_w} Let $f_{\theta}: \mathbb{R}^d \to \mathbb{R}^k$ represent a deep neural network consisting of $l$ consecutive dense layers. Using the notation above, for $i = 1, 2,\dots, l$, we have
\begin{equation}\label{equation:weight_derivative}
    \|\nabla_x f_\theta(x)\|_2^2 \left(\frac{\|h_{i-1}(x)\|_2}{\|w_i\|_2\|\nabla_x h_{i-1}(x)\|_2}\right)^2 \leq \|\nabla_{w_i}f_{\theta}(x)\|_2^2,
\end{equation}
where $\|\cdot\|_2$ denotes the $L^2$ operator norm when applied to matrices and the $L^2$ norm when applied to vectors.
\end{lemma}

\begin{proof}
For $i=1, 2,\dots, l$ and following the notation above, the model function as it depends on the weight matrix $w_i$ of layer $i$ can be written as  $f_{w_i}(x) = g_i(w_ih_{i-1}(x) + b_i)$. Now consider a small perturbation $x+\delta x$ of the input $x$.
There exists a corresponding perturbation $w_i + u(\delta x)$ of the weight matrix in layer $i$ such that
\begin{equation}\label{equation:equivalence}
f_{w_i}(x + \delta x) = f_{w_i + u(\delta x)}(x).
\end{equation}
To see this, note that for sufficiently small $\delta x$ we can identify $h_i(x+\delta x)$ with its linear approximation around $x$  so that $h_{i}(x + \delta x) = h_{i}(x) + \nabla_x h_{i}(x)\delta x$, where $\nabla_x h_{i}:\mathbb{R}^d \to \mathbb{R}^{|h_i|}$ is the total derivative of $h_i: \mathbb{R}^d \to \mathbb{R}^{|h_i|}$. Thus, by representing $f_{w_i}$ as in (\ref{equation:f_in_terms_of_g}), we have
\begin{eqnarray*}
    f_{w_i}(x + \delta x) &=& g_i(w_ih_{i-1}(x + \delta x) + b_i) \\
    &=& g_i(w_ih_{i-1}(x) + w_i\nabla_xh_{i-1}(x)\delta x + b_i).
\end{eqnarray*}
Similarly for $f_{w_i + u(\delta x)}(x)$, we have
\begin{eqnarray*}
    f_{w_i + u(\delta x)}(x) &=& g_i((w_i + u(\delta x))h_{i-1}(x) + b_i) \\
    &=& g_i(w_ih_{i-1}(x) + u(\delta x)h_{i-1}(x) + b_i).
\end{eqnarray*}

Thus Eqn.~(\ref{equation:equivalence}) is satisfied provided $u(\delta x)h_{i-1}(x) = w_i\nabla_x h_{i-1}(\delta x)$. Using the fact that $h_i(x)h_i(x)^T = \|h_i(x)\|_2^2$ and rearranging terms, we get

\begin{equation}\label{equation:perturbation}
    u(\delta x) = \frac{w_i (\nabla_xh_{i-1}(x)\delta x) h_{i-1}(x)^T}{\|h_{i-1}(x)\|_2^2}.
\end{equation}

Defining $u(\delta x)$ as in Eqn.~(\ref{equation:perturbation}) and taking the derivative of both sides of Eqn.~(\ref{equation:equivalence}) with respect to $\delta x$ at $\delta x = 0$ via the chain rule, we get (since by Eqn.~(\ref{equation:perturbation}) we have $u(\delta x)$ is linear in $\delta x$)

\begin{equation}
    \nabla_x f_{w_i}(x) = \nabla_{w_i} f_{\theta}(x) \frac{(w_i \nabla_xh_{i-1}(x))h_{i-1}^T(x)}{\|h_{i-1}(x)\|_2^2}.
\end{equation}

Finally, taking the square of the $L^2$ operator norm $\|\cdot\|_2$ on both sides, we have
\begin{eqnarray*}
    \|\nabla_x f_{w_i}(x) \|_2^2 &=& \left\|\nabla_{w_i} f_{\theta}(x) \dfrac{w_i \nabla_xh_{i-1}(x)h_{i-1}^T(x)}{\|h_{i-1}(x)\|_2^2}\right\|_2^2 \\
    &\leq& \|\nabla_{w_i} f_{\theta}(x)\|_2^2 \dfrac{\|w_i\|_2^2 \|\nabla_xh_{i-1}(x)\|_2^2}{\|h_{i-1}(x)\|_2^2}. 
\end{eqnarray*}
Rearranging the terms and since $\nabla_x f_{w_i} = \nabla_x f_{\theta}$ we get \eqref{equation:weight_derivative} which completes the proof.
\end{proof}

Following the same argument, we can also prove the corresponding lemma with respect to the derivative of the biases $b_i$ at each layer:

\begin{lemma}\label{lemma:grad_wb}
Let $f_{\theta}: \mathbb{R}^d \to \mathbb{R}^k$ represent a deep neural network consisting of $l$ consecutive dense layers. Using the notation above, for $i = 1, 2,\dots, l$, we have
\begin{equation}\label{equation:bias_derivative}
    \|\nabla_x f_\theta(x)\|_2^2 \left(\frac{1}{\|w_i\|_2\|\nabla_x h_{i-1}(x)\|_2}\right)^2 \leq \|\nabla_{b_i}f_{\theta}(x)\|_2^2,
\end{equation}
where $\|\cdot\|_2$ denotes the $L^2$ operator norm when applied to matrices and the $L^2$ norm when applied to vectors.
\end{lemma}

\begin{proof}
The proof is similar in spirit to the proof of Lemma \ref{lemma:grad_w}. Namely, for small perturbations $x+ \delta x$ of the input $x$, we verify that we can find a corresponding small perturbation $b_i + u(\delta x)$ of the bias $b_i$ so that $f_{b_i}(x+ \delta x) = f_{b_i + u(\delta x)}(x)$. From Eqn.~(\ref{equation:f_in_terms_of_g}), this time in relation to the bias term, we can write $f_{b_i}(x) = g_i(w_ih_{i-1}(x) + b_i)$. Then, taking a small perturbation $x+\delta x$ of the input $x$ and simplifying as before, we get
\begin{equation}
    u(\delta x) = w_i \nabla_xh_{i-1}(x)\delta x.
\end{equation}
As before, taking this as our definition of $u(\delta x)$, and taking the derivative w.r.t.~$\delta(x)$ as before via the chain rule, we obtain
\begin{equation}
    \nabla_x f_{b_i}(x) = \nabla_{b_i}f_{\theta}(x) w_i\nabla_xh_{i-1}(x).
\end{equation}
Again, after taking the square of the $L^2$ operator norm on both sides, and moving the terms around we get
\begin{equation}
    \|\nabla_x f_\theta(x)\|_2^2 \left(\frac{1}{\|w_i\|_2\|\nabla_xh_{i-1}(x)\|_2}\right)^2 \leq \|\nabla_{b_i}f_{\theta}(x)\|_2^2.
\end{equation}
\end{proof}

It remains to prove Theorem \ref{theorem:transfer}:
\begin{proof}[Proof of Theorem \ref{theorem:transfer}]
Recall, that $f_{\theta}(x)$ represents a deep neural network consisting of $l$ layers and parameterized by $\theta = (w_1, b_1, w_2, b_2, \dots, w_l, b_l)$. Furthermore, we use the notation $f_{w_i}$ (resp.~$f_{b_i}$) to denote $f_{\theta}$ as dependent only on $w_i$ (resp.~$b_i$); i.e., all other parameter values are considered constant. Let  $\|\cdot\|_F$ denote the Frobenius norm. For this model structure, by Pythagoras's theorem, it follows that
\begin{equation}\label{equation:decomposition}
\|\nabla_\theta f_\theta(x)\|_F^2 = 
\|\nabla_{w_1} f_{\theta}(x)\|_F^2 + \|\nabla_{b_1} f_{\theta}(x)\|_F^2 
+ \cdots + 
\|\nabla_{w_l} f_{\theta}(x)\|_F^2 + \|\nabla_{b_l} f_{\theta}(x)\|_F^2.
\end{equation}

The remainder of the proof relies on the following general property of matrix norms:  Given a matrix $A \in \mathbb{R}^{m\times n}$, let $\sigma_{\max}(A)$ represent the largest singular value of $A$. The Frobenius norm $\|\cdot\|_F$ and the $L^2$ operator norm $\|\cdot\|_2$ are related by the following inequalities
\begin{equation}\label{equation:norm_inequality}
    \|A\|_2^2 = \sigma_{\max}^2(A) \leq \|A\|_F^2 = \sum_{i=1}^{\min(m, n)}\sigma_k^2 \leq \min(m,n)\cdot\sigma^2_{\max}(A) = \min(m, n)\cdot\|A\|^2_2
\end{equation}
where $\sigma_i(A)$ are the singular values of the matrix $A$. Considering $\nabla_xf_{\theta}$ as a map from $\mathbb{R}^d$ to $\mathbb{R}^k$, by (\ref{equation:norm_inequality}) it follows that
\begin{equation}\label{equation:norm_relation}
    \|\nabla_x f_{\theta}(x)\|_2^2 
    \quad\geq\quad
    \dfrac{1}{\min(d,k)}\|\nabla_x f_{\theta}(x)\|_F^2.
\end{equation}

By Lemma \ref{lemma:grad_w} and Lemma \ref{lemma:grad_wb}, we have that for $i=1,2,\dots,l$
\begin{equation}\label{equation:decomp_single_term}
\|\nabla_{w_i} f_{\theta}(x) \|_2^2 + \|\nabla_{b_i} f_{\theta}(x) \|_2^2 
\quad\geq\quad
\|\nabla_x f_{\theta}(x)\|_2^2\left(\dfrac{1 + \|h_{i-1}(x)\|_2^2}{\|w_i\|_2^2\|\nabla_xh_{i-1}(x)\|_2^2} \right).
\end{equation}
Therefore, we can re-write Eqn.~(\ref{equation:decomposition}) using (\ref{equation:norm_inequality}), (\ref{equation:norm_relation}) and (\ref{equation:decomp_single_term}) to get
\begin{eqnarray*}
    \|\nabla_\theta f_\theta(x)\|_F^2 &=& \sum_{i=1}^l\left( \|\nabla_{w_i} f_{\theta}(x)\|_F^2 + \|\nabla_{b_i} f_{\theta}(x)\|_F^2 \right) \\
    &\geq& \sum_{i=1}^l \|\nabla_{w_i} f_{\theta}(x)\|_2^2 + \|\nabla_{b_i} f_{\theta}(x)\|_2^2  \\
    &\geq&  \sum_{i=1}^l \|\nabla_x f_{\theta}(x)\|_2^2  \left(\dfrac{1 + \|h_{i-1}(x)\|_2^2}{\|w_i\|_2^2\|\nabla_xh_{i-1}(x)\|_2^2} \right) \\
    &\geq&  \|\nabla_x f_{\theta}(x)\|_F^2 \sum_{i=1}^l\dfrac{1}{\min(d,k)} \left(\dfrac{1 + \|h_{i-1}(x)\|_2^2}{\sigma^2_{\max}(w_i)\sigma^2_{\max}(\nabla_xh_{i-1}(x))} \right).
\end{eqnarray*}
For $i=1,2,\dots,l$, define
\begin{equation}
    T_i(x, \theta) := \dfrac{1}{\sqrt{\min(d,k)}}\dfrac{\sqrt{1 + \|h_{i-1}(x)\|_2^2}}{\sigma_{\max}(w_i)\sigma_{\max}(\nabla_xh_{i-1}(x))}.
\end{equation}
We call $T_i(x, \theta)$ the {\em transfer function for layer $i$} in the network.

Then, rearranging terms in the final inequality above, we get
\begin{equation}
    \|\nabla_x f_{\theta}(x)\|_F^2 \leq \dfrac{\|\nabla_\theta f_\theta(x)\|_F^2}{T_1^2(x, \theta) + \cdots + T_l^2(x, \theta)}.
\end{equation}
This completes the proof.
\end{proof}

\section{Experiment details}\label{appendix:experiments}

\subsection{Figure 1: Motivating 1D example} 

We trained a ReLU MLP with 3 layers of 300 neurons each to regress 10 data points on an exact parabola $(x, x^2)\in \R^2$ where $x$ ranges over 10 equidistant points in the interval $[-1,1]$. We use full batch gradient descent with learning rate $0.02$ for $100000$ steps/epochs. We plotted the model function across the $[-1, 1]$ range and computed its geometric complexity over the dataset at step 0 (initialization), step 10, step 1000, and step 10000 (close to interpolation). The network was randomly initialized using the standard initialization (i.e. the weights were sampled from a truncated normal distribution with variance inversely proportional to the number of input units and the bias terms were set to zero). This model was trained five separate times using five different random seeds. Each line marks the mean of the five runs and the shaded region is 95\% confidence interval over these five seeds.

\subsection{Figure 2: Geometric complexity and initialization} \label{section:figure_2}

\paragraph{Left and Middle:} We initialized several ReLU MLP's $f_{\theta_0}: \mathbb{R}^d \to \mathbb{R}^k$ with large input and output: $d=224 \times 224 \times 3$ and $k=1000$, with 500 neurons per layer, and with a varying number of layers $l\in [1, 2, 4, 8, 16, 32, 64]$. 
We used the standard initialization scheme: we sample them from a truncated normal distribution with variance inversely proportional to the number of input units and the bias terms were set to zero.
We measured and plotted the following quantities
\begin{eqnarray}
f_{\textrm{mean}}(x) & = & \textrm{mean}\left\{f_{\theta_0}^i(P_1 + (P_2 - P_1)x),\quad i=1,\dots, 1000\right\},  \\
f_{\textrm{max}}(x) & = & \textrm{max}\left\{(|f_{\theta_0}^i(P_1 + (P_2 - P_1)x))| ,\quad i=1,\dots, 1000\right\},
\end{eqnarray}
with $x$ ranging over 50 equidistant points in the interval $[0,1]$. The points $P_1$  and $P_2$ where chosen to be the two diagonal points $(-1, \dots, -1)$ and $(1, \dots, 1)$, respectively, of the normalized data hyper-cube $[-1,1]^d$. Each line marks the mean of the 5 runs and the shaded region is 95\% confidence interval over these five seeds.

What we observe with the standard initialization scheme also persists with the Glorot initialization scheme, where the biases are set to zero, and the weight matrices parameters are sampled from the uniform distribution on $[-1,1]$ and scaled at each layer by
$
\sqrt{6/d_l + d_{l-1}}
$,
where $d_l$ is the number of units in layer $l$. We report this in Fig. \ref{fig:initialization_glorot} below.

\begin{figure}[h]
  \centering
  \includegraphics[width=0.7
  \linewidth]{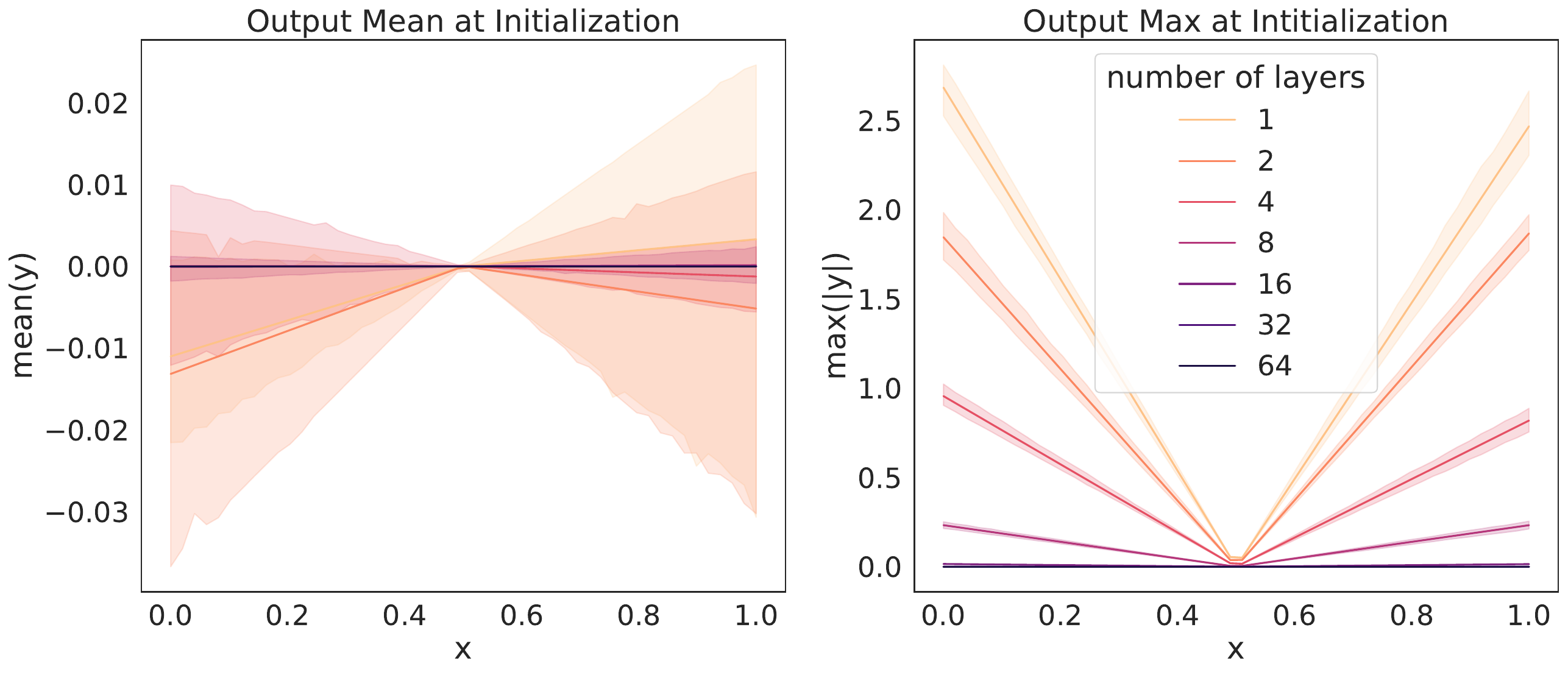}
  \caption{Additional plots to complement Fig. 2: Deeper neural Relu networks initialize closer to the zero function with the Glorot scheme on normalized data. We repeat the setup of Fig. 2 Left and Middle described in Section \ref{section:figure_2} but with the Glorot scheme instead of the standard scheme.}
  \label{fig:initialization_glorot}
\end{figure}

\paragraph{Right:} 
We measured the geometric complexity $\langle f_{\theta_0},\, D\rangle_G$ at initialization for ReLU MLP's $f_{\theta_0}: \mathbb{R}^d \to \mathbb{R}^k$ with large input and output: $d=224 \times 224 \times 3$ and $k=1000$, with 500 neurons per layer, and with a varying number of layers $l\in [1, 2, 4, 8, 16, 32, 64]$. The geometric complexity was computed over a dataset $D$ of 100 points sampled uniformly from the normalized data hyper-cube $[-1, 1]^d$. For each combination of activation and initialization in $\{\textrm{ReLU}, \textrm{sigmoid}\} \times \{\textrm{standard}, \textrm{Glorot}\}$ we repeated the experiment 5 times with different random seeds. The Glorot initialization scheme is the one described in the paragraph above. For the standard initialization we set the biases to zero, and initialized the weight matrices parameters from a normal distribution truncated to the range $[-2, 2]$ and rescaled using $1/\sqrt{d_{l-1}}$ at each layer. We plotted the mean geometric complexity with error bars representing the 95\% confidence interval over the 5 random seeds. The error bars are tiny in comparison of the plotted quantities and are therefore not visible on the plot. 

\subsection{Figure 3: Geometric complexity and explicit regularization}

\paragraph{Left:}
We trained a ResNet18 on CIFAR10 three times with different random seeds with a learning rate of 0.02, batch size of 512, for 10000 steps for each combination of regularization rate and regularization type in 
$
\{\textrm{L2}, \textrm{Spectral}, \textrm{Flatness}\} 
\times 
[0, 0.01, 0.025, 0.05, 0.075, 0.1].
$
We measured the geometric complexity at time of maximum test accuracy for each of these runs and plotted the mean  with 95\% confidence interval over the 3 random seeds. 
For the L2 regularization we added the sum of the parameter squares to the loss multiplied with the regularization rate. For the spectral regularization, we followed the procedure described in \cite{Yoshida2017SpectralNR} by adding to the loss the penalty $(\alpha/2)\sum_i \sigma_{\max}(W_i)^2$ where $W_i$ is either the layer weight matrix for a dense layer or, for a convolution layer, the matrix  of shape $b \times ak_wk_h$ obtained by reshaping the convolution layer with $a$ input channels, $b$ output channels, and a kernel of size $k_w \times k_h$. For the flatness regularization, we added to the batch loss $L_B$ at each step the norm square of the batch loss gradient $\|\nabla_\theta L_B(\theta)\|^2$ multiplied by the regularization rate.

\begin{figure}[h]
  \centering
  \includegraphics[width=0.4
  \linewidth]{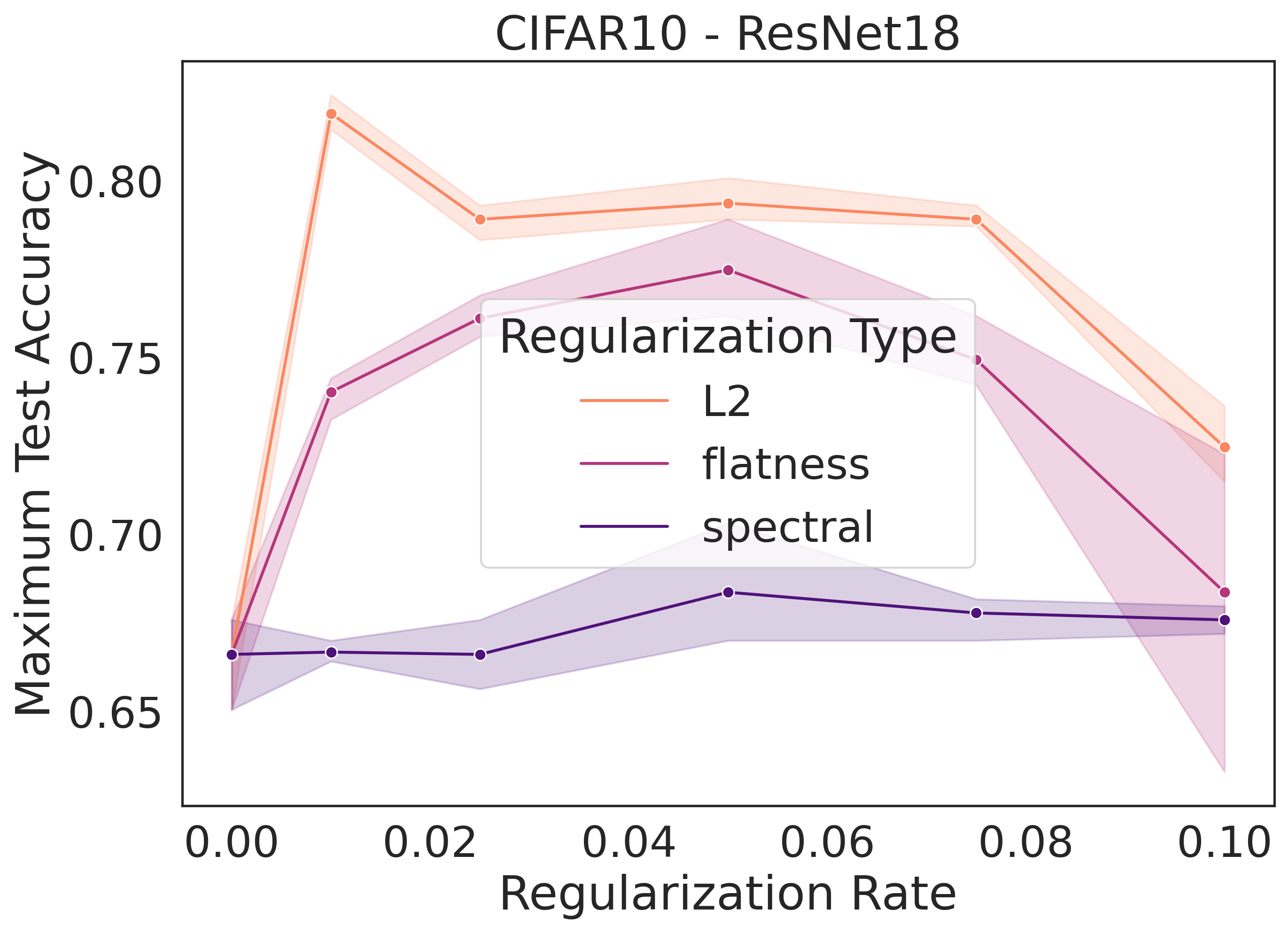}
   \caption{Additional plot to complement Fig. \ref{fig:explicit_regularization}: Maximum test accuracy recorded for different levels of explicit regularization.}
  \label{fig:explicit_regularization_test_accuracy}
\end{figure}

\begin{figure}[h]
  \centering
  \includegraphics[width=1
  \linewidth]{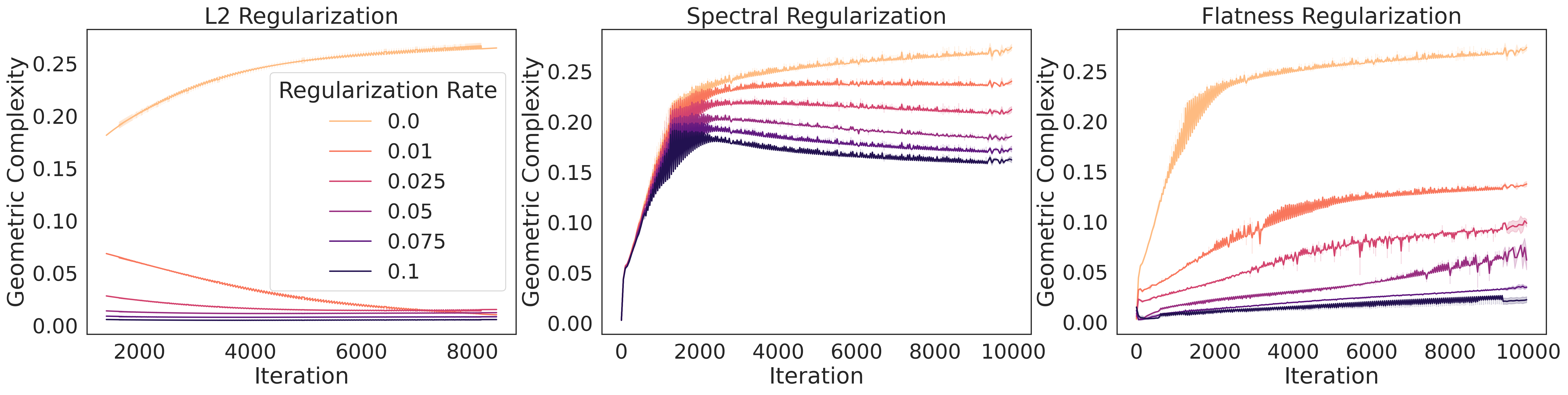}
  \caption{Additional plots to complement Fig. \ref{fig:explicit_regularization}: GC plotted against training iterations for different explicit regularization types and rates for ResNet18 trained on CIFAR10.}
  \label{fig:explicit_regularization_learning_curves}
\end{figure}

\begin{figure}[h]
  \centering
  \includegraphics[width=0.5
  \linewidth]{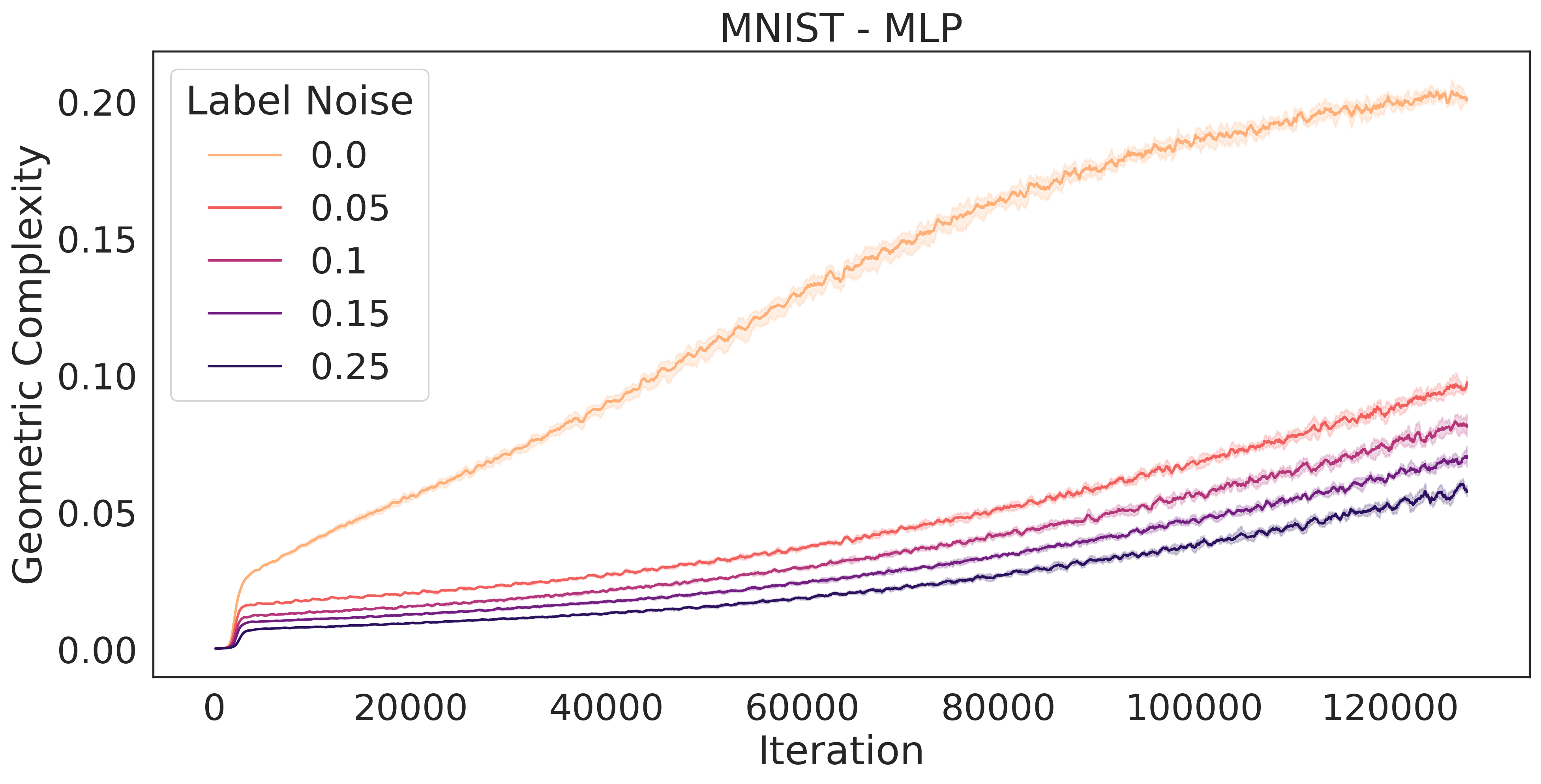}
  \includegraphics[width=0.5
  \linewidth]{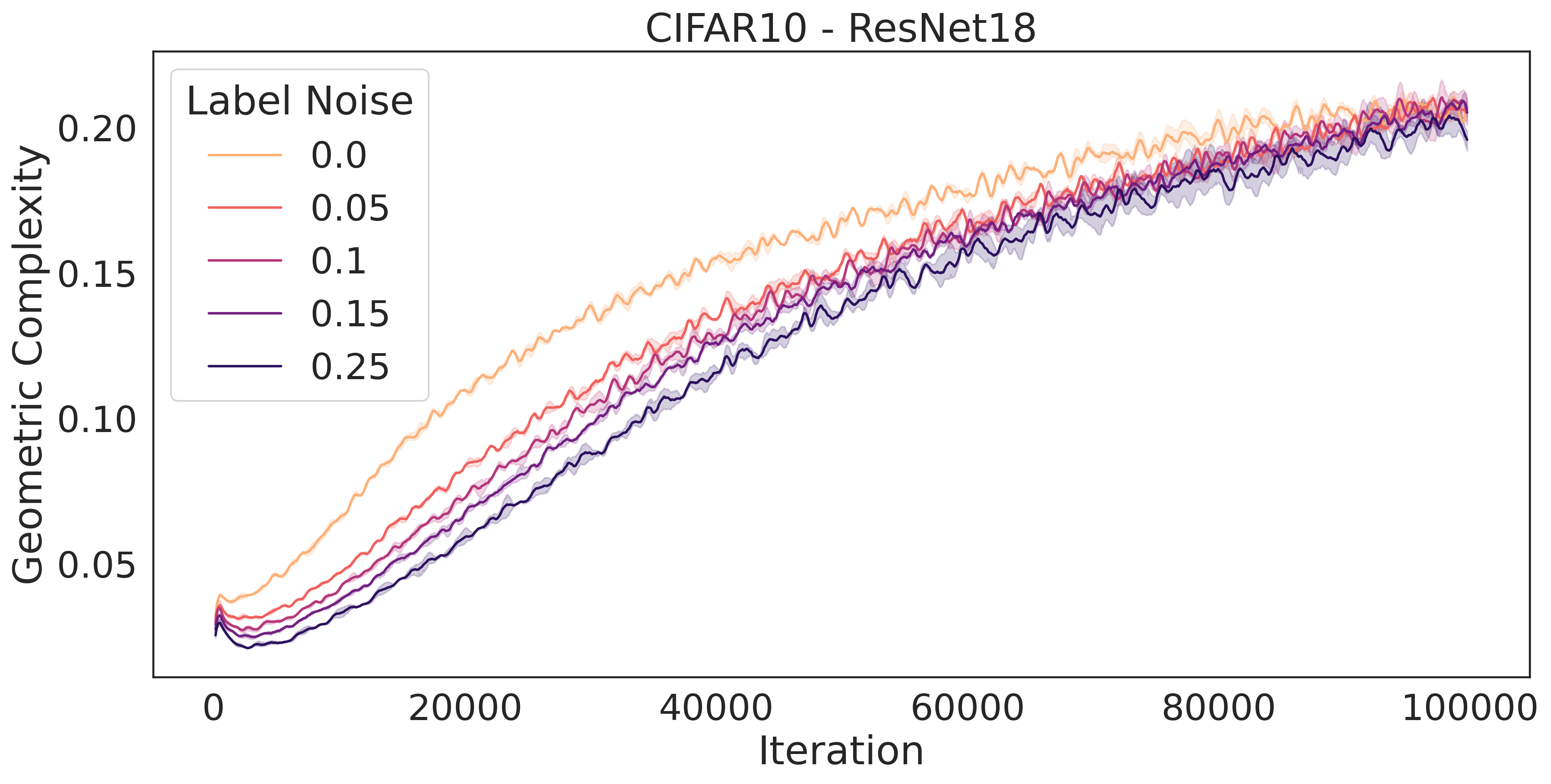}
  \caption{Additional plots to Fig. \ref{fig:explicit_regularization} experiments: GC plotted against training iterations for different label noise proportions.}
  \label{fig:label_noise_learning_curve_mnist}
\end{figure}

\paragraph{Middle:}
We trained a ResNet18 on CIFAR10 with learning rate 0.005 and batch size 16 for 100000 steps with varying proportion of label noise. For each label noise proportion $\alpha \in [0, 0.05, 0.1, 0.15, 0.25]$, we trained three times with different random seeds. We plotted the mean geometric complexity at time of maximum test accuracy as well as the 95\% confidence interval over the three runs. The label noise was created by mislabelling $\alpha$ \% of the true labels before training. 

\paragraph{Right:}
We trained a MLP with six layers of 214 neurons on MNIST with learning rate 0.005 and batch size 32 for 128000 steps with varying proportion of label noise. For each label noise proportion $\alpha \in [0, 0.05, 0.1, 0.15, 0.25]$, we trained the neural network 5 times with a different random seed. We plotted the mean geometric complexity at time of maximum accuracy as well as the 95\% confidence interval over the 5 runs. The label noise was created by mislabelling $\alpha$ \% of the true labels before training. 

\subsection{Figure 4: Geometric complexity and implicit regularization}

\paragraph{Top row (varying learning rates):}
We trained a ResNet18 on CIFAR10 with batch size 512, for 30000 steps with varying learning rates. For each learning rate in $[0.005, 0.01, 0.05, 0.1, 0.2]$, we trained the neural network three times with a different random seed. We plotted the learning curves for the test accuracy, the geometric complexity, and the training loss, where the solid lines represent the mean of these quantities over the three runs and the shaded area represents the 95\% confidence interval over the three runs. We applied a smoothing over 50 steps for each run, before computing the mean and the confidence interval of the three runs. 

\paragraph{Bottom row (varying batch sizes):}

We trained a ResNet18 on CIFAR10 with learning rate 0.2, for 100000 steps with varying batch sizes. For each batch size in $[8, 16, 32, 64, 128, 256, 512, 1024]$, we trained the neural network three times with a different random seed. We plotted the learning curves for the test accuracy, the geometric complexity, and the training loss, where the solid lines represent the mean of these quantities over the three runs and the shaded area represents the 95\% confidence interval over the three runs. We applied a smoothing over 50 steps for each run, before computing the mean and the confidence interval of the three run.

\subsection{Figure 5: Geometric complexity and double-descent}

We train a ResNet18 on CIFAR10 with learning rate 0.8, batch size 124, for 100000 steps, with 18 different network widths. For each network width in $[1, 2, 3, 4, 5, 6, 7, 8, 9, 10, 11, 12, 13, 14, 20, 30, 64, 100, 128]$ we train three times with a different random seed.

\paragraph{Left:} We measure the geometric complexity as well as the test loss at the end of training. We plot the mean and the 95\% confidence interval of both quantities over the three runs against the network width. The critical region in yellow identified experimentally in \cite{deep_double_descent} indicates the transition between the under-parameterized regime and the over-parameterized regime, where the second descent starts. 

\paragraph{Right:} We plot the test loss measured at the end of training against the geometric complexity at the end of training. The top plot shows every model width for every seed, while the bottom plot shows averages of these quantities over the three seeds. We then fit the data with a polynomial of degree six to sufficiently capture any high order relationship between the test loss and geometric complexity.

\paragraph{Definition of width:} We follow the description of ResNet width discussed in \cite{deep_double_descent}. The ResNet18 architecture we follow is that of \cite{he2016identity}.  The original ResNet18 architecture has four successive ResNet blocks each formed by two identical stacked sequences of a BatchNorm and a convolution layer. The number of filters for the convolution layers in each of the successive ResNet  block is $(k, 2k, 4k, 8k)$ with $k=64$ for the original RestNet18. Following \cite{deep_double_descent}, we take $k$ to be the width of the network, which we vary during our experiments.

\newpage

\section{Additional Experiments}\label{appendix:additional_experiments}

\subsection{Geometric complexity at initialization decreases with added layers on large domains}\label{appendix:initialization_on_large_domains}

We reproduce the experiments from Fig. 2 on a larger domain with the same conclusion: For both the standard and the Glorot initialization schemes ReLU networks initialize closer to the zero function and with lower geometric complexity as the number of layer increases.

\begin{figure}[h]
  \centering
  \includegraphics[width=0.4
  \linewidth]{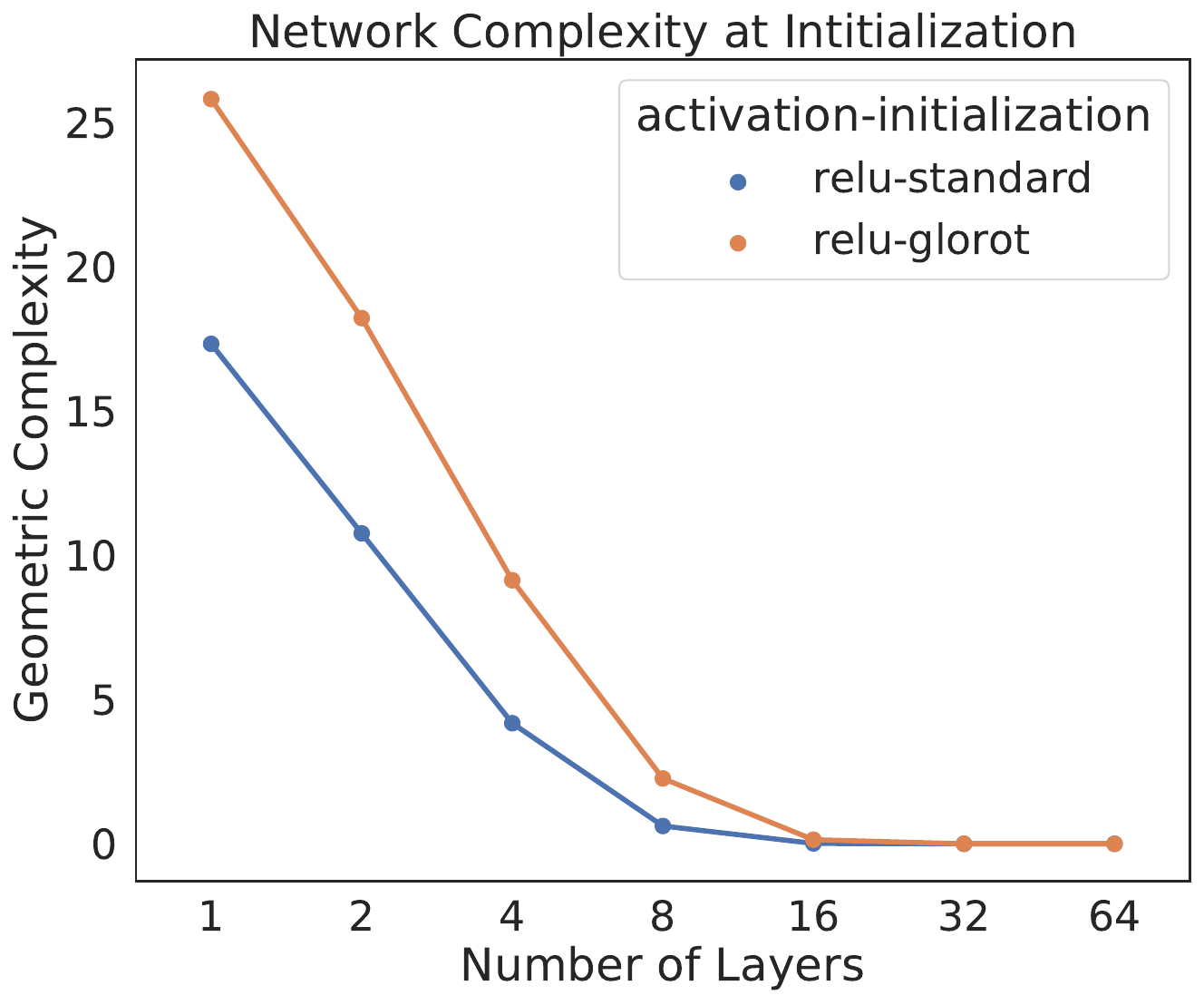}
  \caption{{\bf Geometric complexity at initialization decreases with the number of layers up to zero even on large domains:} We repeat the setup of Fig. 2 Right described in Section \ref{section:figure_2} but we sample the dataset from points in a large domain $[-1000, 1000]^d$ instead of the normalized hyper-cube $[-1, 1]^d$. We measure the geometric complexity for both the Glorot and the standard initialization schemes.}
  \label{fig:initialization_large}
\end{figure}

\begin{figure}[h]
  \centering
  \includegraphics[width=0.7
  \linewidth]{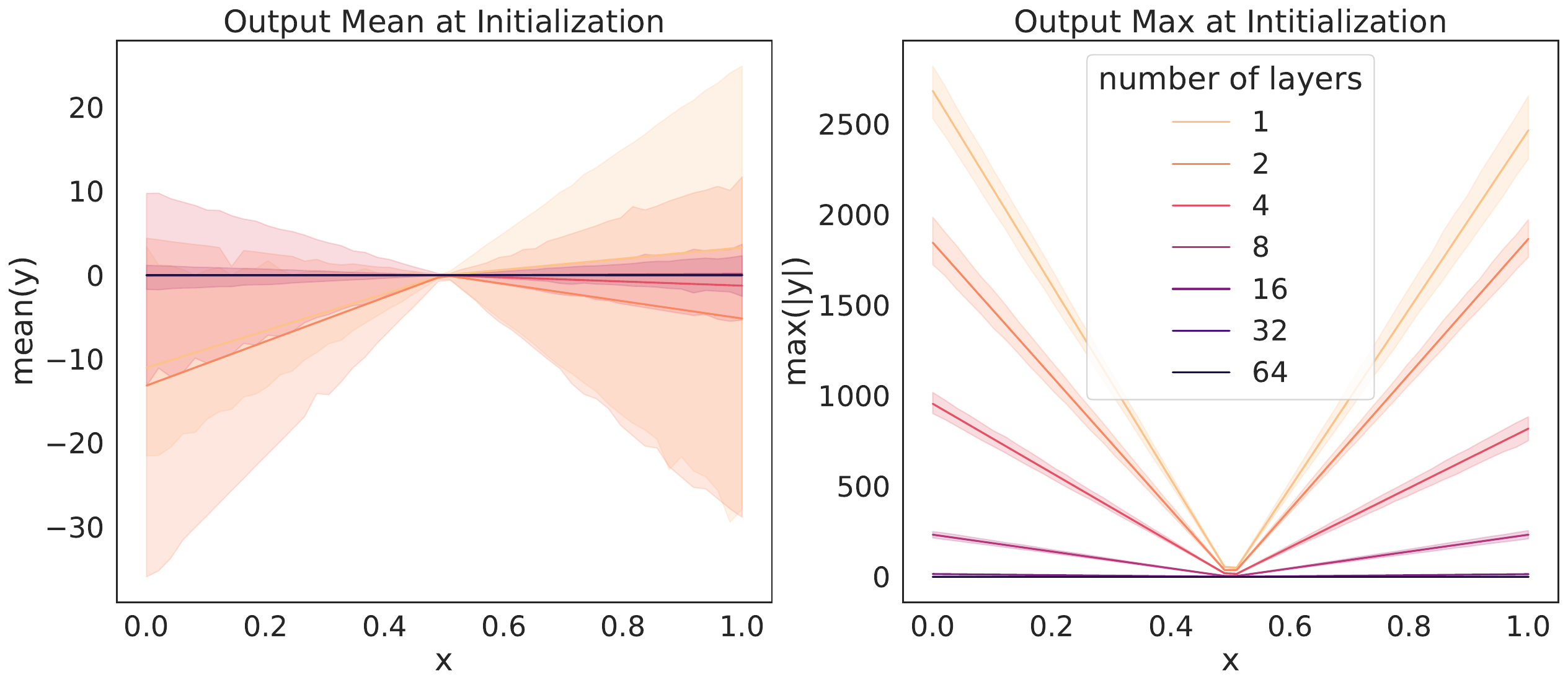}
  \caption{{\bf Deeper neural Relu networks initialize closer to the zero function with the Glorot scheme even on large domains:} We repeat the setup of Fig. 2 Left and Middle described in Section \ref{section:figure_2} but we evaluate the networks on a diagonal of the larger hyper-cube $[-1000, 1000]^d$ instead of the normalized hyper-cube $[-1, 1]^d$.}
  \label{fig:initialization_large_relu_glorot}
\end{figure}

\begin{figure}[h]
  \centering
  \includegraphics[width=0.7
  \linewidth]{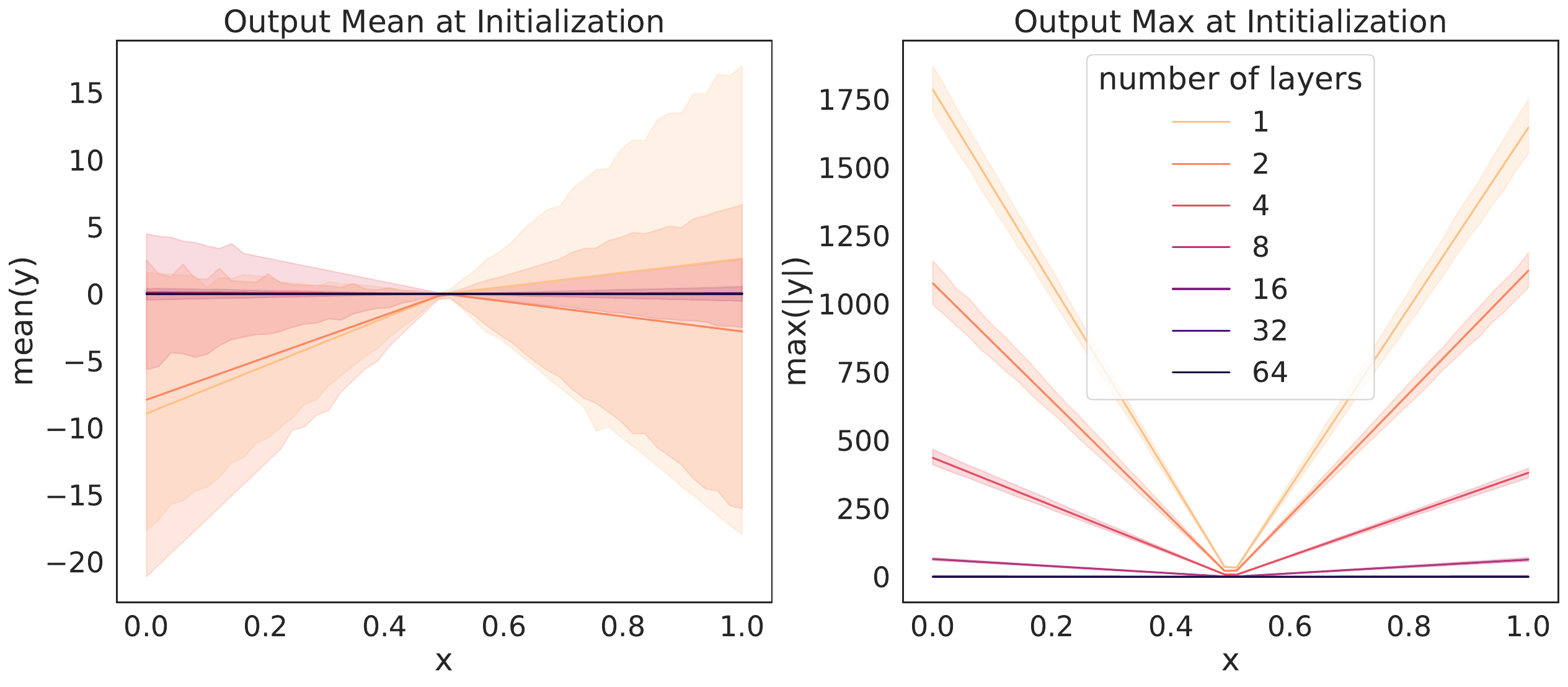}
  \caption{{\bf Deeper neural Relu networks initialize closer to the zero function with the standard scheme even on large domains:} We repeat the setup of Fig. 2 Left and Middle described in Section \ref{section:figure_2} but we evaluate the networks on a diagonal of the larger hyper-cube $[-1000, 1000]^d$ instead of the normalized hyper-cube $[-1, 1]^d$.}
  \label{fig:initialization_large_relu_glorot}
\end{figure}

\newpage 

\subsection{Geometric complexity decreases with implicit regularization for MNIST}

We replicate the implicit regularization experiments we conducted for CIFAR10 with ResNet18 in Fig. 4 for MLP's trained on MNIST. The conclusion is the same: higher learning rates and smaller batch sizes decrease geometric complexity through implicit gradient regularization, and are correlated with higher test accuracy.

\begin{figure}[h]
  \centering
  \includegraphics[width=1
  \linewidth]{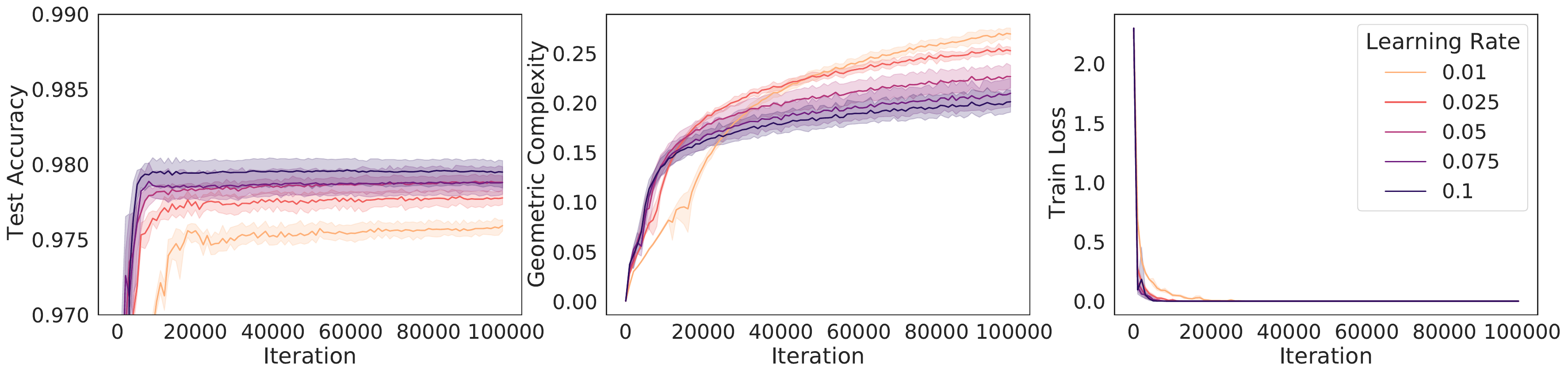}
  \caption{{\bf Geometric complexity decreases with higher learning rates on MNIST}: We trained  a selection of MLP's with 6 hidden layers with 500 neurons per layer on MNIST with batch size of 512, for 100000 steps and with varying batch sizes. For each learning rate in  $[0.01, 0.025, 0.05, 0.075, 0.1]$, we trained 5 different times with a different random seed. The MLP were initialized using the standard initialization scheme.}
  \label{fig:learning_rates_mnist}
\end{figure}

\begin{figure}[h]
  \centering
  \includegraphics[width=1
  \linewidth]{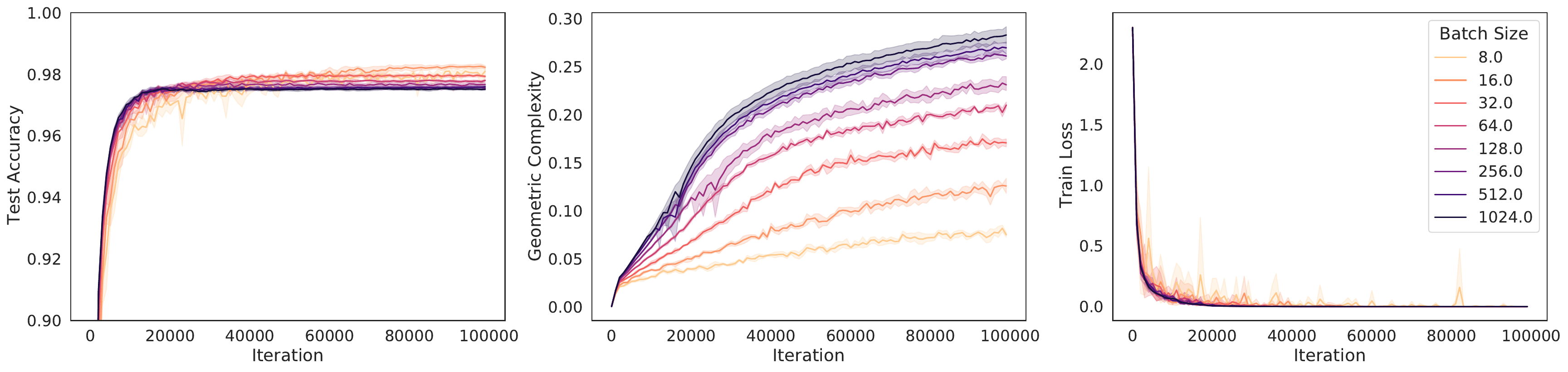}
  \caption{{\bf Geometric complexity decreases with smaller batch sizes on MNIST}: We trained  a selection of MLP's with 6 hidden layers with 500 neurons per layer on MNIST with learning rate of 0.02, for 100000 steps and with varying batch sizes. For each batch size in $[8, 16, 32, 64, 128, 256, 512, 1024]$, we trained 5 different times with a different random seed. The MLP were initialized using the standard initialization scheme.}
  \label{fig:batch_sizes_mnist}
\end{figure}

\newpage

\subsection{Geometric complexity decreases with explicit regularization for MNIST}

We replicate the explicit regularization experiments we conducted for CIFAR10 with ResNet18 in Fig. 3 for MLP's trained on MNIST with similar conclusions: higher regularization rates for L2, spectral, and flatness regularization decrease geometric complexity and are correlated with higher test accuracy.

\begin{figure}[h]
  \centering
  \includegraphics[width=1
  \linewidth]{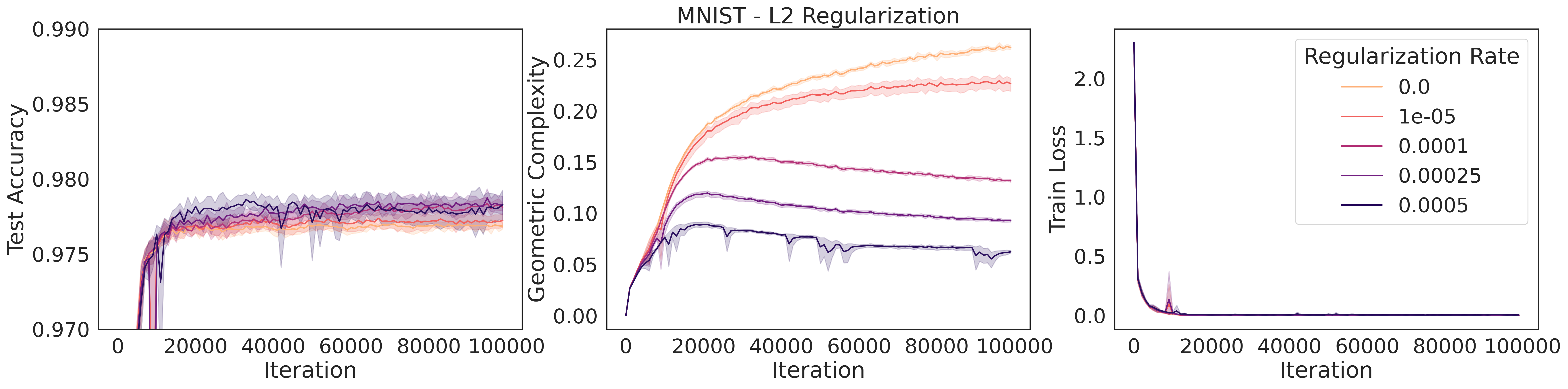}
  \caption{{\bf Geometric complexity decreases with L2 regularization on MNIST}: We trained  a selection of MLP's with 6 hidden layers with 500 neurons per layer on MNIST with learning rate of 0.02, batch size of 512, for 100000 steps. We regularized the loss by adding to it the L2 norm penalty $\alpha\sum_i \|W_i\|_F^2$ where $W_i$ are the layer weight matrices. For each regularization rate  $\alpha \in [0, 0.00001, 0.0001, 0.00025, 0.0005]$, we trained 5 different times with a different random seed. The MLP were initialized using the standard initialization scheme.}
  \label{fig:l2_regularization_mnist}
\end{figure}

\begin{figure}[h]
  \centering
  \includegraphics[width=1
  \linewidth]{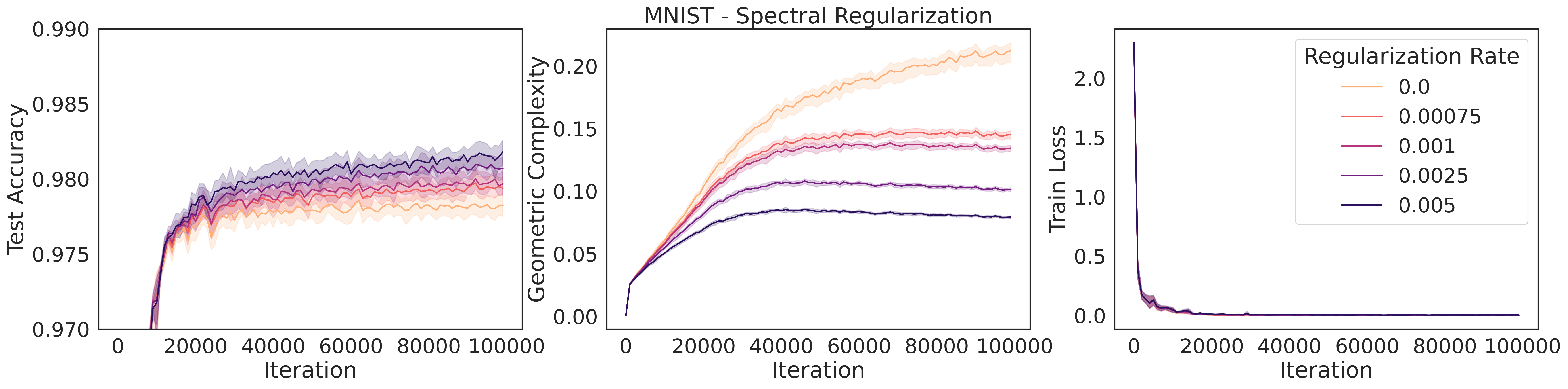}
  \caption{{\bf Geometric complexity decreases with spectral regularization on MNIST}: We trained  a selection of MLP's with 4 hidden layers with 200 neurons per layer on MNIST with learning rate of 0.02, batch size of 128, for 100000 steps. We regularized the loss by adding to it the spectral norm penalty $\alpha/2\sum_i \sigma_{\max}(W_i)^2$ where $W_i$ are the layer weight matrices as described in \cite{miyato2018spectral}. For each regularization rate  $\alpha \in [0, 0.00075, 0.001, 0.0025, 0.005]$, we trained 5 different times with a different random seed. The MLP were initialized using the standard initialization scheme.}
  \label{fig:spectral_regularization_mnist}
\end{figure}

\begin{figure}[h]
  \centering
  \includegraphics[width=1
  \linewidth]{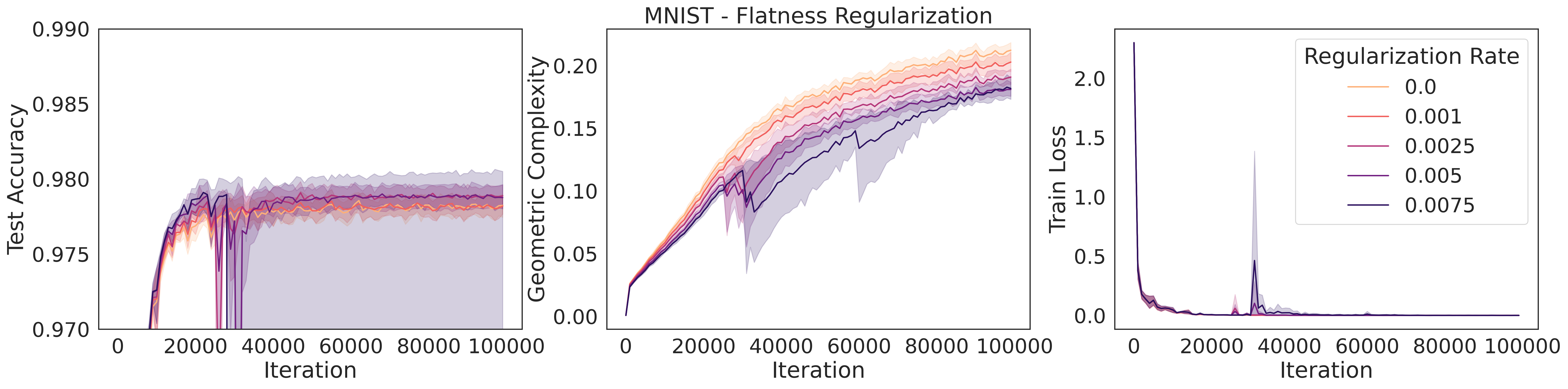}
  \caption{{\bf Geometric complexity decreases with flatness regularization on MNIST}: We trained  a selection of MLP's with 4 hidden layers with 200 neurons per layer on MNIST with learning rate of 0.02, batch size of 128, for 100000 steps. We regularized the loss by adding to it the gradient penalty $\alpha \|\nabla_\theta L_B(\theta)\|^2$ where $L_B$ is the batch loss. For each regularization rate  $\alpha \in [0, 0.001, 0.0025, 0.005, 0.0075]$, we trained 5 different times with a different random seed. The MLP were initialized using the standard initialization scheme.}
  \label{fig:flatness_regularization_mnist}
\end{figure}

\clearpage

\subsection{Explicit geometric complexity regularization for MNIST and CIFAR10}\label{appendix:gc_regularization_additional_experiments}

In this section, we explicitly regularize for the geometric complexity. This is a known form of regularization also called Jacobian regularization \cite{hoffman2020robust, Sokolic2017RobustLM, varga2018gradient, Yoshida2017SpectralNR}.
We first perform this regularization for a MLP trained on MNIST (Fig. \ref{fig:explicit_gc_regularization_mnist}) and then for a ResNet18 trained on CIFAR10 (Fig. \ref{fig:explicit_gc_regularization_cifar}) with the following conclusion: test accuracy increases with higher regularization strength while the geometric complexity decreases. 

\begin{figure}[h]

  \centering
  \includegraphics[width=1
  \linewidth]{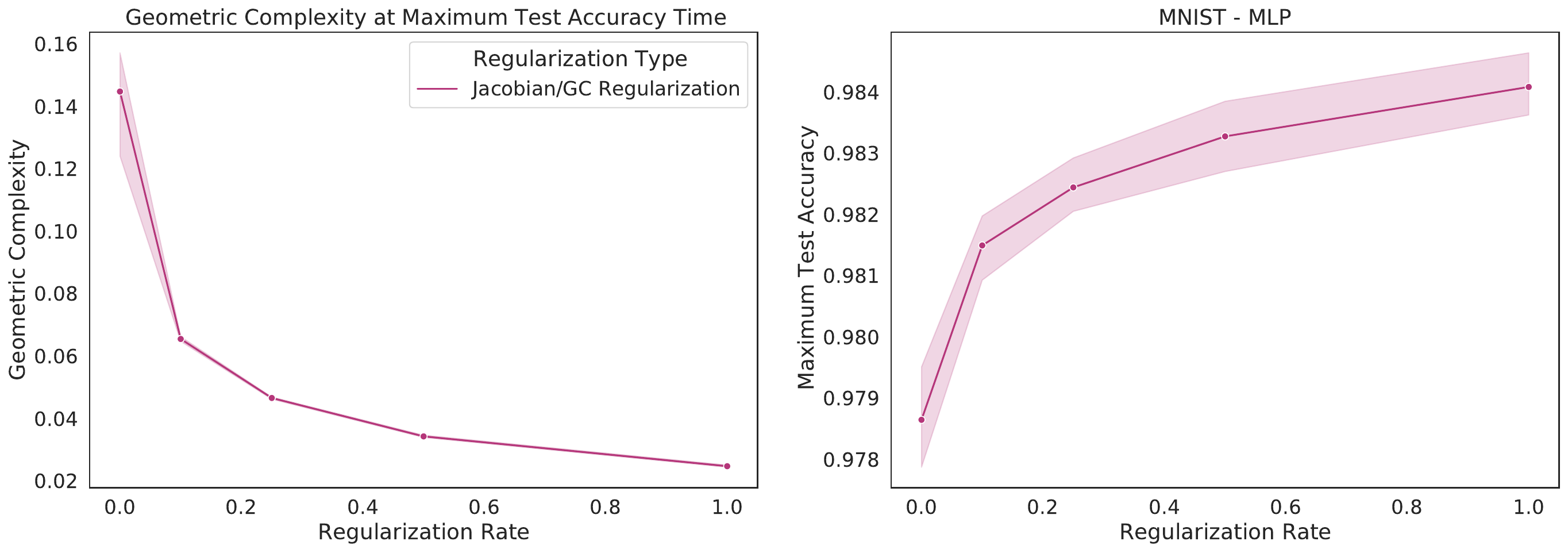}
  \caption{{\bf Test accuracy increases with explicit GC regularization on MNIST}: We trained  a selection of MLP's with 4 hidden layers with 200 neurons per layer on MNIST with learning rate of 0.02, batch size of 128, for 100000 steps. We regularized the loss by adding to it the gradient penalty $\alpha /B \sum_{x \in B}\|\nabla_x f_\theta(x)\|^2_F$ where $f_\theta(x)$ is the logit network. For each regularization rate  $\alpha \in [0, 0.1, 0.25, 0.5, 1]$, we trained 5 different times with a different random seed. The MLP were initialized using the standard initialization scheme.}
  \label{fig:explicit_gc_regularization_mnist}
\end{figure}

\begin{figure}[h]

  \centering
  \includegraphics[width=1
  \linewidth]{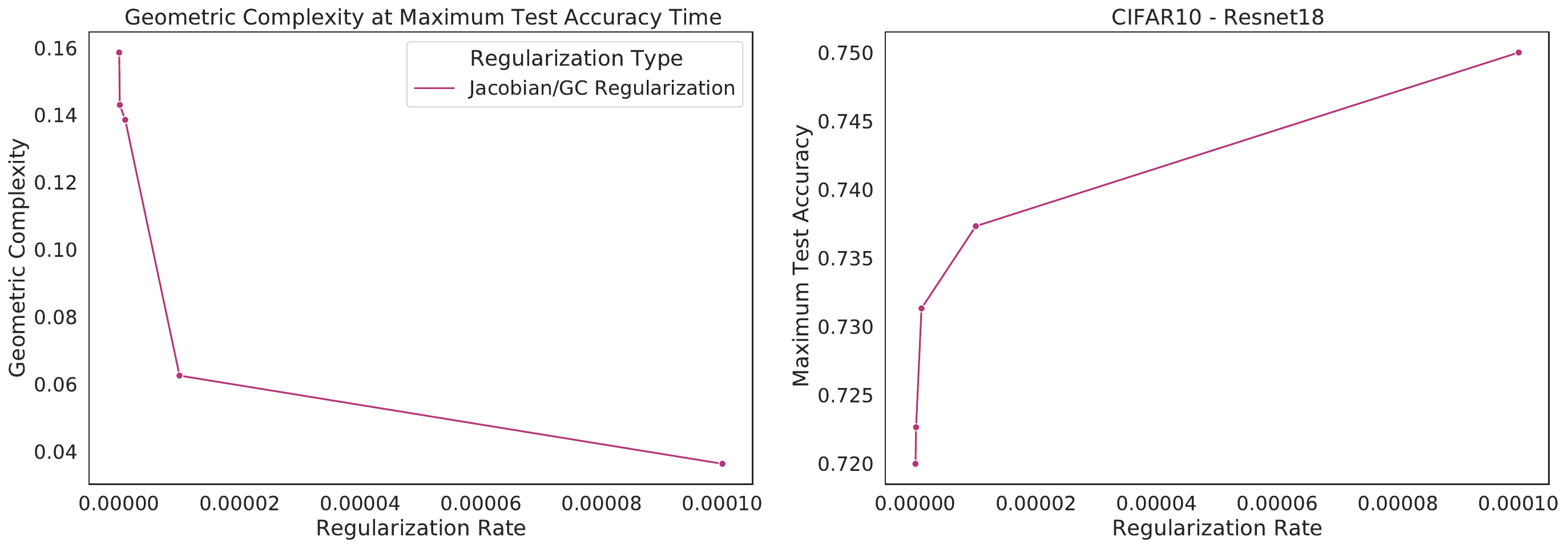}
  \caption{{\bf Test accuracy increases with explicit GC regularization on CIFAR10}: We trained  a selection of ResNet18 with learning rate of 0.02, batch size of 128, for 10000 steps. We regularized the loss by adding to it the gradient penalty $\alpha /B \sum_{x \in B}\|\nabla_x f_\theta(x)\|^2_F$ where $f_\theta(x)$ is the logit network. For each regularization rate  $\alpha \in [0, 0.0000001, 0.000001, 0.00001, 0.0001]$,  we trained only one time with a single random seed, and the training had to be stopped before reaching peak test accuracy because of the heavy computational time due to this regularization.}
  \label{fig:explicit_gc_regularization_cifar}
\end{figure}

\clearpage

\subsection{Separate L2, flatness, and spectral regularization experiments for CIFAR10}
\label{appendix:additional_experiments_for_explicit_regularization}

For the sake of space in Fig. 3 (right) in the main paper, we used the same regularization rate range for all types of explicit regularization we tried. 
In this section, we perform the experiments on a targeted range for each regularization type, leading to clearer plots (Fig. \ref{fig:explicit_l2_regularization_cifar}, Fig. \ref{fig:flatness_regularization_cifar}, and Fig. \ref{fig:spectral_regularization_cifar}). 

\begin{figure}[h]

  \centering
  \includegraphics[width=0.8
  \linewidth]{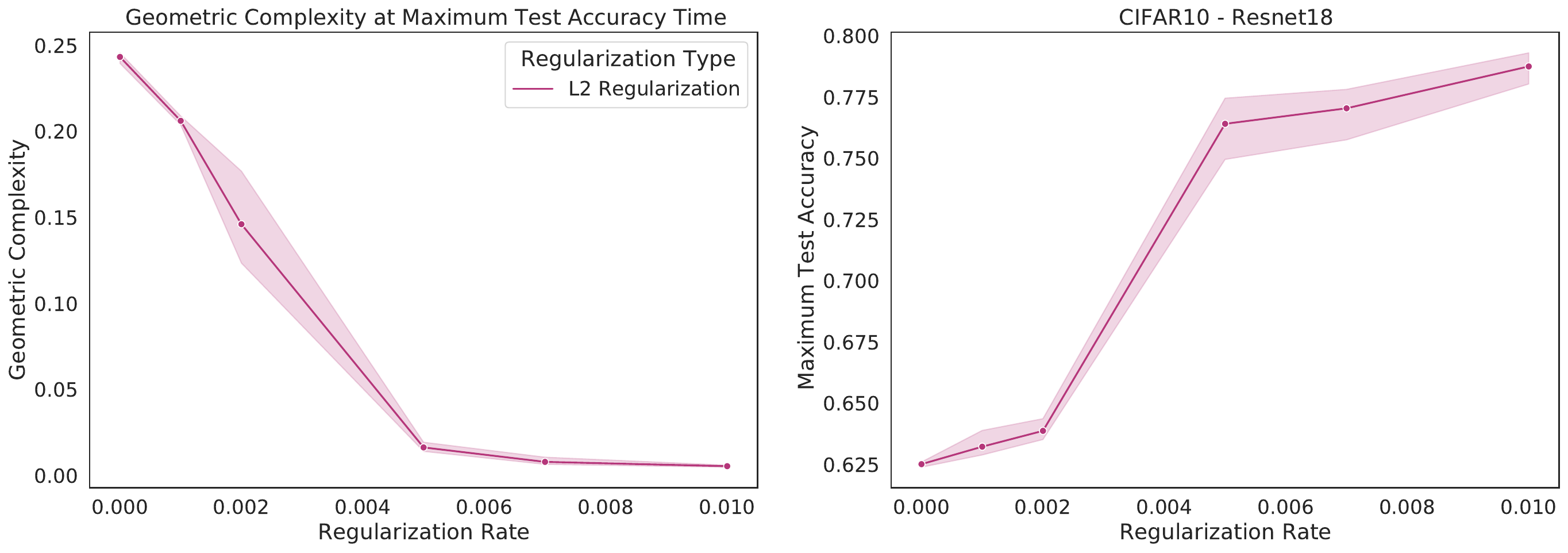}
  \caption{{\bf GC decreases with explicit L2 regularization on CIFAR10}: We trained  a selection of ResNet18 with learning rate of 0.02, batch size of 512, for 10000 steps. We regularized the loss by adding to it the standard L2 loss penalty. For each regularization rate  $\alpha \in [0, 0.001, 0.002, 0.005, 0.007, 0.01]$, we trained 3 different times with a different random seed. }
  \label{fig:explicit_l2_regularization_cifar}
\end{figure}

\begin{figure}[h]

  \centering
  \includegraphics[width=0.8
  \linewidth]{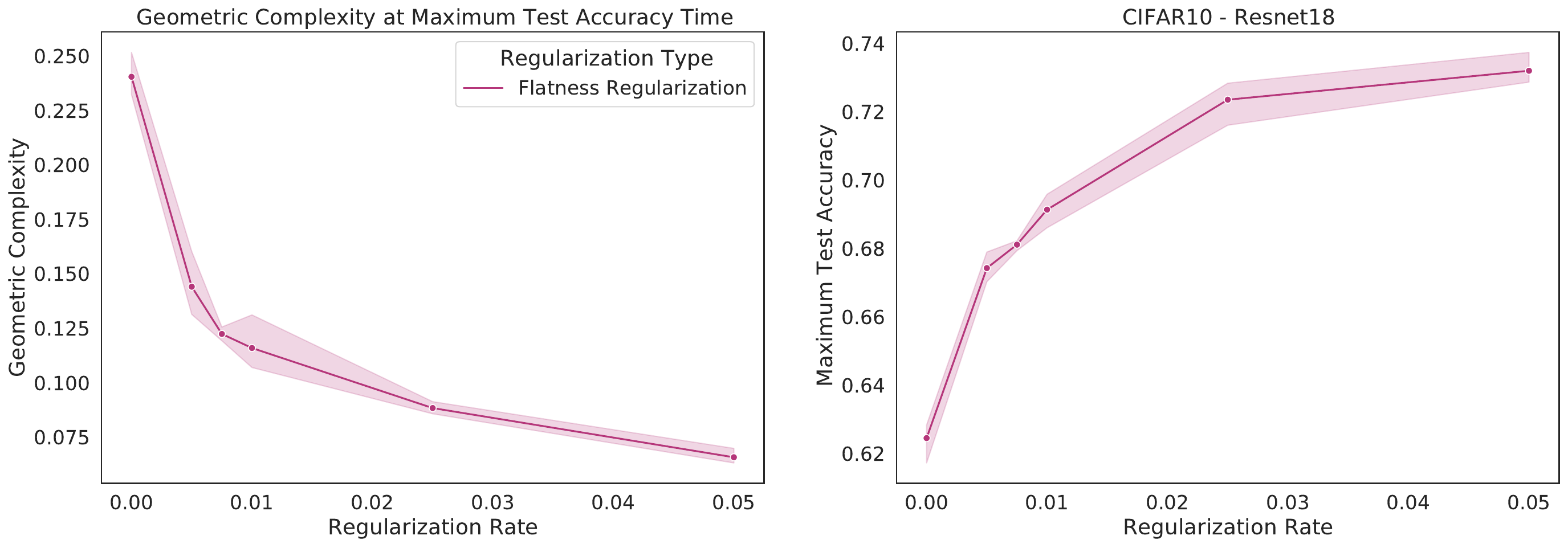}
  \caption{{\bf GC decreases with explicit flatness regularization on CIFAR10}: We trained  a selection of ResNet18 with learning rate of 0.02, batch size of 512, for 10000 steps. We regularized the loss by adding to it the gradient penalty $\alpha \|\nabla_\theta L_B(\theta)\|^2$ where $L_B$ is the batch loss. For each regularization rate  $\alpha \in [0, 0.005, 0.0075, 0.01, 0.025, 0.05]$, we trained 3 different times with a different random seed. }
  \label{fig:flatness_regularization_cifar}
\end{figure}

\begin{figure}[h]

  \centering
  \includegraphics[width=0.8
  \linewidth]{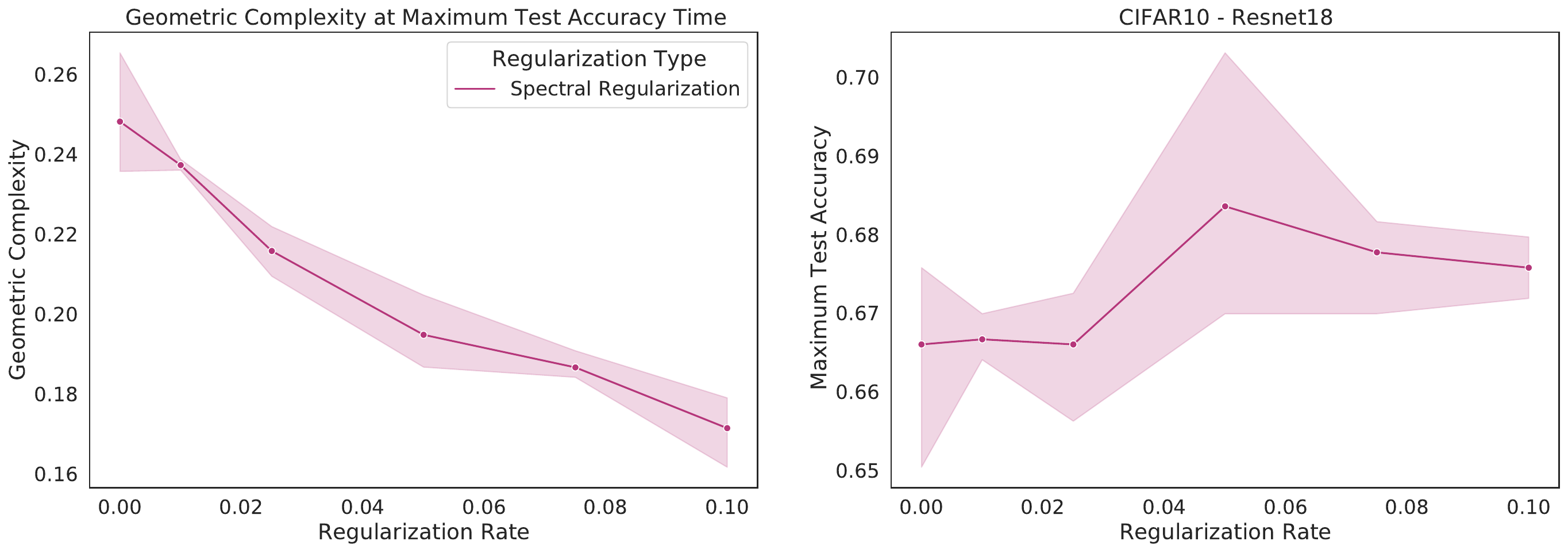}
  \caption{{\bf Geometric complexity decreases with spectral regularization on CIFAR10:} We trained  a selection of ResNet18 on CIFAR10 with learning rate of 0.02, batch size of 512, for 100000 steps. For each regularization rate  $\alpha \in [0, 0.01, 0.025, 0.05, 0.075, 0.1]$, we trained 3 different times with a different random seed.}
  \label{fig:spectral_regularization_cifar}
\end{figure}

\clearpage

\subsection{Geometric complexity in the presence of multiple tuning mechanisms}\label{appendix:real_life_gc_regularization}

For most of this paper, we studied the impact of tuning strategies, like the choice of initialization, hyper-parameters, or explicit regularization in isolation from other very common heuristics like learning rate schedules and data-augmentation. In this section, we reproduce the implicit regularization effect of the batch size and the learning rate on GC (c.f. Fig.~4 in the main paper) while using these standard tricks to achieve better performance. The resulting models achieve performance closer to SOTA for the ResNet18 architecture. 

Although the learning curves are messier and harder to interpret (Fig. \ref{figure:tuned_cifar_learning_curves}) because of the multiple mechanisms interacting in complex ways, we still observe that the general effect of the learning rate (Fig. \ref{fig:lr_gc_accuracy_cifar10}) and the batch size (Fig. \ref{fig:batch_size_gc_accuracy_cifar10}) on geometric complexity is preserved in this context. More importantly, we also note that the sweeps with higher test accuracy solutions tend also to come with a lower geometric complexity, even in this more complex setting (Fig. \ref{fig:lr_gc_accuracy_cifar10} right and Fig. \ref{fig:batch_size_gc_accuracy_cifar10} right). Namely, models with higher test accuracy have correspondingly lower GC. Specifically in terms of implicit regularization, as the learning rate increases, the geometric complexity decreases and the maximum test accuracy increases. Similarly, smaller batch size leads to lower geometric complexity as well as higher test accuracy. 

\begin{figure}[h]

  \centering
    \includegraphics[width=1\linewidth]{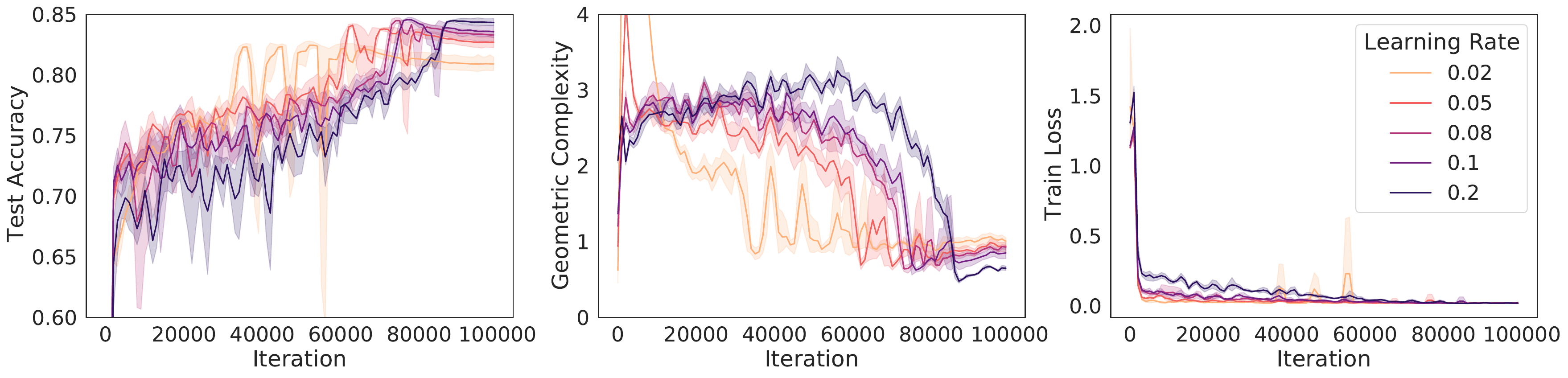}
    \includegraphics[width=1\linewidth]{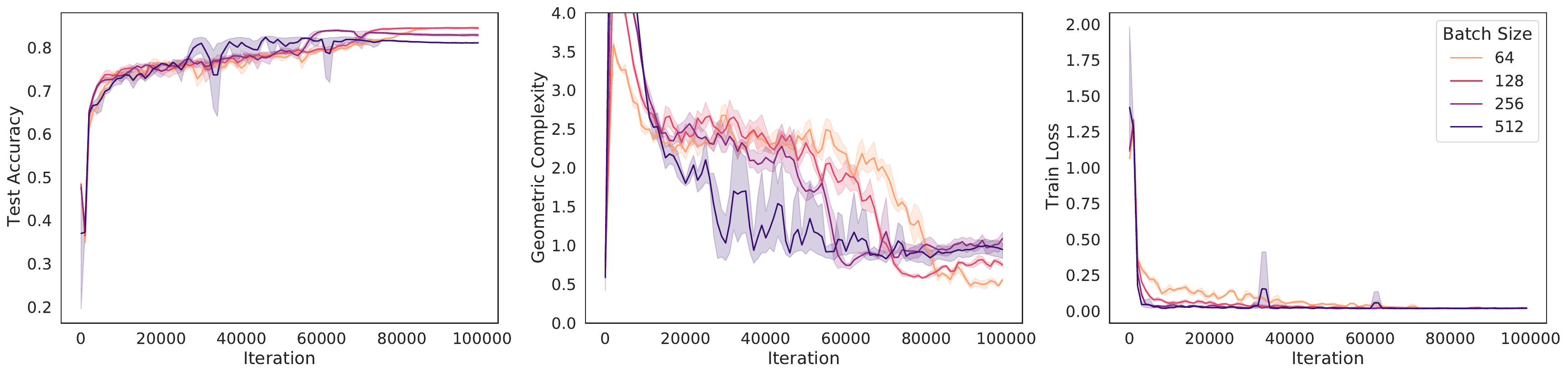}
  \caption{Impact of IGR when training ResNet18 on CIFAR10 with multiple tuning mechanisms including cosine learning rate scheduler, data augmentation and L2 regularization. Note that the GC is computed on batches of size 128 which leads to a lot of variance in the estimate. \textbf{Top row:} As IGR increases through higher learning rates, GC decreases. \textbf{Bottom row:} Similarly, lower batch size leads to decreased GC.}
  \label{figure:tuned_cifar_learning_curves}
\end{figure}

\begin{figure}[h]

  \centering
  \includegraphics[width=1
  \linewidth]{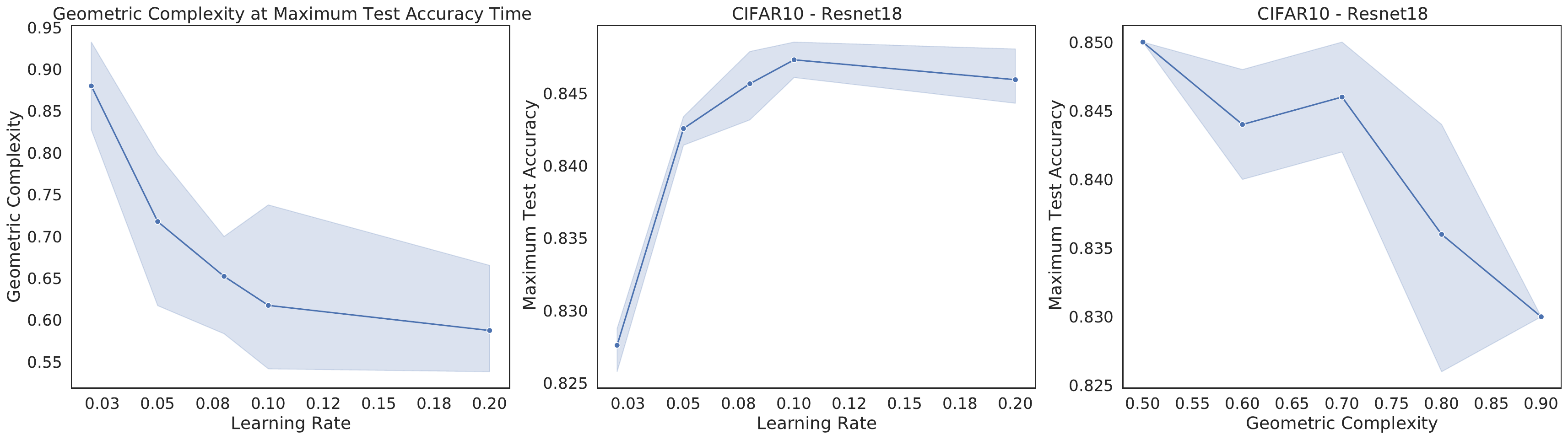}
  \caption{{\bf GC decreases as learning rate and model test accuracy increases on CIFAR10}: We trained a collection of ResNet18 models on CIFAR10 with varying initial learning rates $h \in [0.02, 0.05, 0.08, 0.1, 0.2]$ and cosine learning rate schedule. Each job was trained with SGD without momentum for 100000 steps, with batch size 128 and L2 regularized loss with regularization rate 0.005. We also included data augmentation in the form of random flip. Test accuracy is reported as best test accuracy during training. GC is computed during training on the training batches, which produces a large variance in the estimate when the batch size is small.}
  \label{fig:lr_gc_accuracy_cifar10}
\end{figure}

\begin{figure}[h]

  \centering
  \includegraphics[width=1
  \linewidth]{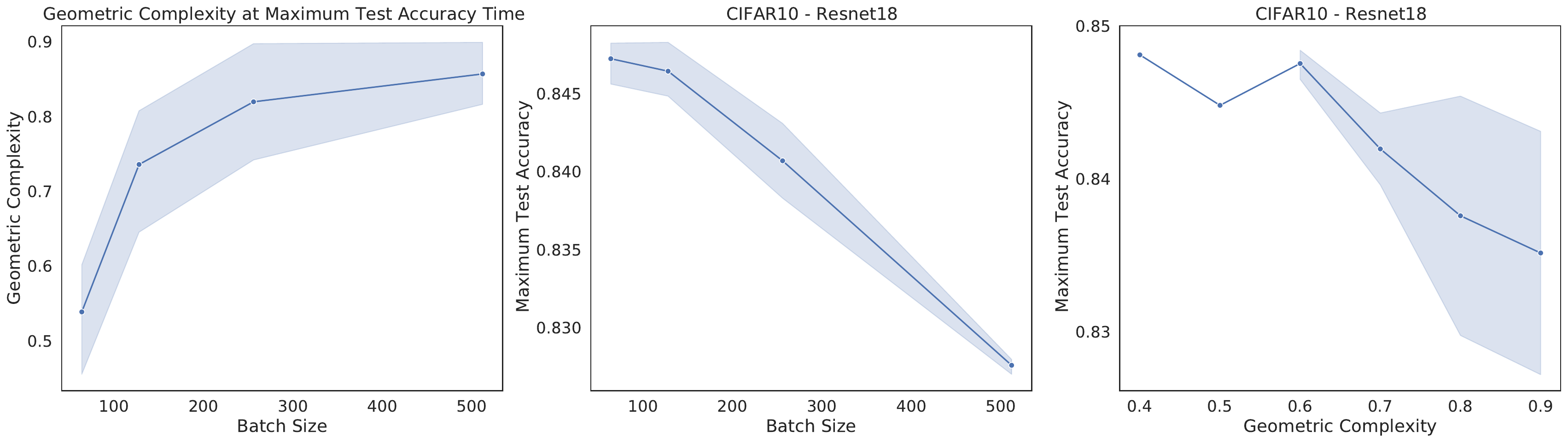}
  \caption{{\bf GC increases as batch size increases on CIFAR10}: We trained a collection of ResNet18 models on CIFAR10 with varying batch sizes of 64, 128, 256, and 512. Each job was trained with SGD without momentum for 100000 steps, with cosine learning rate scheduler initialized at 0.02 and L2 regularized loss with regularization rate 0.005. We also included data augmentation in the form of random flip. Test accuracy is reported as best test accuracy during training. GC is computed during training on the training batches, which produces a large variance in the estimate when the batch size is small.}
  \label{fig:batch_size_gc_accuracy_cifar10}
\end{figure}

\clearpage

\subsection{Geometric complexity in the presence of momentum}\label{appendix:additional_experiments_momentum}

We replicate the implicit and explicit regularization experiments using SGD \emph{with momentum}, which is widely used in practise. The conclusion remains the same as for vanilla SGD: More regularization (implicit through batch size or learning rate or explicit through flatness, spectral, and L2 penalties) produces solutions with higher test accuracy and lower geometric complexity.

\begin{figure}[h]

  \centering
  \includegraphics[width=1
  \linewidth]{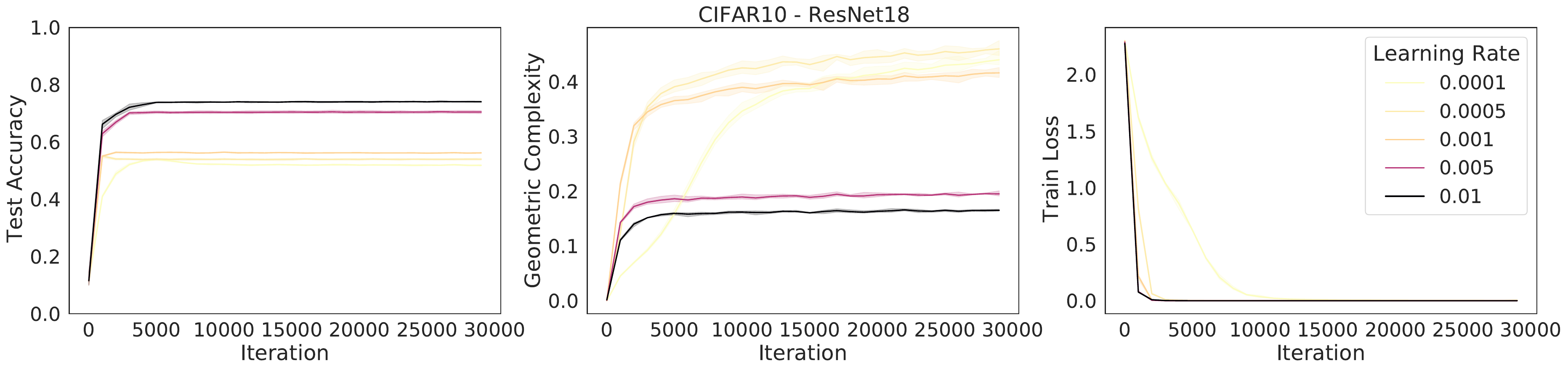}
  \caption{{\bf Geometric complexity decreases with higher learning rates on CIFAR10 trained with momentum}: We trained  a selection of ResNet18 with  batch size of 512 for 30000 steps using SGD with a momentum of 0.9. For each learning rate $\alpha \in [0.0001, 0.0005, 0.001, 0.005, 0.01]$, we trained 3 different times with a different random seed.}
  \label{fig:momentum_flatness_regularization_mnist}
\end{figure}

\begin{figure}[h]
  \centering
  \includegraphics[width=1
  \linewidth]{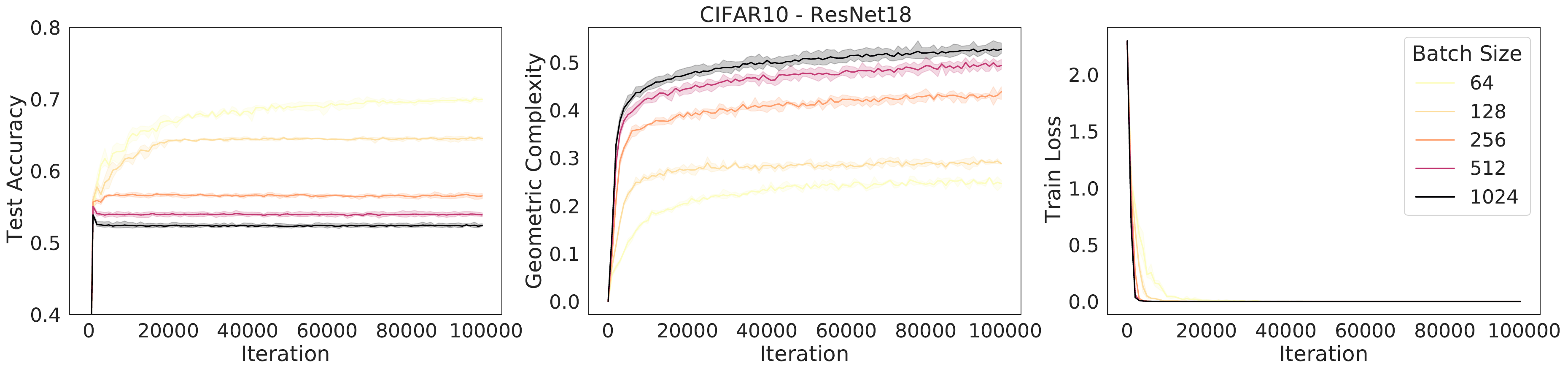}
  \caption{{\bf Geometric complexity decreases with lower batch sizes on on CIFAR10 trained with momentum}: We trained  a selection of ResNet18 with  learning rate of 0.0005 for 100000 steps using SGD with a momentum of 0.9. For each batch size in  $\alpha \in [64, 128, 256, 512, 1024]$, we trained 3 different times with a different random seed.}
  \label{fig:momentum_flatness_regularization_mnist}
\end{figure}

\begin{figure}[h]
  \centering
  \includegraphics[width=1
  \linewidth]{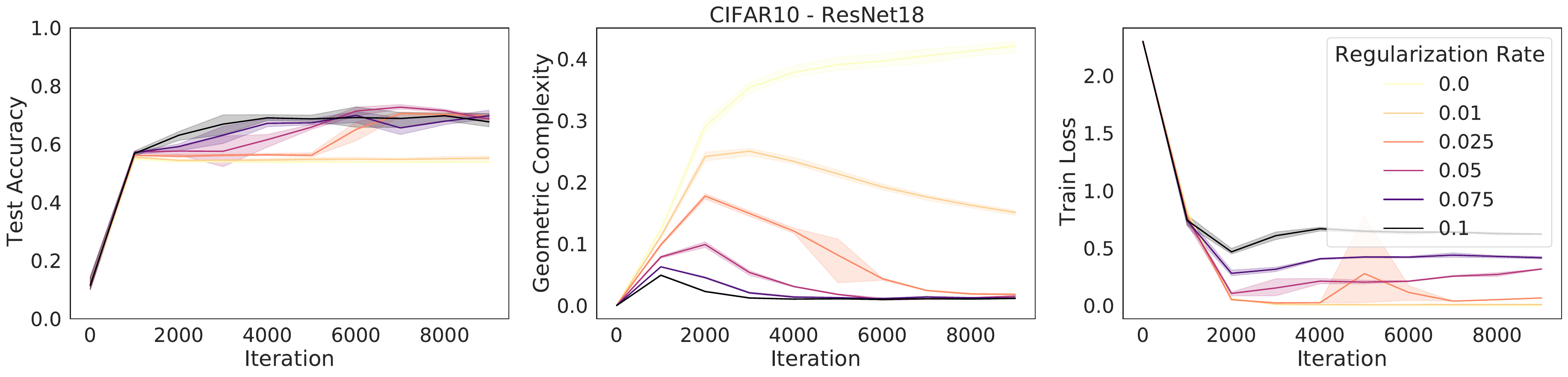}
  \caption{{\bf Geometric complexity decreases with increased L2 regularization on CIFAR10 trained with momentum}: We trained  a selection of ResNet18 with  learning rate of 0.0005 with batch size of 512 for 10000 steps using SGD with a momentum of 0.9. For each regularization rate in  $\alpha \in [0, 0.01, 0.025, 0.05, 0.075, 0.1]$, we trained 3 different times with a different random seed.}
  \label{fig:momentum_flatness_regularization_mnist}
\end{figure}

\begin{figure}[h]
  \centering
  \includegraphics[width=1
  \linewidth]{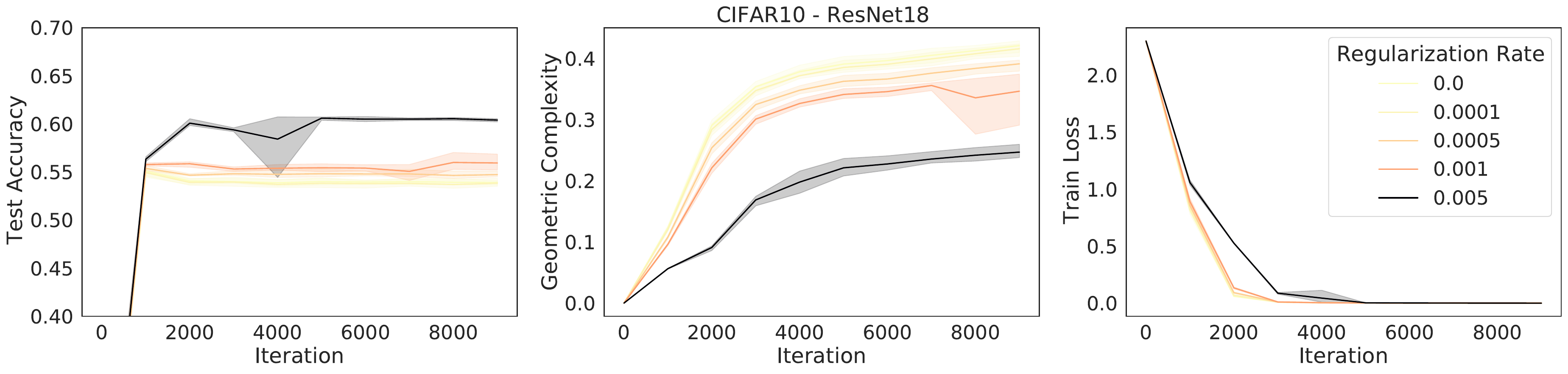}
  \caption{{\bf Geometric complexity decreases with increased flatness regularization on CIFAR10 trained with momentum}: We trained  a selection of ResNet18 with  learning rate of 0.0005 with batch size of 512 for 10000 steps using SGD with a momentum of 0.9.  We regularized the loss by adding to it the gradient penalty $\alpha \|\nabla_\theta L_B(\theta)\|^2$ where $L_B$ is the batch loss. For each regularization rate in  $\alpha \in [0, 0.0001, 0.0005, 0.001, 0.005]$, we trained 3 different times with a different random seed.}
  \label{fig:momentum_flatness_regularization_mnist}
\end{figure}

\begin{figure}[h]
  \centering
  \includegraphics[width=1
  \linewidth]{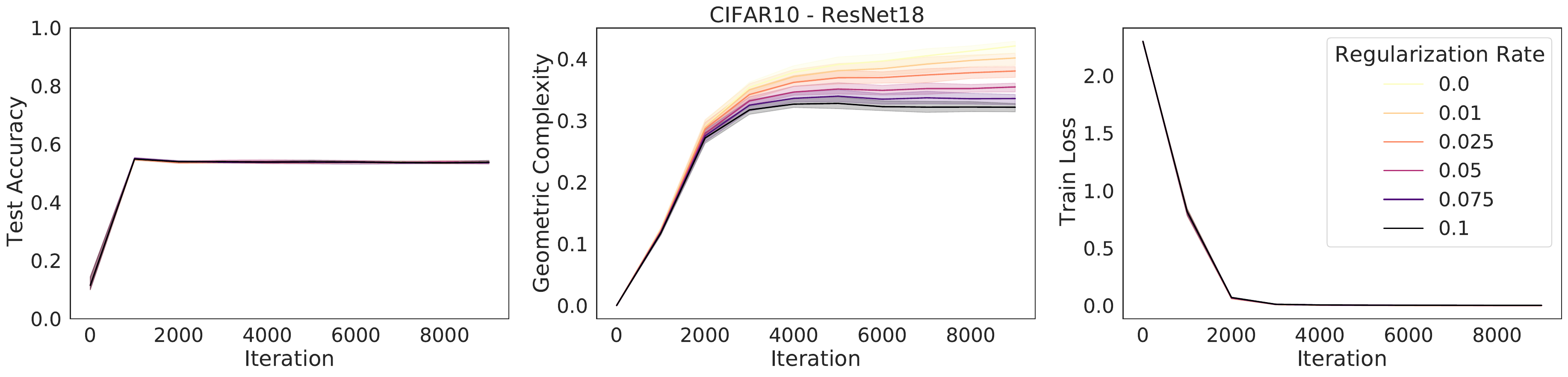}
  \caption{{\bf Geometric complexity decreases with increased spectral regularization on CIFAR10 trained with momentum}: We trained  a selection of ResNet18 with  learning rate of 0.0005 with batch size of 512 for 10000 steps using SGD with a momentum of 0.9.  We regularized the loss by adding to it the spectral norm penalty $\alpha/2\sum_i \sigma_{\max}(W_i)^2$ where $W_i$ are the layer weight matrices as described in \cite{miyato2018spectral}. For each regularization rate in  $\alpha \in [0, 0.01, 0.025, 0.05, 0.075, 0.1]$, we trained 3 different times with a different random seed.}
  \label{fig:momentum_flatness_regularization_mnist}
\end{figure}

\clearpage

\subsection{Geometric complexity in the presence of Adam}
\label{appendix:additional_experiments_adam}

We replicate the implicit and explicit regularization experiment using Adam, which is widely used in practice. In this case the conclusions are less clear than with vanilla SGD or SGD with momentum. While higher learning rates, and higher explicit flatness, L2, and spectral regularization still put a regularizing pressure on the geometric complexity for most of the training, the effect of batch size on geometric complexity is not clear. This may be that the local built-in re-scaling of the gradient sizes in Adam affects the pressure on the geometric complexity in complex ways when the batch size changes. 

\begin{figure}[h]
  \centering
  \includegraphics[width=1
  \linewidth]{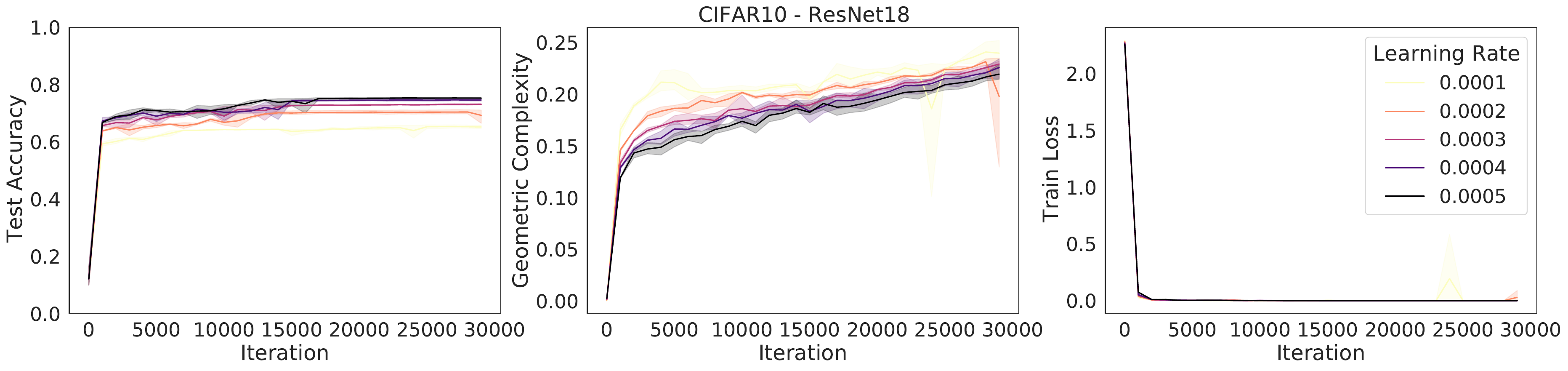}
  \caption{{\bf Geometric complexity decreases with higher learning rates on CIFAR10 trained using Adam with $b1=0.9$, $b2=0.999$}: We trained  a selection of ResNet18 with  batch size of 512 for 30000 steps using Adam with a momentum of 0.9. For each learning rate $\alpha \in [0.0001, 0.0002, 0.0003, 0.0004, 0.0005]$, we trained 3 different times with a different random seed.}
  \label{fig:adam_flatness_regularization_mnist}
\end{figure}

\begin{figure}[h]
  \centering
  \includegraphics[width=1
  \linewidth]{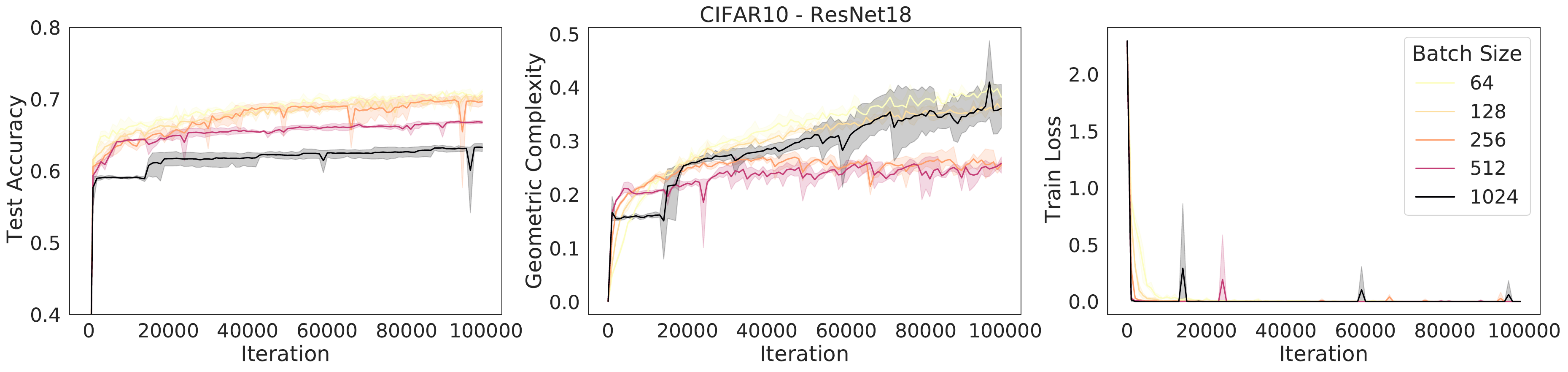}
  \caption{{\bf The relation between geometric complexity and batch size is ambiguous on CIFAR10 trained with Adam}: We trained  a selection of ResNet18 with  learning rate of 0.0001 for 100000 steps using using Adam with $b1=0.9$, $b2=0.999$. For each batch size in  $\alpha \in [64, 128, 256, 512, 1024]$, we trained 3 different times with a different random seed.}
  \label{fig:adam_flatness_regularization_mnist}
\end{figure}

\begin{figure}[h]
  \centering
  \includegraphics[width=1
  \linewidth]{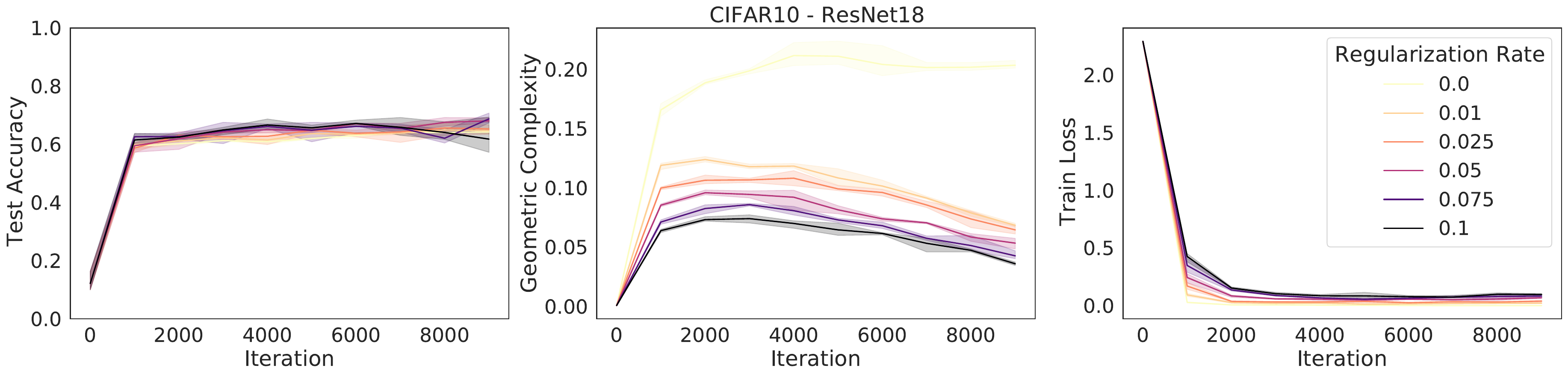}
  \caption{{\bf Geometric complexity decreases with increased L2 regularization on CIFAR10 trained with Adam}: We trained  a selection of ResNet18 with  learning rate of 0.0001 with batch size of 512 for 10000 steps using using Adam with $b1=0.9$, $b2=0.999$. For each regularization rate in  $\alpha \in [0, 0.01, 0.025, 0.05, 0.075, 0.1]$, we trained 3 different times with a different random seed.}
  \label{fig:adam_flatness_regularization_mnist}
\end{figure}

\begin{figure}[h]
  \centering
  \includegraphics[width=1
  \linewidth]{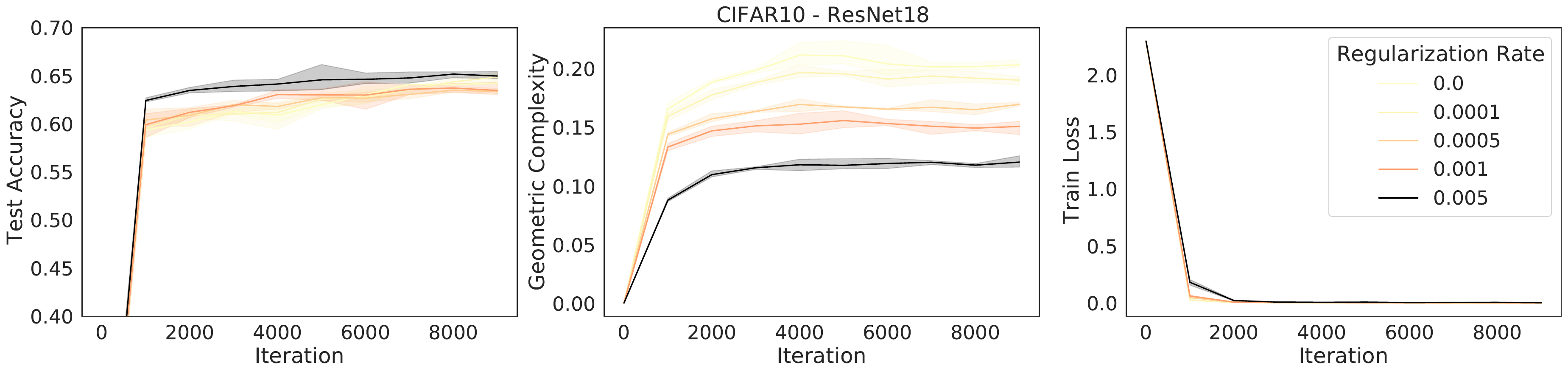}
  \caption{{\bf Geometric complexity decreases with increased flatness regularization on CIFAR10 trained with Adam}: We trained  a selection of ResNet18 with  learning rate of 0.0001 with batch size of 512 for 10000 steps using using Adam with $b1=0.9$, $b2=0.999$.  We regularized the loss by adding to it the gradient penalty $\alpha \|\nabla_\theta L_B(\theta)\|^2$ where $L_B$ is the batch loss. For each regularization rate in  $\alpha \in [0, 0.0001, 0.0005, 0.001, 0.005]$, we trained 3 different times with a different random seed.}
  \label{fig:adam_flatness_regularization_mnist}
\end{figure}

\begin{figure}[h]
  \centering
  \includegraphics[width=1
  \linewidth]{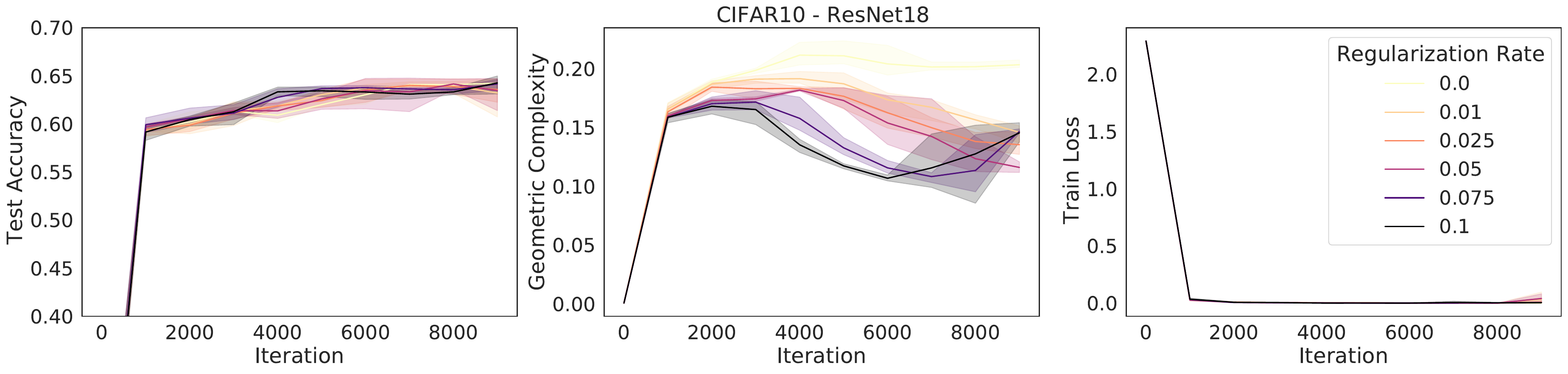}
  \caption{{\bf Geometric complexity decreases with increased spectral regularization on CIFAR10 trained with Adam}: We trained a selection of ResNet18 with learning rate of 0.0001 with batch size of 512 for 10000 steps using Adam with $b1=0.9$, $b2=0.999$. We regularized the loss by adding to it the spectral norm penalty $\alpha/2\sum_i \sigma_{\max}(W_i)^2$ where $W_i$ are the layer weight matrices as described in \cite{miyato2018spectral}. For each regularization rate in  $\alpha \in [0, 0.01, 0.025, 0.05, 0.075, 0.1]$, we trained 3 different times with a different random seed.}
  \label{fig:momentum_flatness_regularization_mnist}
\end{figure}

\newpage

\section{Comparison of the Geometric Complexity to other complexity measures}\label{appendix:comparison_survey}

One of the primary challenges in deep learning is to better understand mechanisms or techniques that correlate well with (or can imply a bound on) the generalization error for large classes of models. The standard approach of splitting the data into a train, validation and test set has become the de facto way to achieve such a bound. With this goal in mind, a number of complexity measures have been proposed in the literature with varying degrees of theoretical justification and/or empirical success. In this section we compare our geometric complexity measure with other, more familiar complexity measures such as the Rademacher complexity, VC dimension and sharpness-based measures.

\subsection{Rademacher Complexity}\label{appendix:rademacher_comparison}

Perhaps the most historically popular and widely known complexity measure is the Rademacher complexity \cite{bartlett2002model, bartlett2002rademacher, koltchinskii2001rademacher, koltchinskii2000rademacher}. Loosely speaking the Rademacher complexity measures the degree to which a class of functions $\mathcal{H}$ can fit random noise. The idea behind this complexity measure is that a more complex function space is able to generate more complex representation vectors and thus, on average, produce learned functions that are better able to correlate with random noise than a less complex function space. 

To make this definition more precise and frame it in the context of machine learning (see also \cite{mohri2018foundations}), given an input feature space $X$ and a target space $Y$, let $\mathcal{G}$ denote a family of loss functions \mbox{$L: \mathcal{Z} = X \times Y \to \mathbb{R}$} associated with a function class $\mathcal{H}$. Notationally,
$$
\mathcal{G} = \{g : (x, y) \mapsto L(h(x), y) : h \in \mathcal{H}\}.
$$
We define the \textit{empirical Rademacher complexity} as follows:

\begin{definition}[Empirical Rademacher Complexity] With $\mathcal{G}$ as above, let $S = \{z_1, z_2, \dots, z_m\}$ be a fixed sample of size $m$ of elements of $\mathcal{Z} = X \times Y$. The {\em empirical Rademacher complexity of $\mathcal{G}$} with respect to the sample $S$ is defined as: 
$$
\widehat{\textfrak{R}}_S(\mathcal{G}) = \mathbb{E}_{\bm{\sigma}} \left[\sup_{g\in \mathcal{G}} \dfrac{1}{m} \sum_{i=1}^m \sigma_ig(z_i) \right],
$$
where $\bm{\sigma} = (\sigma_1, \sigma_2, \dots, \sigma_m)^\intercal$, with the $\sigma_i$'s being independent uniform random variables which take values in $\{-1, +1\}$. These random variables $\sigma_i$ are called {\em Rademacher variables}.
\end{definition}

If we let $g_S$ denote the vector of values taken by function $g$ over the sample $S$, then the Rademacher complexity, in essence, measures the expected value of the supremum of the correlation of $g_S$ with a vector of random noise $\bm{\sigma}$; i.e., the empirical Rademacher complexity measures on average how well the function class $\mathcal{G}$ correlates with random noise on the set $S$. More complex families $\mathcal{G}$ can generate more vectors $g_S$ and thus better correlate with random noise on average, see \cite{mohri2018foundations} for more details.

Note that the empirical Rademacher complexity depends on the sample $S$. The Rademacher complexity is then an average of this empirical measure over the distribution from which all samples are drawn:

\begin{definition}[Rademacher Complexity] Let $\mathcal{D}$ denote the distribution from which all samples $S$ are drawn. For any integer $m \geq 1$, the {\em Rademacher complexity of $\mathcal{G}$} is the expectation of the empirical Rademacher complexity over all samples of size $m$ drawn according to $\mathcal{D}$: 
$$
\mathfrak{R}_m(\mathcal{G}) = \mathbb{E}_{S\sim\mathcal{D}^m} [\widehat{\mathfrak{R}}_S(\mathcal{G}) ].
$$

\end{definition}
The Rademacher complexity is distribution dependent and defined for any class of real-valued functions. However, computing it can be intractable for modern day machine learning models. Similar to the empirical Rademacher complexity, the geometric complexity is also computed over a sample of points, in this case the training dataset, and is well-defined for any class of differentiable functions. In contrast, the Rademacher complexity (and the VC dimension which we discuss below) measures the complexity for an entire hypothesis space, while the geometric complexity focuses only on  single functions. Furthermore, since the Geometric Complexity relies only on first derivatives of the learned model function making it much easier to compute. 
\subsection{VC dimension}\label{appendix:vcdimension_comparison}
The Vapnik–Chervonenkis (VC) dimension \cite{blumer1989learnability, chervonenkis1971theory, vapnik1974method} is another common approach to measuring the complexity of a class of functions $\mathcal{H}$ and is often easier to compute than the Rademacher Complexity, see \cite{mohri2018foundations} for further discussion on explicit bounds which compare the Rademacher complexity with the VC dimension. 

The \mbox{VC dimension} is a purely combinatorial notion and defined using the concept of a \textit{shattering} of a set of points. A set of points is said to be shattered by $\mathcal{H}$ if, no matter how we assign a binary label to each point, there exists a member of $\mathcal{H}$ that can perfectly separate the points; i.e., the growth function for $\mathcal{H}$ is $2^m$. The \mbox{VC dimension} of a class $\mathcal{H}$ is the size of the largest set that can be shattered by $\mathcal{H}$.

More formally, we have

\begin{definition}[VC dimension] Let $\mathcal{H}$ denote a class of functions on $X$ taking values in $\{-1, +1\}$. Define the growth function $\Pi_{\mathcal{H}}: \mathbb{N} \to \mathbb{N}$ as
$$
\Pi_{\mathcal{H}}(m) = \max_{\{x_1, \dots, x_m\}\subset X}\left| \{(h(x_1), \dots, h(x_m)) : h\in \mathcal{H}\}\right|.
$$
If $\Pi_{\mathcal{H}} = 2^m$ we say $\mathcal{H}$ shatters the set $\{x_1,\dots,x_m\}$. The {\em VC dimension of $\mathcal{H}$} is the size of the largest shattered set, i.e. 
$$
\textrm{VCdim}(\mathcal{H}) = \max\{m : \Pi_{\mathcal{H}}(m) = 2^m\}
$$
If there is no largest $m$, we define $\textrm{VCdim}(\mathcal{H}) = \infty$.
\end{definition}

The VC dimension is appealing partly because it can be upper bounded for many classes of functions (see for example, \cite{bartlett2019nearly}). Similar to the Rademacher complexity, the VC dimension is measured on the entire hypothesis space. The Geometric Complexity, in contrast, is instead measured for given function within the hypothesis space allowing for more direct comparison between elements within the class $\mathcal{H}$. Computing the VC dimension for a given function set $\mathcal{H}$ may not be always convenient since, by definition, it requires computing the growth function $\Pi_{\mathcal{H}}(m)$ for all subsets of order $m \geq 1$; whereas, the Geometric Complexity relies only on first derivatives and is much easier to precisely compute. 

\subsection{Sharpness and Hessian related measures}\label{appendix:sharpness_comparison}
Another broad category of generalization measures concerns the concept of ``sharpness'' of the local minima; for example, see \cite{hochreiter1997flat, keskar2016large,  mcallester1999pac}. Such complexity measures aim to quantify the sensitivity of the loss to perturbations in model parameters. Here a flat minimizer is a point in parameter space where the loss varies only slightly in a relatively large neighborhood of the point. Conversely, the variation of the loss function is less controlled in a neighborhood around a sharp minimizer. The idea is that for sharp minimizers the training function is more sensitive to perturbations in the model parameters and thus negatively impacts the model's ability to generalize; see also \cite{achille2018emergence} which argues that flat solutions have low information content. 

The sharpness of a minimizer can be
characterized by the magnitude of the eigenvalues of the Hessian of the loss function. However, since the Hessian requires two derivatives this can be computationally costly in most deep learning use cases. To overcome this drawback,  \cite{keskar2016large} suggest a metric that explores the change in values of the loss function $f$ within small neighborhoods of points. More precisely, let $\mathcal{C}_{\epsilon}$ denote a box around an optimal point in the domain of $f$ and let $A\in \mathbb{R}^{n\times p}$ be a matrix whose columns are randomly generated. The constraint $\mathcal{C}_{\epsilon}$ is then defined as:
$$
\mathcal{C}_{\epsilon} = \{z \in \mathbb{R}^p : -\epsilon(|(A^+x)_i|) \leq z_i \leq \epsilon(|(A^+x)_i| + 1) \quad \forall i \in \{1,2,\dots,p\} \}
$$
where $A^+$ denotes the pseudo-inverse of $A$ and $\epsilon$ controls the size of the box. Keskar et al.~\cite{keskar2016large} then define ``sharpness'' by

\begin{definition}[Sharpness]
Given $x \in \mathbb{R}^n$, $\epsilon >0$ and $A\in \mathbb{R}^{n\times p}$, the \emph{$(\mathcal{C}_{\epsilon})$-sharpness of $f$ at $x$ is defined as}
$$
\phi_{x, f}(\epsilon, A) = \dfrac{(\max_{y\in \mathcal{C}_{\epsilon}} f(x + Ay)) - f(x)}{1 + f(x)} \times 100
$$
\end{definition}

Another related complexity measure is the \emph{effective dimension} which is computed using the spectral decomposition of the loss Hessian ~\cite{maddox2020rethinking}. Since the effective dimension relies on the Hessian, it causes flat regions in the loss surface to also be regions of low complexity w.r.t. this measure. 

Effective dimensionality \cite{mackay1991bayesian} was originally proposed to measure the dimensionality of the parameter space determined by the data and is computed using the eigenspectrum of the Hessian of the training loss. 

\begin{definition}[Effective dimensionality of a symmetric matrix] The {\em effective dimensionality} of a symmeteric matrix $A \in \mathbb{R}^{k \times k}$ is defined as 
$$
N_{\text{eff}}(A, z) = \sum_{i=1}^k \dfrac{\lambda_i}{\lambda_i + z}
$$
where $\lambda_i$ are the eigenvalues of $A$ and $z >0$ is a regularization constant.
\end{definition}

When used in the context of measuring the effective dimension of a neural network $f(x; \theta)$ with inputs $x$ and parameters $\theta \in \mathbb{R}^k$, we take $A$ to be the Hessian of the loss function; i.e., the $k \times k$ matrix of second derivatives of the loss $\mathcal{L}$ over the data distribution $\mathcal{D}$ defined as $\text{Hess}_{\theta} = -\nabla^2 \mathcal{L}(\theta, \mathcal{D})$. Furthermore, the computation of the effective dimension involves both double derivatives (to compute the Hessian) but also evaluation of the eigenvalues of the resulting matrix. This can introduce a prohibitive cost in computation for many deep learning models.

The Geometric Complexity differs from these complexity measures in a meaningful way. Sharpness and effective dimension are ultimately concerned with the behavior and  derivatives of the loss function with respect to the parameter space. The Geometric Complexity, however, is measured using derivatives of the learned model function with respect to the model inputs. Furthermore, the Geometric Complexity is computed using only a single derivative, making it computationally tractable to measure and track.

That being said, these sharpness measures and the Geometric Complexity are also quite related. For example, as explained by the Transfer Theorem in Section \ref{section:implicit_regularization}, for neural networks these flat regions are also the regions of low loss gradient and thus of low GC. Furthermore, similar to the GC, the effective dimension can also capture the double descent phenomena and, in ~\cite{maddox2020rethinking}, the authors argue the effective dimension provides an efficient mechanism for model selection.
\end{document}